\documentclass{article} 
\usepackage{iclr2027_conference,times}

\usepackage{amsmath,amsfonts,bm}

\def\eqref#1{equation~\ref{#1}}

\def\1{\bm{1}}

\DeclareMathAlphabet{\mathsfit}{\encodingdefault}{\sfdefault}{m}{sl}
\SetMathAlphabet{\mathsfit}{bold}{\encodingdefault}{\sfdefault}{bx}{n}

\usepackage{hyperref}
\usepackage{url}

\usepackage{comment}
\usepackage{booktabs}
\usepackage{lipsum}
\usepackage{placeins}
\usepackage{graphicx}
\usepackage{amsmath, amsthm, amssymb}
\usepackage{amsfonts}
\usepackage[OT1]{fontenc}
\usepackage{bm}
\usepackage{multirow}
\usepackage{subcaption}
\usepackage{enumitem}

\usepackage{algorithm}
\usepackage{algorithmic}

\newtheorem{prop}{Proposition}

\title{Repurposing Obsolete Representations for Post-Deployment Adaptation}

\author{
Daniel Bethell$^{1}$,
Charmaine Barker$^{1}$,
Simos Gerasimou$^{1,2}$ \\
\\
$^{1}$Department of Computer Science, University of York, UK \\
$^{2}$Department of Elect. Eng., and Computer Science and Eng., Cyprus University of Technology, Cyprus \\
\texttt{\{daniel.bethell, charmaine.barker, simos.gerasimou\}@york.ac.uk}
}

\newcommand{\simos}[1]{\textcolor{black}{#1}}

\iclrfinalcopy 
\begin{document}

\maketitle

\begin{abstract}
Deep neural networks are increasingly deployed in long-lived systems, where task requirements may change after training. In such settings, part of the original output space may become obsolete: a class, prediction region, or learned behaviour may no longer be valid. Existing approaches either leave the obsolete behaviour intact or require fine-tuning, which can be expensive. We propose Deep Repurposing (DR), a post-hoc framework for adapting models under task obsolescence. DR estimates the latent geometry of obsolete and retained regions, removes obsolete-supporting components, and reallocates retained-compatible evidence through an analytic repair map without gradient updates. This yields repaired predictions and representations in which obsolete regions no longer act as valid outputs, while useful obsolete structure can support the retained task. Across multiple task settings, DR removes obsolete behaviour while preserving retained utility. More importantly, across classification benchmarks, DR matches or exceeds competing unlearning and editing baselines in retained accuracy, eliminates obsolete predictions, and adapts up to $60\times$ faster than competing unlearning methods.
\end{abstract}

\section{Introduction}
\label{sec:Introduction}

Deep neural networks have become central to modern perception and decision-making systems, achieving strong performance across computer vision~\citep{khan2022transformers}, robotics~\citep{rudin2022learning}, and autonomous exploration~\citep{cimurs2021goal}. In these settings, models are often deployed as part of larger operational pipelines, where their predictions may guide downstream planning. In practice, deployment settings often evolve after training, while the model itself remains costly or difficult to revise~\citep{sculley2015hidden}.

When deployment conditions change, parts of the original task may cease to reflect the system’s current objective. This leads to task obsolescence, where a previously valid output class, prediction region, or learned behaviour becomes invalid, irrelevant, or no longer actionable. For example, a perception model may retain an output category that is no longer used by the downstream system, or a prediction region whose operational meaning has changed~\citep{sculley2015hidden}.

Existing approaches to task obsolescence remain limited. Leaving a model unchanged allows obsolete outputs to persist~\citep{sculley2015hidden}, while output-level suppression or post-hoc repair can alter exposed predictions without necessarily changing the underlying latent representation~\citep{hakemi2025post}. Retraining and fine-tuning provide stronger adaptation but require training data and substantial computation. Machine unlearning~\citep{zhang2023review} and model editing~\citep{ilharco2022editing} instead remove or modify learned behaviour, but do not consider whether obsolete structure can be reused. Existing methods therefore suppress, erase, or retrain obsolete behaviour rather than repurposing useful obsolete evidence for the retained task.

Motivated by this gap, we propose \textit{Deep Repurposing (DR)}, a post-hoc framework for adapting models under task obsolescence. Rather than simply suppressing obsolete outputs or erasing their associated knowledge, DR removes obsolete identity while reallocating retained-compatible evidence through the model’s learned geometry. This produces a repaired prediction space in which obsolete regions no longer act as valid outputs, but their useful structure can still support the retained task. Empirically, we show that this enables task adaptation without full retraining, reducing obsolete behaviour while preserving retained performance across multiple task settings.

\section{Related Work}
\label{sec:Related Work}

Deployed neural networks may become misaligned as data, labels, or downstream objectives evolve~\citep{sculley2015hidden}. Retraining or fine-tuning can adapt models to revised tasks, but require access to training data and significant computation.

Machine unlearning~\citep{zhang2023review} removes unwanted behaviour from trained models, for example through curvature-based degradation~\citep{golatkar2020eternal} or parameter dampening~\citep{foster2024fast, foster2024loss}. However, these methods treat the target primarily as information to suppress and can harm retained performance.

Model editing and classifier repair instead alter targeted behaviours through changes to classifier weights, feature directions~\citep{kim2026exploring}, or decision regions~\citep{hakemi2025post, ilharco2022editing}. While cheaper than retraining, their objective is generally to erase, weaken, or redirect unwanted behaviour rather than distinguish obsolete identity from useful obsolete evidence or repurpose that evidence toward retained regions.

Representation-based approaches provide another alternative. Prototype and metric-learning methods use embedding-space centroids or distances~\citep{mensink2013distance, snell2017prototypical}, while feature-direction methods identify behaviour-associated directions~\citep{kim2018interpretability, ilharco2022editing}. These approaches are typically designed for recognition, generalisation, or generic editing rather than task-obsolescence repair. DR instead uses representation geometry to remove obsolete-supporting components, redirect useful evidence to retained regions, and reconstruct repaired outputs and latents without gradient updates.

Overall, existing approaches can retrain, forget, edit, or mask obsolete behaviour, but do not directly address task obsolescence as repurposing. DR fills this gap by analytically redirecting useful obsolete structure toward the retained task.

\section{Preliminaries}
\label{sec:Preliminaries}
We consider a supervised prediction problem with an input space $\mathcal{X}$ and target space $\mathcal{Y}$. Samples $(x, y) \in \mathcal{X} \times \mathcal{Y}$ are drawn from an underlying distribution $P$, from which we obtain disjoint train, calibration, and test sets $\mathcal{D}_{\text{train}}, \mathcal{D}_{\text{cal}}, \mathcal{D}_{\text{test}} \overset{\text{i.i.d.}}{\sim} P_{\mathcal{X}\times\mathcal{Y}}$. We assume a pretrained model $f_\theta : \mathcal{X}\rightarrow\hat{\mathcal{Y}}$ with parameters $\theta$ that was trained on $\mathcal{D}_{\text{train}}$. The model is defined as $f_\theta(x) = h_\theta(\phi_\theta(x))$ where $z = \phi_\theta(x) \in \mathbb{R}^d$ is the penultimate latent representation and $h_\theta$ is the following prediction head.

To express obsolescence, we introduce a finite set of decision regions $\mathcal{C} = \{1, \ldots, K\}$ where $K$ is the number of task-specific regions, and $c: \mathcal{Y} \rightarrow \mathcal{C}$ maps each target to its associated region. In classification, $c(y) = y$ and each region corresponds to an output class. In regression or structured prediction, $c(y)$ may instead denote a value interval, output prototype, anchor, or application-defined regime. A post-deployment change identifies a subset $\mathcal{O} \subset \mathcal{C}$ as obsolete, with the remaining regions $\mathcal{R} = \mathcal{C} \setminus \mathcal{O}$ retained.

The objective of Deep Repurposing is to construct a lightweight repaired predictor $f_\theta^{*}$ that removes prediction behaviour associated with $\mathcal{O}$, while preserving or improving performance on retained regions $\mathcal{R}$. Importantly, the pretrained parameters remain fixed: i.e., $\theta$ is not updated during repair or deployment adaptation.

\section{Deep Repurposing}
\label{sec:Deep Repurposing}
Deep Repurposing (DR) is an analytic post-deployment repair process that operates within the latent space of a neural network. Given a model $f_\theta = h_\theta \circ \phi_\theta$ and a calibration set $\mathcal{D}_{\text{cal}}$, DR operates in two phases. In the offline calibration phase, a small calibration set is used to estimate region geometry, obsolete-retained directions, and residual correction prototypes. In the online inference phase, these fixed quantities are used to repair each input analytically, without gradients, optimiser steps, or fine-tuning. The repaired model removes obsolete regions from the output space while reallocating preserved-compatible evidence to retained regions.

\begin{figure*}[tb]
    \centering
    \includegraphics[width=0.95\linewidth]{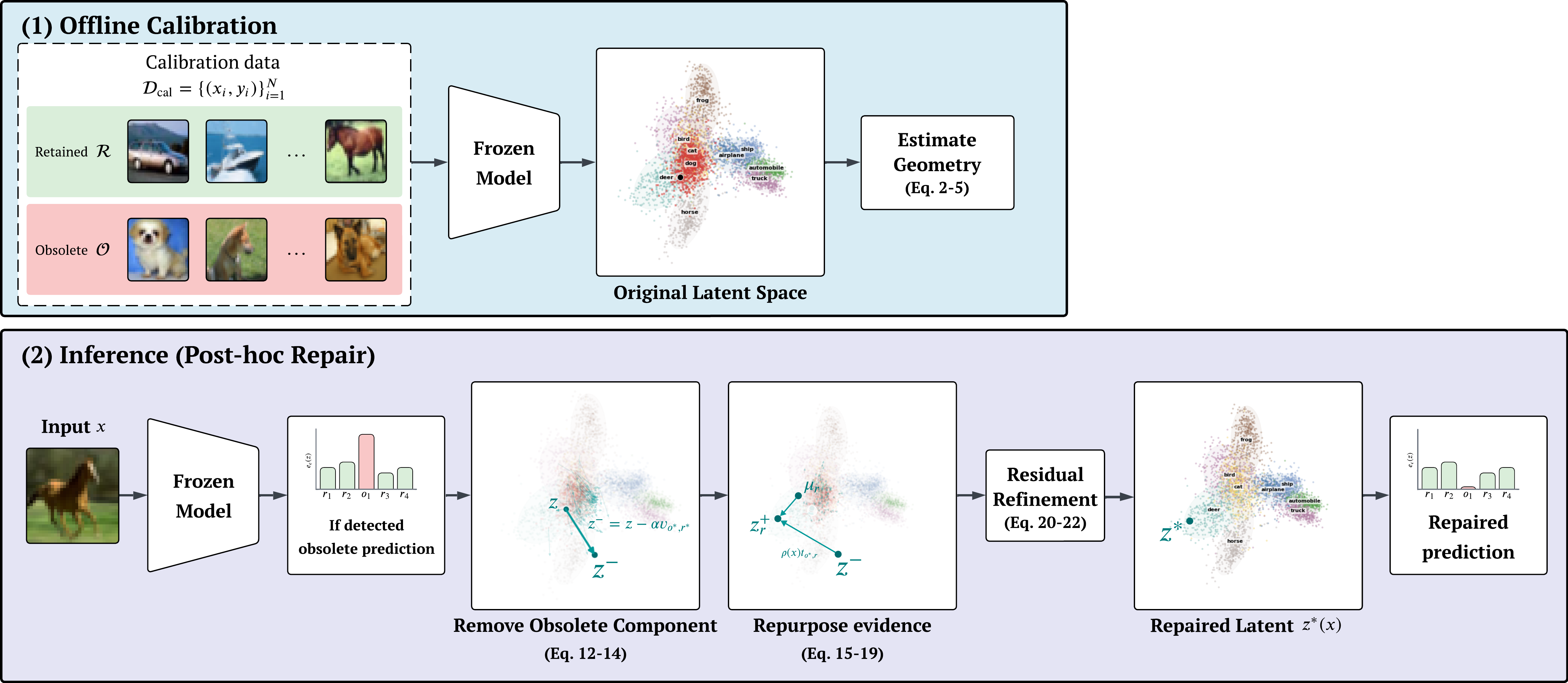}
    \caption{Overview of Deep Repurposing (DR). Offline calibration estimates retained and obsolete latent geometry and residual correction quantities using a frozen pretrained model. At inference, DR detects obsolete predictions, removes obsolete-supporting components, repurposes evidence towards the best-suited retained regions, applies residual refinement, and returns repaired predictions and latent representations.}
    \label{fig:overview}
\end{figure*}

\subsection{Offline Calibration}
\label{sec:DR-offline-calibration}

The offline phase is performed once before deployment. Its role is to estimate all geometric and residual quantities required by DR at inference time. The pretrained network is used only as a fixed feature extractor and predictor; no gradients are computed, and no model parameters are updated.

DR first embeds calibration samples using the backbone:
\begin{equation}
Z_{\mathrm{cal}} = \{z_i = \phi_\theta(x_i) : (x_i,y_i) \in \mathcal{D}_{\mathrm{cal}}\}
\label{eq:dr-cal-latents}
\end{equation}
For each region $c \in \mathcal{O}\cup\mathcal{R}$, let $\mathcal{D}^{c}_{\mathrm{cal}} = \{(x_i,y_i)\in\mathcal{D}_{\mathrm{cal}} : c(y_i)=c \}$ denote the calibration samples assigned to region $c$. DR fits a diagonal Gaussian prototype to the latent representations of each region:
\begin{equation}
\begin{aligned}
\mu_c
&=
\frac{1}{|\mathcal{D}^{c}_{\mathrm{cal}}|}
\sum_{(x_i,y_i)\in\mathcal{D}^{c}_{\mathrm{cal}}}
z_i, \\
\sigma_c^2
&=
\frac{1}{|\mathcal{D}^{c}_{\mathrm{cal}}|}
\sum_{(x_i,y_i)\in\mathcal{D}^{c}_{\mathrm{cal}}}
(z_i-\mu_c)\odot(z_i-\mu_c)
+
\epsilon\mathbf{1}_d
\end{aligned}
\label{eq:dr-region-moments}
\end{equation}
where $\epsilon>0$ is a numerical stabiliser. These Gaussian prototypes provide a lightweight estimate of where each retained and obsolete region lies in the model's latent space.

For any latent representation $z$, DR defines a region evidence score $e_c(z)$ for each $c\in\mathcal{O}\cup\mathcal{R}$. When the model has a final linear layer (a score head) producing region scores $s_\theta(z)=Wz+b$, the evidence is the corresponding model score $e_c(z) = s_{\theta,c}(z)$. When no explicit score head is available, as in prototype-defined regression or structured prediction, DR instead uses the calibrated latent likelihood $e_c(z) = \log \widehat{p}(z\mid c)$. Thus, $e_c(z)$ always measures the support that the representation gives to region $c$, either through the model's own score head or through calibration-defined prototype geometry.

DR then constructs pairwise obsolete-retained directions for every $(o,r) \in \mathcal{O} \times \mathcal{R}$. If a score head is available, the obsolete-removal direction is induced by the classifier:
\begin{equation}
v_{o,r}
=
\frac{W_o-W_r}{\|W_o-W_r\|_2}
\label{eq:dr-score-removal-direction}
\end{equation}
\simos{where $\{W_o, W_r\} \in \mathbb{R}^d$ are the weight vectors corresponding to the obsolete $o$ and retained $r$ classes, respectively.}
This is the direction in latent space that increases the score of obsolete region $o$ relative to retained region $r$. Moving in the opposite direction, therefore, decreases obsolete evidence relative to that retained alternative. The corresponding retained-transport direction is estimated from the calibrated prototypes:
\begin{equation}
t_{o,r}
=
\frac{\mu_r-\mu_o}{\|\mu_r-\mu_o\|_2+\epsilon}
\label{eq:dr-transport-direction}
\end{equation}
The removal direction $v_{o,r}$ determines which obsolete-supporting component should be stripped from the representation, while $t_{o,r}$ determines where the removed evidence is redirected in retained latent geometry.

For models without an explicit score head, both directions are defined from the prototype axis:
\begin{equation}
v_{o,r}
=
\frac{\mu_o-\mu_r}{\|\mu_o-\mu_r\|_2+\epsilon},
\qquad
t_{o,r}
=
-v_{o,r}
\label{eq:dr-prototype-directions}
\end{equation}
In this case, the same obsolete-retained axis is used in opposite directions: $v_{o,r}$ removes movement toward the obsolete prototype, while $t_{o,r}$ redirects the representation toward the retained prototype.

The offline phase also calibrates residual refinement. This step corrects systematic errors left by the analytic core repair. Since obsolete labels are no longer valid under the revised task, residual refinement is estimated only on retained calibration samples $\mathcal{D}^{\mathcal{R}}_{\mathrm{cal}} = \{(x_i,y_i)\in\mathcal{D}_{\mathrm{cal}} : c(y_i)\in\mathcal{R}\}$. Let $\psi(y_i)\in\Delta_{\mathcal{R}}$ denote the retained target representation for $y_i$. In classification, $\psi(y_i)$ is the one-hot vector over retained classes. In region-based regression, it is the target responsibility vector over retained value regions.

Residual refinement is fitted using the output of DR before any residual correction is applied. Let $\mathcal{T}_{\mathrm{core}}: \mathcal{X} \rightarrow \Delta_{\mathcal{R}}$ define the pre-refinement DR mapping from an input to a retained-region responsibility vector. This is the analytic retained-region repair used at inference after the geometric quantities above have been fitted, but with residual correction disabled. Its construction is given in the online inference phase below. For each retained calibration sample $(x_i,y_i) \in \mathcal{D}^{\mathcal{R}}_{\mathrm{cal}}$, DR computes $q^{\mathrm{core}}_{\mathcal{R}}(x_i) = \mathcal{T}_{\mathrm{core}}(x_i)$. 

The residual error is then defined as the difference between the desired retained target representation (one-hot for classification) and the pre-refinement probability vector $\delta_i^{\mathrm{res}} = \psi(y_i)-q^{\mathrm{core}}_{\mathcal{R}}(x_i)$.

Each residual is assigned to the retained destination selected by the core repair $a_i = \arg\max_{r\in\mathcal{R}} q^{\mathrm{core}}_r(x_i)$. For each retained destination $a \in \mathcal{R}$, let $\mathcal{I}_a=\{i:a_i=a\}$ denote the set of retained calibration samples assigned to destination $a$. DR estimates the mean residual for destination $a$ as:
\begin{equation}
\bar{\delta}_a
=
\frac{1}{|\mathcal{I}_a|}
\sum_{i\in\mathcal{I}_a}
\delta_i^{\mathrm{res}}
\label{eq:dr-mean-residual}
\end{equation}
To avoid over-correcting noisy destinations, this residual is reliability-adjusted:
\begin{equation}
\eta_a
=
\frac{\|\bar{\delta}_a\|_2^2}
{\|\bar{\delta}_a\|_2^2
+
\frac{1}{|\mathcal{I}_a|}
\operatorname{tr}(\widehat{\Sigma}_{\delta,a})
+
\epsilon}
\label{eq:dr-residual-shrinkage}
\end{equation}
where $\widehat{\Sigma}_{\delta,a}$ is the empirical covariance of the residuals assigned to destination $a$. The stored residual prototype is $\widetilde{\delta}_a = \eta_a\bar{\delta}_a$. If no calibration sample selects destination $a$, DR sets $\widetilde{\delta}_a = 0$.

Finally, DR estimates the amount of obsolete responsibility that is typical for retained calibration samples. First, the all-region responsibility is defined as:
\begin{equation}
\widetilde{q}_c(x)
=
\frac{\exp(e_c(\phi_\theta(x)))}
{\sum_{k\in\mathcal{O}\cup\mathcal{R}}
\exp(e_k(\phi_\theta(x)))},
\qquad
c\in\mathcal{O}\cup\mathcal{R}
\label{eq:dr-all-region-responsibility}
\end{equation}
The total obsolete responsibility is $\widetilde{q}_{\mathcal{O}}(x) = \sum_{o\in\mathcal{O}} \widetilde{q}_o(x)$. The retained-background obsolete responsibility is then:
\begin{equation}
\bar{q}^{\mathrm{cal}}_{\mathcal{O},\mathcal{R}}
=
\frac{1}{|\mathcal{D}^{\mathcal{R}}_{\mathrm{cal}}|}
\sum_{(x_i,y_i)\in\mathcal{D}^{\mathcal{R}}_{\mathrm{cal}}}
\widetilde{q}_{\mathcal{O}}(x_i)
\label{eq:dr-retained-background-obsolete-responsibility}
\end{equation}
This quantity provides the calibration baseline used to decide how strongly residual refinement should be applied at inference. For region-based regression, any scalar decoder coefficient $\kappa\in[0,1]$ is also selected on calibration data. The offline phase returns the fixed DR calibration object $\Lambda_{\mathrm{DR}}=\{\mu_c,\sigma_c^2,v_{o,r},t_{o,r},\widetilde{\delta}_a,\bar{q}^{\mathrm{cal}}_{\mathcal{O},\mathcal{R}},\kappa\}$,
whose components remain fixed during deployment.

\subsection{Online Inference}
\label{sec:DR-online-inference}

At inference time, DR repairs each input independently using the fixed calibration object $\Lambda_{\mathrm{DR}}$. Given an input $x$ and the original latent representation $z=\phi_\theta(x)$, DR evaluates the evidence scores $e_c(z)$ for all regions and defines the repair mask:
\begin{equation}
\mathcal{M}(x)
=
\mathbf{1}
\left[
\arg\max_{c\in\mathcal{O}\cup\mathcal{R}} e_c(z)
\in\mathcal{O}
\right]
\label{eq:dr-repair-mask}
\end{equation}
The mask identifies whether the model currently assigns the input to an obsolete region. This keeps the core intervention targeted: the retained evidence is analytically replaced only when obsolete evidence is active.

When $\mathcal{M}(x)=1$, DR selects the strongest obsolete region and the strongest retained alternative:
\begin{equation}
o^{*}
=
\arg\max_{o\in\mathcal{O}}e_o(z),
\qquad
r^{*}
=
\arg\max_{r\in\mathcal{R}}e_r(z)
\label{eq:dr-active-regions}
\end{equation}
It then removes the obsolete-supporting component from $z$:
\begin{equation}
z^{-}
=
z-\alpha(x)v_{o^{*},r^{*}}
\label{eq:dr-stripped-latent}
\end{equation}
For score-based heads, the removal coefficient is the smallest non-negative movement along $v_{o^{*},r^{*}}$ that removes the active obsolete-retained margin up to margin $m \geq 0$:
\begin{equation}
\alpha(x)
=
\max
\left\{
\frac{e_{o^{*}}(z)-e_{r^{*}}(z)+m}
{\|W_{o^{*}}-W_{r^{*}}\|_2},
0
\right\}
\label{eq:dr-score-removal-coefficient}
\end{equation}
For prototype-based repair, where no score-head margin is available, the removal coefficient is the positive projection of $z$ along the obsolete-retained prototype axis:
\begin{equation}
\alpha(x)
=
\max
\left\{
\langle z-\mu_{r^{*}},v_{o^{*},r^{*}}\rangle,
0
\right\}
\label{eq:dr-prototype-removal-coefficient}
\end{equation}
If $\mathcal{M}(x)=0$, DR sets $z^{-}=z$ and preserves the original retained evidence.

The stripped representation is then evaluated under each retained prototype by computing $B_r(z^{-})=\log \widehat{p}(z^{-}\mid r)$ for every $r\in\mathcal{R}$. These scores measure which retained region best explains the representation after obsolete-supporting evidence has been removed. Since Bayes scores and model scores may lie on different numerical scales, DR aligns the retained Bayes-score vector to the scale of the retained model-evidence vector before combining them. For retained-region vectors $a_{\mathcal{R}}$ and $u_{\mathcal{R}}$, define
\begin{equation}
\operatorname{Align}(a_{\mathcal{R}};u_{\mathcal{R}})_r
=
\bar{u}
+
\frac{s_u}{s_a+\epsilon}
(a_r-\bar{a}),
\quad r\in\mathcal{R}
\label{eq:dr-scale-alignment}
\end{equation}
where $\bar{a},s_a$ and $\bar{u},s_u$ are the mean and standard deviation of $a_{\mathcal{R}}$ and $u_{\mathcal{R}}$, respectively. DR then uses the aligned retained Bayes scores $\widehat{B}_{\mathcal{R}}(z^{-})=\operatorname{Align}(B_{\mathcal{R}}(z^{-});e_{\mathcal{R}}(z))$.

DR next estimates which retained region benefits most specifically from the removed obsolete evidence. Let $\rho(x) = \|z-z^{-}\|_2$ denote the magnitude of the stripped component. For each retained candidate $r$, DR transports the stripped representation toward that retained region $z^{+}_r = z^{-} + \rho(x)t_{o^{*},r}$. The contrastive gain for the retained region $r$ is:
\begin{equation}
G_r(x)
=
\left[B_r(z^{+}_r)-B_r(z^{-})\right]
-
\max_{k\in\mathcal{R},\,k\neq r}
\left[B_k(z^{+}_r)-B_k(z^{-})\right]
\label{eq:dr-contrastive-gain}
\end{equation}
Only positive gains are retained $G^{+}_r(x) = \max\{G_r(x),0\}$. The gain is therefore high only when transport toward $r$ improves the likelihood of $r$ more than it improves competing retained regions. The aligned gain is:
\begin{equation}
\widehat{G}^{+}_r(x)
=
\operatorname{Align}(G^{+}_{\mathcal{R}}(x);e_{\mathcal{R}}(z))_r
\label{eq:dr-aligned-gain}
\end{equation}
The repaired retained evidence proposed by the core map is:
\begin{equation}
\widetilde{e}_r(x)
=
(1-\beta)e_r(z)
+
\beta\widehat{B}_r(z^{-})
+
\lambda\widehat{G}^{+}_r(x),
\qquad
r\in\mathcal{R}
\label{eq:dr-repaired-retained-evidence}
\end{equation}
where $\beta$ controls the contribution of retained prototype evidence and $\lambda$ controls the contribution of the contrastive gain. The core retained evidence is then:
\begin{equation}
e^{\mathrm{core}}_r(x)
=
\begin{cases}
\widetilde{e}_r(x), & \mathcal{M}(x)=1,\\
e_r(z), & \mathcal{M}(x)=0,
\end{cases}
\qquad
r\in\mathcal{R}
\label{eq:dr-core-retained-evidence}
\end{equation}
Normalising over retained regions gives the core retained responsibility vector:
%
%
Applying a softmax over these retained-region scores gives $q^{\mathrm{core}}_\mathcal{R}(x)=\mathcal{T}_{\mathrm{core}}(x)$, the pre-refinement probability vector used during calibration to estimate the residual prototypes. DR then applies the residual correction estimated during calibration. The correction for the current input is the responsibility-weighted combination of stored residual prototypes:
\begin{equation}
\Delta(x)
=
\sum_{a\in\mathcal{R}}
q^{\mathrm{core}}_a(x)\widetilde{\delta}_a
\label{eq:dr-residual-correction}
\end{equation}
The correction is gated by the input's obsolete responsibility relative to the retained calibration background:
\begin{equation}
g_{\mathrm{obs}}(x)
=
\min
\left\{
1,
\frac{\widetilde{q}_{\mathcal{O}}(x)}
{\bar{q}^{\mathrm{cal}}_{\mathcal{O},\mathcal{R}}+\epsilon}
\right\}
\label{eq:dr-obsolete-gate}
\end{equation}
The final retained responsibility vector is:
\begin{equation}
q^{*}_{\mathcal{R}}(x)
=
\Pi_{\Delta_{\mathcal{R}}}
\left[
q^{\mathrm{core}}_{\mathcal{R}}(x)
+
g_{\mathrm{obs}}(x)\Delta(x)
\right]
\label{eq:dr-refined-responsibility}
\end{equation}
where $\Pi_{\Delta_{\mathcal{R}}}$ projects onto the retained simplex. For classification, this projection is simplex normalisation. For unconstrained regression targets, the corresponding projection may be the identity map or the appropriate retained-output projection.

For score-based classification, the final retained responsibilities are converted back into retained evidence as $e^{*}_r(x)=\log(q^{*}_r(x)+\epsilon)$ for each $r\in\mathcal{R}$. Obsolete regions are assigned a finite evidence floor:
\begin{equation}
e^{*}_o(x)
=
2\min_{r\in\mathcal{R}} e^{*}_r(x)
-
\log\!\left(1+|\mathcal{R}|\epsilon\right),
\quad
o\in\mathcal{O}
\label{eq:dr-obsolete-evidence-floor}
\end{equation}
This makes obsolete evidence non-competitive with retained evidence in the repaired raw output, rather than merely removing obsolete labels with an external mask. The repaired classification distribution is:
\begin{equation}
p^{*}_c(x)
=
\frac{\exp(e^{*}_c(x))}
{\sum_{k\in\mathcal{O}\cup\mathcal{R}}\exp(e^{*}_k(x))},
\qquad
c\in\mathcal{O}\cup\mathcal{R}
\label{eq:dr-final-classification-distribution}
\end{equation}
and the repaired prediction is $\widehat{y}^{*}(x) = \arg\max_{r\in\mathcal{R}} p^{*}_r(x)$. The maximisation is restricted to $\mathcal{R}$ because obsolete regions are no longer valid outputs under the revised task.

For target-region regression or structured prediction, the retained responsibilities are decoded through the retained output support $\widehat{y}_{\mathrm{reg}}(x) = \sum_{r\in\mathcal{R}} q^{*}_r(x)v_r$ where $v_r$ is the representative output value for the retained region $r$. DR also projects the original prediction onto the retained output support $\widehat{y}_{\mathrm{proj}}(x) = \Pi_{\mathcal{Y}_{\mathcal{R}}} [\widehat{y}_{\theta}(x)]$. The final repaired regression prediction is:
\begin{equation}
\widehat{y}^{*}(x)
=
\Pi_{\mathcal{Y}_{\mathcal{R}}}
\left[
\kappa \widehat{y}_{\mathrm{reg}}(x)
+
(1-\kappa)\widehat{y}_{\mathrm{proj}}(x)
\right]
\label{eq:dr-final-regression-output}
\end{equation}
When the obsolete region is semantic rather than a numeric output interval, the repaired retained responsibilities are instead used to construct a repaired latent representation, which is decoded by the original task head.

The final step is latent repurposing. This ensures that downstream modules inspecting the representation receive a latent vector consistent with the repaired task. For score-based heads, DR constructs a latent whose retained head scores match the repaired retained evidence. Let $W_{\mathcal{R}}$ and $b_{\mathcal{R}}$ denote the retained rows of the score head. The score-consistent component is $z^{*}_{\mathrm{row}} = W_{\mathcal{R}}^{\top} \left(W_{\mathcal{R}}W_{\mathcal{R}}^{\top}\right)^{-1}\left(e^{*}_{\mathcal{R}}(x)-b_{\mathcal{R}}\right)$ which gives the minimum-norm latent component whose retained-head scores match the repaired evidence. Since multiple latent vectors can produce the same retained scores, DR fills the remaining nullspace using retained-region geometry:
\begin{equation}
z^{*}_{\mathrm{null}}
=
\sum_{r\in\mathcal{R}}
q^{*}_r(x)
\left(I-W_{\mathcal{R}}^{\dagger}W_{\mathcal{R}}\right)
\mu_r
\label{eq:dr-nullspace-latent}
\end{equation}
The repaired latent is then $z^{*}(x) = z^{*}_{\mathrm{row}} + z^{*}_{\mathrm{null}}$. The head-visible component is therefore consistent with the repaired retained decision, while the remaining degrees of freedom are reconstructed from retained prototype geometry.

When no score head is available, exact score consistency cannot be imposed. DR instead returns the retained posterior barycentre $z^{*}(x) = \sum_{r\in\mathcal{R}} q^{*}_r(x)\mu_r$. Thus, in all task settings, DR removes obsolete regions from the repaired output space and reconstructs the returned representation from retained-region structure, while the pretrained parameters $\theta$ remain fixed throughout calibration and inference.

\section{Experimental Evaluation}
\label{sec:Experimental Evaluation}
We evaluate Deep Repurposing in a comprehensive series of experiments comparing against state-of-the-art approaches across multiple tasks and models. The 

\textbf{Comparative Approaches.} We compare Deep Repurposing against the base model, SSD~\citep{foster2024fast}, LFSSD~\citep{foster2024loss}, Boundary Expand~\citep{chen2023boundary}, Boundary Shrink~\citep{chen2023boundary}, Class Vector Unlearning~\citep{kim2026exploring}, Fictitious Class Removal~\citep{hakemi2025post},  Fisher Forgetting~\citep{golatkar2020eternal}, finetuning the base model, and retraining the base model to represent a range of approaches in unlearning and model editing that allow for a fair comparison.

\textbf{Datasets.} 
Adopting the procedure from recent research~\citep{foster2024fast}, we evaluate all approaches on the CIFAR10~\citep{krizhevsky2009learning}, CIFAR100~\citep{krizhevsky2009learning}, Oxford Flowers~\citep{Nilsback08}, ImageNet~\citep{deng2009imagenet},
FishNet~\citep{khan2023fishnet},
MSL~\citep{wagstaff2018deep},
UTKFace~\citep{zhifei2017cvpr}, Bike Sharing~\citep{fanaee2014event}, and ADE20K~\citep{zhou2017scene} datasets, selected to cover a diverse set of domains and tasks.

Implementation details, theoretical analysis, experimental setup, additional experiments, and extended discussion are provided in Appendices~\ref{sec:appendix-Implementation Details},~\ref{sec:appendix-Theoretical Analysis},~\ref{sec:appendix-Additional Experiments},~\ref{sec:appendix-Additional Discussion}, and~\ref{sec:appendix-Experimental Setup}. Appendix~\ref{sec:appendix-case-studies} provides operational case studies on Mars surface classification, deep-sea fish classification, and semantic-region regression, illustrating deployment-driven obsolescence from mission-phase changes, unreliable categories, and application-defined semantic regions.

\begin{figure*}[tb]
    \centering
    \includegraphics[width=\linewidth]{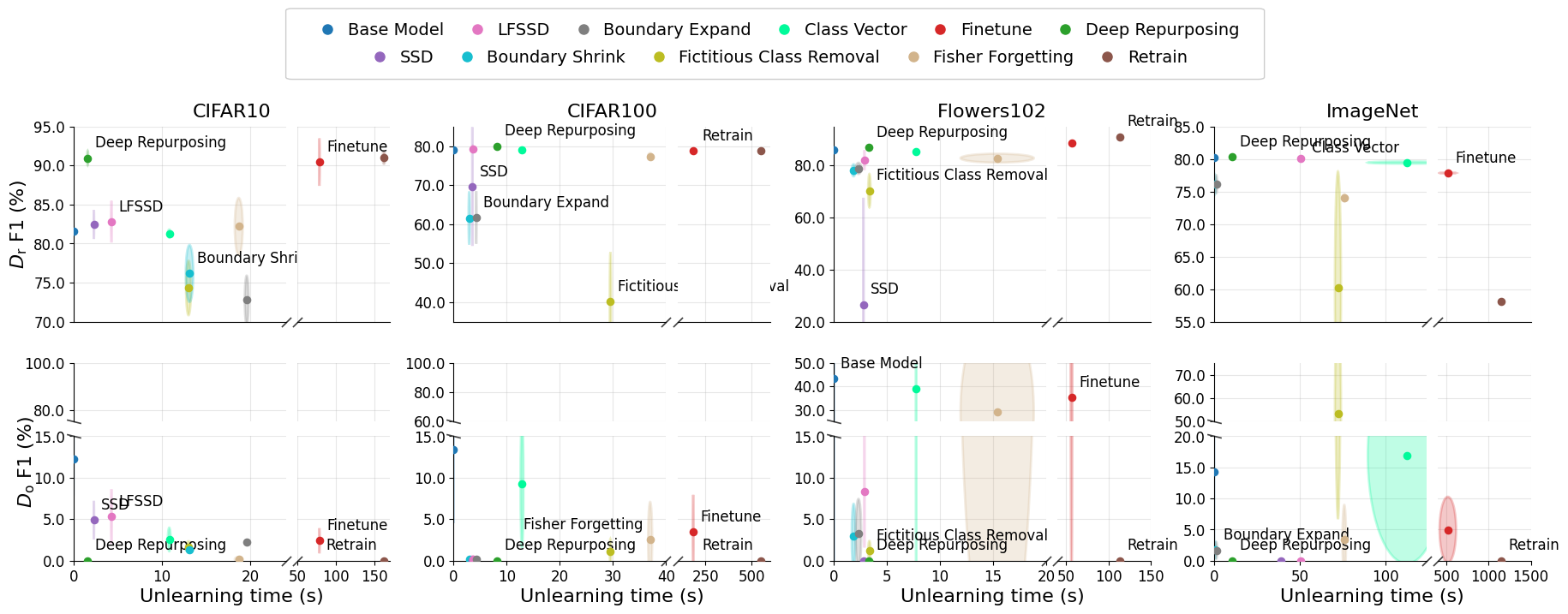}
    \caption{Classification repair trade-off across all datasets. Each panel compares retained-set F1 on $D_{\mathrm{r}}$ and obsolete-set F1 on $D_{\mathrm{o}}$ against repair time.}
    \label{fig:classification_do_dr_acc}
\end{figure*}

\subsection{Core Results}
\label{sec:Core Results}

In the core experiments, we evaluate whether Deep Repurposing achieves our main objective: removing obsolete behaviour while preserving/improving retained task utility. For each dataset, we partition the test set into retained $D_{\mathrm{r}}$ and obsolete samples $D_{\mathrm{o}}$. We then report upon metrics covering retained performance, obsolete-output suppression, predictive reliability, representation-level repair, and computational cost. This allows us to distinguish methods that merely reduce obsolete predictions from those that preserve the retained task while also repairing obsolete behaviour in the model’s output and latent spaces.

Results across the classification datasets are summarised in Figure~\ref{fig:classification_do_dr_acc}. DR consistently occupies the desired trade-off region, preserving or improving retained performance by up to $\approx10$ F1 over the base model, while suppressing obsolete behaviour at low repair cost. Competing methods either sacrifice retained utility or require substantially greater optimisation cost.

On CIFAR10, DR matches retraining while requiring $1.62 \pm 0.04$ seconds versus $162.17 \pm 1.12$ seconds. On CIFAR100, it preserves base-model accuracy, improves F1 from $0.79$ to $0.80$, and completes repair in $8.25 \pm 0.24$ seconds; on Flowers102, it similarly improves retained accuracy and F1 in $3.33 \pm 0.28$ seconds. At ImageNet scale, retraining requires $1167.15 \pm 48.04$ seconds ($\approx20$ minutes) yet reaches only $\approx58\%$ accuracy, compared with $\approx81\%$ for both the base model and DR.

\begin{figure*}[tb]
    \centering
    \includegraphics[width=\linewidth]{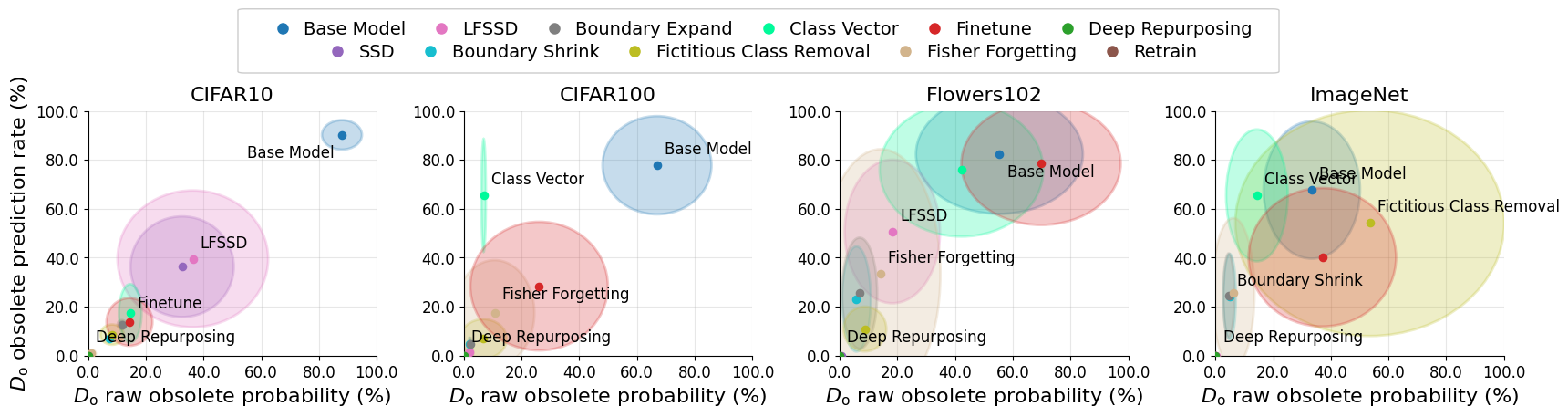}
    \caption{Obsolete-output suppression across classification datasets. Each panel compares the raw probability assigned to obsolete classes on $D_{\mathrm{o}}$ against the obsolete prediction rate on $D_{\mathrm{o}}$.}
    \label{fig:classification_obs_rate_prob}
\end{figure*}

Figure~\ref{fig:classification_obs_rate_prob} examines obsolete-output behaviour on $D_{\mathrm{o}}$, where the desired outcome is low obsolete prediction rate and raw obsolete probability. DR consistently lies near the origin, showing that it suppresses both obsolete predictions and residual probability mass rather than merely changing the final label. In contrast, boundary-based methods often trade suppression for retained accuracy, while class vector editing and fine-tuning leave non-trivial obsolete behaviour in several settings. Retraining provides strong repair at substantially greater computational cost. Overall, DR suppresses obsolete behaviour while preserving retained utility, without gradient-based optimisation of the original model.

We next evaluate whether the same repurposing principle extends beyond discrete classification labels. Table~\ref{tab:regression-main-results} reports results on UTKFace, where obsolescence is defined as a target-associated output region rather than a class label. The base model assigns a large fraction of obsolete-region samples back to the obsolete output region, with an obsolete prediction rate of $0.61 \pm 0.01$. Fine-tuning improves retained MAE but only partially suppresses the obsolete region, reducing the obsolete rate to $0.38 \pm 0.03$. Retraining removes obsolete predictions entirely and achieves the lowest retained MAE, but requires substantially more computation. In contrast, DR reduces the obsolete prediction rate to $0.00 \pm 0.00$ while preserving the retained-set error of the base model, with $D_{\mathrm{r}}$ MAE changing only from $5.69 \pm 0.04$ to $5.68 \pm 0.08$. This repair is also lightweight, requiring $16.00 \pm 2.81$ seconds compared with $1208.13 \pm 31.28$ seconds for retraining.

\begin{table}[tb]
\centering
\caption{Regression trade-off results on the UTKFace dataset where the obsolete region is tied to the target, comparing retained set error, obsolete-set prediction rate and repair time.}
\label{tab:regression-main-results}
\setlength{\tabcolsep}{3pt}
\begin{tabular}{lccc}
\toprule
Method & $D_{\mathrm r}$ MAE $\downarrow$ & $D_{\mathrm o}$ Obs. Rate $\downarrow$ & Time (s) $\downarrow$ \\
\midrule
Base Model & $5.69 \pm 0.04$ & $0.61 \pm 0.01$ & -- \\
Finetune & $5.28 \pm 0.07$ & $0.38 \pm 0.03$ & $282.05 \pm 100.70$ \\
Retrain & $4.46 \pm 0.03$ & $0.00 \pm 0.00$ & $1208.13 \pm 31.28$ \\
DR & $5.68 \pm 0.08$ & $0.00 \pm 0.00$ & $16.00 \pm 2.81$ \\
\bottomrule
\end{tabular}
\end{table}

Table~\ref{tab:bikesharing-semantic-main} considers semantic-region regression obsolescence, treating \texttt{light\_precip} as obsolete and \texttt{clear}, \texttt{mist}, and \texttt{heavy\_precip} as retained. DR preserves predictive quality, reducing MAE from $16.34$ to $16.27$ on $D_{\mathrm r}$ and from $14.53$ to $13.72$ on $D_{\mathrm o}$, with RMSE on $D_{\mathrm o}$ decreasing from $35.92$ to $31.94$. It also reduces reliance on the obsolete region: responsibility/winner rate fall from $0.430/0.447$ to $0.283/0.237$ on $D_{\mathrm r}$ and from $0.562/0.589$ to $0.376/0.304$ on $D_{\mathrm o}$. Thus, DR suppresses obsolete semantic structure without sacrificing the underlying regression task.

\begin{table}[tb]
\centering
\setlength{\tabcolsep}{3.5pt}
\caption{Semantic-region repurposing results on the Bike Sharing dataset. 
The obsolete region is \texttt{light\_precip}; retained regions are 
\texttt{clear}, \texttt{mist}, and \texttt{heavy\_precip}. 
$q_{\mathrm{obs}}$ denotes obsolete-region responsibility, and Obs. win 
denotes the fraction of samples assigned to the obsolete region.}
\label{tab:bikesharing-semantic-main}

\begin{tabular}{@{}llccccc@{}}
\toprule
& &
\multicolumn{3}{c}{\textbf{Predictive performance}} &
\multicolumn{2}{c}{\textbf{Semantic-region performance}} \\
\cmidrule(lr){3-5}
\cmidrule(l){6-7}
Method & Split
& MAE $\downarrow$
& RMSE $\downarrow$
& $R^2$ $\uparrow$
& $q_{\mathrm{obs}}$ $\downarrow$
& Obs. win $\downarrow$ \\
\midrule

Base & $D_{\mathrm r}$
& $16.34 \pm 0.36$
& $32.12 \pm 1.12$
& $0.956 \pm 0.002$
& $0.430 \pm 0.120$
& $0.447 \pm 0.131$ \\

Base & $D_{\mathrm o}$
& $14.53 \pm 1.50$
& $35.92 \pm 3.99$
& $0.897 \pm 0.019$
& $0.562 \pm 0.154$
& $0.589 \pm 0.177$ \\

DR & $D_{\mathrm r}$
& $16.27 \pm 0.48$
& $31.35 \pm 1.62$
& $0.958 \pm 0.003$
& $0.283 \pm 0.206$
& $0.237 \pm 0.287$ \\

DR & $D_{\mathrm o}$
& $13.72 \pm 1.74$
& $31.94 \pm 4.68$
& $0.921 \pm 0.019$
& $0.376 \pm 0.255$
& $0.304 \pm 0.375$ \\

\bottomrule
\end{tabular}
\end{table}

\begin{figure*}[tb]
    \centering
    \includegraphics[width=\linewidth]{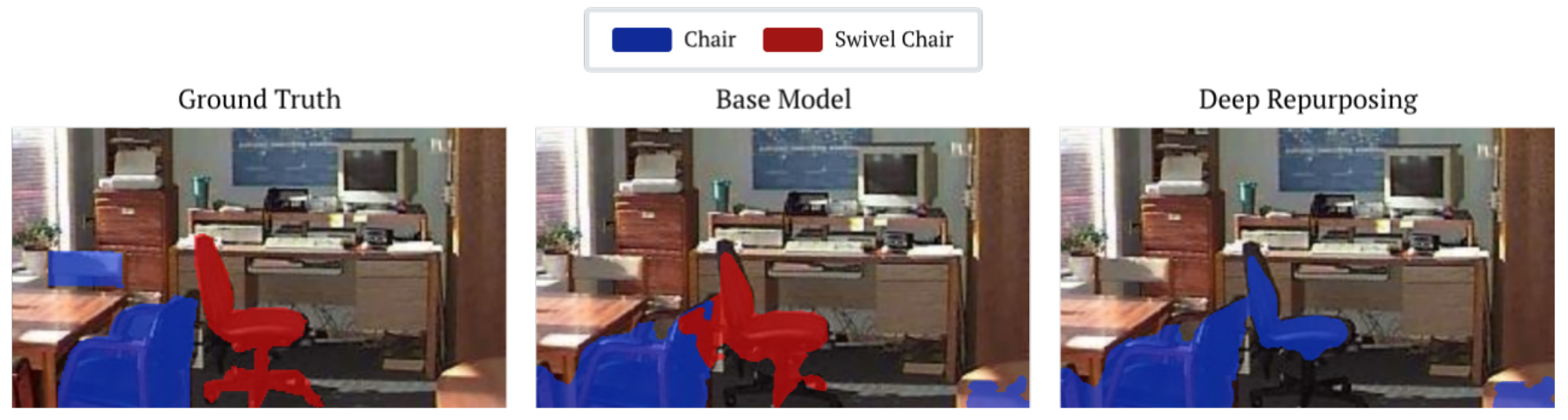}
    \caption{ADE20K segmentation example showing taxonomy-level repurposing from obsolete chair subtypes \{\texttt{armchair}, \texttt{seat}, \texttt{swivel chair}, \texttt{stool}\} to the retained \texttt{chair} class.}
    \label{fig:segmentation_example}
\end{figure*}

Figure~\ref{fig:segmentation_example} provides a qualitative example of taxonomy-level repurposing in semantic segmentation. The ground truth contains both the retained \texttt{chair} class and the obsolete \texttt{swivel chair} subtype. The base model preserves this fine-grained distinction and predicts the obsolete subtype in red. After applying DR, the same region is reassigned to the retained \texttt{chair} class in blue, while the surrounding retained segmentation structure is largely preserved. This illustrates the intended behaviour under task obsolescence: obsolete semantic labels are no longer treated as valid outputs, but their spatial evidence is redirected to the retained taxonomy.

We further validate DR across three real deployment settings (Appendix~\ref{sec:appendix-case-studies}). On a Mars-rover classifier, DR eliminates 100\% of obsolete predictions under both hardware-induced (Table~\ref{tab:case-study1-drt_results}) and taxonomy-induced (Table~\ref{tab:case-study1-horizon_results}) obsolescence, improving retained macro-F1 by +0.097 and +0.095, respectively, in under $\approx15$ seconds. On a deep-sea fish classifier, an uncertainty-triggered variant (Table~\ref{tab:case-study2-uncertainty_results}) similarly removes all obsolete predictions in $\approx7.53$ seconds, with a smaller but consistent retained-accuracy gain (+0.009) but where closely related classes (to the obsolete class) see an accuracy gain of $\approx0.03$ on this harder 83-class task.

\section{Conclusions}
\label{sec:Conclusions}

We introduced Deep Repurposing, a post-hoc framework for adapting models when part of the task space becomes obsolete. Across classification, regression, and segmentation, DR suppresses obsolete behaviour while preserving retained utility, demonstrating that obsolete structure can be redirected rather than simply masked, forgotten, or retrained away. DR relies on a well-specified obsolete region and representative calibration data, with performance depending on latent separability. Future work will extend DR to foundation models, evaluate on multi-modal models, and enable automatic detection of obsolete behaviours.


\bibliography{iclr2027_conference}

@inproceedings{wagstaff2018deep,
  title={Deep mars: Cnn classification of mars imagery for the pds imaging atlas},
  author={Wagstaff, Kiri and Lu, You and Stanboli, Alice and Grimes, Kevin and Gowda, Thamme and Padams, Jordan},
  booktitle={Proceedings of the AAAI Conference on Artificial Intelligence},
  volume={32},
  year={2018}
}

@inproceedings{NIPS2017_3f5ee243,
 author = {Vaswani, Ashish and Shazeer, Noam and Parmar, Niki and Uszkoreit, Jakob and Jones, Llion and Gomez, Aidan N and Kaiser, \L ukasz and Polosukhin, Illia},
 booktitle = {Advances in Neural Information Processing Systems},
 editor = {I. Guyon and U. Von Luxburg and S. Bengio and H. Wallach and R. Fergus and S. Vishwanathan and R. Garnett},
 publisher = {Curran Associates, Inc.},
 title = {Attention is All you Need},
 url = {https://proceedings.neurips.cc/paper_files/paper/2017/file/3f5ee243547dee91fbd053c1c4a845aa-Paper.pdf},
 volume = {30},
 year = {2017}
}

@inproceedings{mcmahan2017communication,
  title={Communication-efficient learning of deep networks from decentralized data},
  author={McMahan, Brendan and Moore, Eider and Ramage, Daniel and Hampson, Seth and y Arcas, Blaise Aguera},
  booktitle={Artificial intelligence and statistics},
  pages={1273--1282},
  year={2017},
  organization={Pmlr}
}

@article{kairouz2021advances,
  title={Advances and open problems in federated learning},
  author={Kairouz, Peter and McMahan, H Brendan},
  journal={Foundations and trends in machine learning},
  volume={14},
  number={1-2},
  pages={1--210},
  year={2021},
  publisher={Emerald Publishing Limited}
}

@article{bonawitz2019towards,
  title={Towards federated learning at scale: System design},
  author={Bonawitz, Keith and Eichner, Hubert and Grieskamp, Wolfgang and Huba, Dzmitry and Ingerman, Alex and Ivanov, Vladimir and Kiddon, Chloe and Kone{\v{c}}n{\`y}, Jakub and Mazzocchi, Stefano and McMahan, Brendan and others},
  journal={Proceedings of machine learning and systems},
  volume={1},
  pages={374--388},
  year={2019}
}

@book{joint2008handbook,
  title={Handbook on constructing composite indicators: methodology and user guide},
  author={Joint Research Centre},
  year={2008},
  publisher={OECD publishing}
}

@article{sculley2015hidden,
  title={Hidden technical debt in machine learning systems},
  author={Sculley, David and Holt, Gary and Golovin, Daniel and Davydov, Eugene and Phillips, Todd and Ebner, Dietmar and Chaudhary, Vinay and Young, Michael and Crespo, Jean-Francois and Dennison, Dan},
  journal={Advances in neural information processing systems},
  volume={28},
  year={2015}
}

@inproceedings{chen2023boundary,
  title={Boundary unlearning: Rapid forgetting of deep networks via shifting the decision boundary},
  author={Chen, Min and Gao, Weizhuo and Liu, Gaoyang and Peng, Kai and Wang, Chen},
  booktitle={Proceedings of the IEEE/CVF Conference on Computer Vision and Pattern Recognition},
  pages={7766--7775},
  year={2023}
}

@inproceedings{foster2024fast,
  title={Fast machine unlearning without retraining through selective synaptic dampening},
  author={Foster, Jack and Schoepf, Stefan and Brintrup, Alexandra},
  booktitle={Proceedings of the AAAI conference on artificial intelligence},
  volume={38},
  pages={12043--12051},
  year={2024}
}

@inproceedings{foster2024loss,
  title={Loss-Free Machine Unlearning},
  author={Foster, Jack and Schoepf, Stefan and Brintrup, Alexandra},
  booktitle={The Second Tiny Papers Track at ICLR 2024},
  year={2024}
}

@article{kim2026exploring,
  title={Exploring and Leveraging Class Vectors for Classifier Editing},
  author={Kim, Jaeik and Do, Jaeyoung},
  journal={Advances in Neural Information Processing Systems},
  volume={38},
  pages={54894--54927},
  year={2026}
}

@article{hakemi2025post,
  title={Post-hoc spurious correlation neutralization with single-weight fictitious class unlearning},
  author={Hakemi, Shahin and Akhtar, Naveed and Mubashar Hassan, Ghulam and Mian, Ajmal},
  journal={arXiv e-prints},
  pages={arXiv--2501},
  year={2025}
}

@inproceedings{golatkar2020eternal,
  title={Eternal sunshine of the spotless net: Selective forgetting in deep networks},
  author={Golatkar, Aditya and Achille, Alessandro and Soatto, Stefano},
  booktitle={Proceedings of the IEEE/CVF conference on computer vision and pattern recognition},
  pages={9304--9312},
  year={2020}
}

@inproceedings{kristiadi2020being,
  title={Being Bayesian, Even Just a Bit, Fixes Overconfidence in {ReLU} Networks},
  author={Kristiadi, Agustinus and Hein, Matthias and Hennig, Philipp},
  booktitle={Proceedings of the 37th International Conference on Machine Learning},
  pages={5436--5446},
  year={2020},
  volume={119},
  publisher={PMLR}
}

@article{krizhevsky2009learning,
  title={Learning multiple layers of features from tiny images},
  author={Krizhevsky, Alex and Hinton, Geoffrey and others},
  year={2009},
  publisher={Toronto, ON, Canada}
}

@inproceedings{deng2009imagenet,
  title={Imagenet: A large-scale hierarchical image database},
  author={Deng, Jia and Dong, Wei and Socher, Richard and Li, Li-Jia and Li, Kai and Fei-Fei, Li},
  booktitle={2009 IEEE conference on computer vision and pattern recognition},
  pages={248--255},
  year={2009},
  organization={Ieee}
}

@inproceedings{zhifei2017cvpr,
  title={Age Progression/Regression by Conditional Adversarial Autoencoder},
  author={Zhang, Zhifei and Song, Yang and Qi, Hairong},
  booktitle={IEEE Conference on Computer Vision and Pattern Recognition (CVPR)},
  year={2017},
  organization={IEEE}
}

@article{fanaee2014event,
  title={Event labeling combining ensemble detectors and background knowledge},
  author={Fanaee-T, Hadi and Gama, Joao},
  journal={Progress in Artificial Intelligence},
  volume={2},
  number={2},
  pages={113--127},
  year={2014},
  publisher={Springer}
}

@inproceedings{zhou2017scene,
  title={Scene parsing through ade20k dataset},
  author={Zhou, Bolei and Zhao, Hang and Puig, Xavier and Fidler, Sanja and Barriuso, Adela and Torralba, Antonio},
  booktitle={Proceedings of the IEEE conference on computer vision and pattern recognition},
  pages={633--641},
  year={2017}
}

@article{zhang2023review,
  title={A review on machine unlearning},
  author={Zhang, Haibo and Nakamura, Toru and Isohara, Takamasa and Sakurai, Kouichi},
  journal={SN Computer Science},
  volume={4},
  number={4},
  pages={337},
  year={2023},
  publisher={Springer}
}

@article{ilharco2022editing,
  title={Editing models with task arithmetic},
  author={Ilharco, Gabriel and Ribeiro, Marco Tulio and Wortsman, Mitchell and Gururangan, Suchin and Schmidt, Ludwig and Hajishirzi, Hannaneh and Farhadi, Ali},
  journal={arXiv preprint arXiv:2212.04089},
  year={2022}
}

@article{mensink2013distance,
  title={Distance-based image classification: Generalizing to new classes at near-zero cost},
  author={Mensink, Thomas and Verbeek, Jakob and Perronnin, Florent and Csurka, Gabriela},
  journal={IEEE transactions on pattern analysis and machine intelligence},
  volume={35},
  number={11},
  pages={2624--2637},
  year={2013},
  publisher={IEEE}
}

@article{snell2017prototypical,
  title={Prototypical networks for few-shot learning},
  author={Snell, Jake and Swersky, Kevin and Zemel, Richard},
  journal={Advances in neural information processing systems},
  volume={30},
  year={2017}
}

@inproceedings{kim2018interpretability,
  title={Interpretability beyond feature attribution: Quantitative testing with concept activation vectors (tcav)},
  author={Kim, Been and Wattenberg, Martin and Gilmer, Justin and Cai, Carrie and Wexler, James and Viegas, Fernanda and others},
  booktitle={International conference on machine learning},
  pages={2668--2677},
  year={2018},
  organization={PMLR}
}

@article{khan2022transformers,
  title={Transformers in vision: A survey},
  author={Khan, Salman and Naseer, Muzammal and Hayat, Munawar and Zamir, Syed Waqas and Khan, Fahad Shahbaz and Shah, Mubarak},
  journal={ACM computing surveys (CSUR)},
  volume={54},
  number={10s},
  pages={1--41},
  year={2022},
  publisher={ACM New York, NY}
}

@inproceedings{rudin2022learning,
  title={Learning to walk in minutes using massively parallel deep reinforcement learning},
  author={Rudin, Nikita and Hoeller, David and Reist, Philipp and Hutter, Marco},
  booktitle={Conference on robot learning},
  pages={91--100},
  year={2022},
  organization={PMLR}
}

@article{cimurs2021goal,
  title={Goal-driven autonomous exploration through deep reinforcement learning},
  author={Cimurs, Reinis and Suh, Il Hong and Lee, Jin Han},
  journal={IEEE Robotics and Automation Letters},
  volume={7},
  number={2},
  pages={730--737},
  year={2021},
  publisher={IEEE}
}

@inproceedings{guo2020certified,
  title={Certified data removal from machine learning models},
  author={Guo, Chuan and Goldstein, Tom and Hannun, Awni and Van Der Maaten, Laurens},
  booktitle={Proceedings of the 37th International Conference on Machine Learning},
  pages={3832--3842},
  year={2020}
}

@article{parisi2019continual,
  title={Continual lifelong learning with neural networks: A review},
  author={Parisi, German I and Kemker, Ronald and Part, Jose L and Kanan, Christopher and Wermter, Stefan},
  journal={Neural networks},
  volume={113},
  pages={54--71},
  year={2019},
  publisher={Elsevier}
}

@inproceedings{khan2023fishnet,
  title={Fishnet: A large-scale dataset and benchmark for fish recognition, detection, and functional trait prediction},
  author={Khan, Faizan Farooq and Li, Xiang and Temple, Andrew J and Elhoseiny, Mohamed},
  booktitle={Proceedings of the IEEE/CVF international conference on computer vision},
  pages={20496--20506},
  year={2023}
}

@inproceedings{tan2019efficientnet,
  title={Efficientnet: Rethinking model scaling for convolutional neural networks},
  author={Tan, Mingxing and Le, Quoc},
  booktitle={International conference on machine learning},
  pages={6105--6114},
  year={2019},
  organization={PMLR}
}

@inproceedings{tan2021efficientnetv2,
  title={Efficientnetv2: Smaller models and faster training},
  author={Tan, Mingxing and Le, Quoc},
  booktitle={International conference on machine learning},
  pages={10096--10106},
  year={2021},
  organization={PMLR}
}

@inproceedings{he2016deep,
  title={Deep residual learning for image recognition},
  author={He, Kaiming and Zhang, Xiangyu and Ren, Shaoqing and Sun, Jian},
  booktitle={Proceedings of the IEEE conference on computer vision and pattern recognition},
  pages={770--778},
  year={2016}
}

@article{dosovitskiy2020image,
  title={An image is worth 16x16 words: Transformers for image recognition at scale},
  author={Dosovitskiy, Alexey and Beyer, Lucas and Kolesnikov, Alexander and Weissenborn, Dirk and Zhai, Xiaohua and Unterthiner, Thomas and Dehghani, Mostafa and Minderer, Matthias and Heigold, Georg and Gelly, Sylvain and others},
  journal={arXiv preprint arXiv:2010.11929},
  year={2020}
}

@InProceedings{Nilsback08,
  author       = "Maria-Elena Nilsback and Andrew Zisserman",
  title        = "Automated Flower Classification over a Large Number of Classes",
  booktitle    = "Indian Conference on Computer Vision, Graphics and Image Processing",
  month        = "Dec",
  year         = "2008",
}

@inproceedings{howard2019searching,
  title={Searching for mobilenetv3},
  author={Howard, Andrew and Sandler, Mark and Chu, Grace and Chen, Liang-Chieh and Chen, Bo and Tan, Mingxing and Wang, Weijun and Zhu, Yukun and Pang, Ruoming and Vasudevan, Vijay and others},
  booktitle={Proceedings of the IEEE/CVF international conference on computer vision},
  pages={1314--1324},
  year={2019}
}

@article{xie2021segformer,
  title={SegFormer: Simple and efficient design for semantic segmentation with transformers},
  author={Xie, Enze and Wang, Wenhai and Yu, Zhiding and Anandkumar, Anima and Alvarez, Jose M and Luo, Ping},
  journal={Advances in neural information processing systems},
  volume={34},
  pages={12077--12090},
  year={2021}
}
\bibliographystyle{iclr2027_conference}


\setlength{\parindent}{0pt}
\setlength{\parskip}{0.5em}
\newpage
\onecolumn
\appendix

\section*{Appendix Contents}
\begin{center}
\begin{minipage}{0.75\textwidth}
Implementation Details \dotfill \pageref{sec:appendix-Implementation Details}\\
Theoretical Analysis \dotfill \pageref{sec:appendix-Theoretical Analysis}\\
Case Studies \dotfill \pageref{sec:appendix-case-studies}\\
\hspace*{1.5em}Case Study 1: Mars Surface Classification \dotfill \pageref{sec:appendix-case-study1}\\
\hspace*{1.5em}Case Study 2: Deep-sea Fish Classification \dotfill \pageref{sec:appendix-case-study2}\\
Additional Experiments \dotfill \pageref{sec:appendix-Additional Experiments}\\
Additional Discussion \dotfill \pageref{sec:appendix-Additional Discussion}\\
Experimental Setup \dotfill \pageref{sec:appendix-Experimental Setup}\\
\end{minipage}
\end{center}
\setcounter{secnumdepth}{1}


\section{Implementation Details}
\label{sec:appendix-Implementation Details}

Algorithm~\ref{alg:deep-repurpose-calibrate} summarises the calibration stage of Deep Repurposing. This stage is performed once before deployment and does not update the parameters of the pretrained model. Instead, the model is used to extract latent representations and original predictions for the calibration set. These latents are then grouped according to the task-region map, which separates obsolete regions from retained regions. For each region, we estimate a diagonal Gaussian prototype, giving a lightweight description of where that region lies in the model's latent space.

\begin{algorithm}[htb]
\caption{\textsc{DeepRepurposing - Calibration}}
\label{alg:deep-repurpose-calibrate}
\begin{algorithmic}[1]
\REQUIRE Model $f_\theta=h_\theta\circ\phi_\theta$; calibration data $\mathcal{D}_{\mathrm{cal}}$; obsolete regions $\mathcal{O}$; retained regions $\mathcal{R}$; task-region map $c$; target map $\psi$
\ENSURE Prototypes $\{\mu_c,\sigma_c^2\}_{c\in\mathcal{O}\cup\mathcal{R}}$; directions $\{v_{o,r},t_{o,r}\}_{o\in\mathcal{O},r\in\mathcal{R}}$; residual prototypes $\{\widetilde{\delta}_a\}_{a\in\mathcal{R}}$; calibration scale $\bar{q}^{\mathrm{cal}}_{\mathcal{O},\mathcal{R}}$; decoder coefficient $\kappa$

\STATE Extract latents $z_i=\phi_\theta(x_i)$ for all $(x_i,y_i)\in\mathcal{D}_{\mathrm{cal}}$.
\STATE Partition calibration samples into regions using the task-region map $c(y_i)$.
\STATE Fit latent prototypes $\{\mu_c,\sigma_c^2\}_{c\in\mathcal{O}\cup\mathcal{R}}$ from the calibration latents.
\STATE Construct obsolete-to-retained removal directions $v_{o,r}$ and transport directions $t_{o,r}$ for all $(o,r)\in\mathcal{O}\times\mathcal{R}$.

\FOR{each retained calibration sample $(x_i,y_i)\in\mathcal{D}^{\mathcal{R}}_{\mathrm{cal}}$}
    \STATE Apply the pre-refinement map $\mathcal{T}_{\mathrm{core}}$ to obtain $q^{\mathrm{core}}_{\mathcal{R}}(x_i)$.
    \STATE Compute the obsolete responsibility $\widetilde{q}_{\mathcal{O}}(x_i)$.
    \STATE Assign the sample to its repaired retained destination
    $a_i=\arg\max_{r\in\mathcal{R}}q^{\mathrm{core}}_r(x_i)$
    \STATE Compute the residual
    $\delta_i^{\mathrm{res}}=\psi(y_i)-q^{\mathrm{core}}_{\mathcal{R}}(x_i)$.
\ENDFOR

\STATE Estimate one reliability-shrunk residual prototype $\widetilde{\delta}_a$ for each retained destination $a\in\mathcal{R}$.
\STATE Compute the retained-background obsolete-responsibility scale $\bar{q}^{\mathrm{cal}}_{\mathcal{O},\mathcal{R}}$ over retained calibration samples.
\STATE If required for regression, select the decoder coefficient $\kappa$ on calibration data.

\RETURN $\{\mu_c,\sigma_c^2\}_{c\in\mathcal{O}\cup\mathcal{R}}$, $\{v_{o,r},t_{o,r}\}_{o\in\mathcal{O},r\in\mathcal{R}}$, $\{\widetilde{\delta}_a\}_{a\in\mathcal{R}}$, $\bar{q}^{\mathrm{cal}}_{\mathcal{O},\mathcal{R}}$, $\kappa$
\end{algorithmic}
\end{algorithm}

Using these fitted prototypes and the score head when available, we construct pairwise obsolete-to-retained directions. The removal direction identifies the component of a representation that supports an obsolete region over a retained alternative, while the repair direction defines how removed evidence may be redirected toward retained regions. The calibration stage also estimates residual refinement vectors on retained calibration samples. These residuals measure the systematic difference between the core repaired retained responsibility and the desired retained target representation. The resulting correction vectors are reliability-shrunk and stored for use at inference. Finally, the calibration stage computes the retained-background obsolete-responsibility scale, which is later used to gate residual refinement, and selects any scalar deployment coefficients required for regression decoding.

Algorithm~\ref{alg:deep-repurpose-infer} gives the corresponding inference procedure. For a new input, DR first computes the latent representation and estimates evidence for each obsolete and retained region. The repair mask is then determined: DR intervenes only when the strongest region evidence corresponds to an obsolete region. This keeps the method targeted, so inputs already assigned to retained regions are passed through with minimal modification.

\begin{algorithm}[htb]
\caption{\textsc{DeepRepurposing - Inference}}
\label{alg:deep-repurpose-infer}
\begin{algorithmic}[1]
\REQUIRE Input $x$; model $f_\theta=h_\theta\circ\phi_\theta$; calibration object $\Lambda_{\mathrm{DR}}$; obsolete regions $\mathcal{O}$; retained regions $\mathcal{R}$
\ENSURE Repaired output $\widehat{y}^{\star}(x)$ and repaired latent representation $z^\star(x)$

\STATE Compute the original latent representation $z=\phi_\theta(x)$ and original output $\widehat{y}_\theta(x)$.
\STATE Estimate region evidence $e_c(z)$ for each $c\in\mathcal{O}\cup\mathcal{R}$ using either the score head or prototype evidence.
\STATE Compute the obsolete responsibility $\widetilde{q}_{\mathcal{O}}(x)$.
\STATE Set the repair mask
\[
\mathcal{M}(x)
=
\mathbf{1}
\left[
\arg\max_{c\in\mathcal{O}\cup\mathcal{R}} e_c(z)\in\mathcal{O}
\right]
\]

\IF{$\mathcal{M}(x)=1$}
    \STATE Select
    $o^\star=\arg\max_{o\in\mathcal{O}}e_o(z)$
    and
    $r^\star=\arg\max_{r\in\mathcal{R}}e_r(z)$.
    \STATE Remove the obsolete-supporting component along $v_{o^\star,r^\star}$ to obtain $z^{-}$.
    \STATE Compute repaired retained evidence $e^{\mathrm{core}}_{\mathcal{R}}(x)$ by combining original retained evidence, retained prototype evidence at $z^{-}$, and contrastive repair evidence.
\ELSE
    \STATE Set $z^{-}=z$ and $e^{\mathrm{core}}_{\mathcal{R}}(x)=e_{\mathcal{R}}(z)$.
\ENDIF

\STATE Normalise over retained regions:
\[
q^{\mathrm{core}}_{\mathcal{R}}(x)
=
\operatorname{softmax}_{\mathcal{R}}
\left(e^{\mathrm{core}}_{\mathcal{R}}(x)\right)
\]

\STATE Apply obsolete-gated residual refinement:
\[
q^\star_{\mathcal{R}}(x)
=
\Pi_{\Delta_{\mathcal{R}}}
\left[
q^{\mathrm{core}}_{\mathcal{R}}(x)
+
g_{\mathrm{obs}}(x)
\sum_{a\in\mathcal{R}}
q^{\mathrm{core}}_a(x)\widetilde{\delta}_a
\right]
\]

\IF{score-based classification}
    \STATE Convert $q^\star_{\mathcal{R}}(x)$ to retained evidence, assign a finite evidence floor to obsolete regions, and compute the repaired prediction $\widehat{y}^{\star}(x)$.
    \STATE Construct $z^\star(x)$ using retained-head consistency and retained-prototype nullspace reconstruction.
\ELSE
    \STATE Decode $q^\star_{\mathcal{R}}(x)$ through the retained value support and combine it with the projected original prediction to obtain $\widehat{y}^{\star}(x)$.
    \STATE Construct $z^\star(x)$ from retained-prototype geometry.
\ENDIF

\RETURN $\widehat{y}^{\star}(x)$ and $z^\star(x)$.
\end{algorithmic}
\end{algorithm}

When the condition is active, DR selects the strongest obsolete region and the strongest retained alternative, removes the obsolete-supporting component from the latent representation, and computes a repaired retained evidence vector. This repaired evidence combines the original retained evidence, retained prototype evidence after stripping, and contrastive repair evidence. The result is normalised over the retained regions to form the core retained responsibility vector. A residual refinement is then applied, gated by the relative obsolete responsibility of the input, and projected back onto the retained simplex.

The final output is constructed according to the task type. For classification, the repaired prediction is normalised over retained classes, and obsolete evidence is either assigned a finite floor or numerically suppressed. For target-region scalar regression, the retained responsibilities are decoded through the retained value support. For semantic-region scalar regression, where regions do not correspond to numeric output intervals, the repaired latent is decoded through the original regression head. The method also constructs a repaired latent representation, either by enforcing retained-head consistency when a score head is available, or by using retained-prototype geometry otherwise. In implementation, $z^{*}_{\mathrm{row}} = W_{\mathcal{R}}^{\top} \left(W_{\mathcal{R}}W_{\mathcal{R}}^{\top}\right)^{-1}\left(e^{*}_{\mathcal{R}}(x)-b_{\mathcal{R}}\right)$ is computed using an SVD-based Moore-Penrose pseudoinverse. This avoids assuming full row rank and allows the reconstruction to be interpreted as an exact minimum-norm score-consistent solution when the retained system is well posed, or as a minimum-norm least-squares projection when it is rank-deficient or overdetermined. Thus, DR repairs both the exposed prediction and the internal representation while leaving the original model parameters fixed.


\section{Theoretical Analysis}
\label{sec:appendix-Theoretical Analysis}

This appendix analyses the structural properties of the DR correction mechanism. In particular, we show that its main geometric components are locally coherent: obsolete-region evidence can be suppressed relative to retained evidence, and the internal transport step used to score retained candidates can be constrained so that it does not undo this correction.

\begin{prop}
For the internal DR transport candidate:
\[
z_r^+ = z-\alpha_{o,r}(x)v_{o,r}+\tau_r(x)t_{o,r},
\qquad \tau_r(x)=\rho(x)\geq 0
\]
used only to score the contrastive gain $G_r(x)$, not the final repaired latent $z^*(x)$, the retained-side margin is preserved whenever prototype transport does not increase the obsolete-retained margin. When this compatibility condition fails, the same margin is preserved by sufficiently capping the transport strength $\tau_r(x)$.
\end{prop}

\begin{proof}
We consider an arbitrary obsolete-retained pair $(o,r)\in O\times R$ and analyse the internal transported representation used by DR to evaluate the contrastive gain $G_r(x)$. Let $z=\phi_\theta(x)$ be the frozen latent representation of an input $x$. In the notation of the DR transport step, the stripped representation and transported candidate are:
\begin{equation}
z^- = z-\alpha_{o,r}(x)v_{o,r}
\end{equation}
and:
\begin{equation}
z_r^+ = z^-+\tau_r(x)t_{o,r},
\qquad
\tau_r(x)=\rho(x) \geq 0
\end{equation}
Equivalently:
\begin{equation}
z_r^+ = z-\alpha_{o,r}(x)v_{o,r}+\tau_r(x)t_{o,r}
\end{equation}
This is the internal candidate used to score transport toward retained region $r$; it is not the final reconstructed latent $z^\star(x)$ returned by the full DR pipeline.

For a score-based region head, the evidence assigned to region $c\in C$ is:
\begin{equation}
e_c(z)=W_c^\top z+b_c
\end{equation}
The quantity that determines whether the obsolete region $o$ is preferred to the retained region $r$ is the obsolete-retained margin:
\begin{equation}
\begin{aligned}
M_{o,r}(z)
&= e_o(z)-e_r(z) \\
&= (W_o-W_r)^\top z + b_o - b_r
\end{aligned}
\end{equation}
Under this convention, $M_{o,r}(z)>0$ means that the obsolete region has larger evidence than the retained region, whereas $M_{o,r}(z)<0$ means that the retained region is preferred. Moreover:
\begin{equation}
\nabla_z M_{o,r}(z)=W_o-W_r
\end{equation}
Thus, the classifier-induced direction $v_{o,r}$ in Equation~\ref{eq:dr-prototype-directions} is the normalised direction of steepest increase of the obsolete-retained margin. Moving in the direction $-v_{o,r}$ therefore decreases obsolete evidence relative to retained evidence.

Let:
\begin{equation}
w_{o,r}=W_o-W_r
\end{equation}
Assuming $w_{o,r} \neq 0$, the classifier-induced stripping direction used by DR is:
\begin{equation}
v_{o,r}
=
\frac{w_{o,r}}{\|w_{o,r}\|_2}
\end{equation}
This is the normal direction of the obsolete-retained decision boundary in latent space. Therefore, moving in the direction $-v_{o,r}$ directly decreases the margin $M_{o,r}$ and suppresses obsolete evidence relative to retained evidence.

Let $\mu_o$ and $\mu_r$ denote the calibrated latent prototypes for regions $o$ and $r$. Assuming $\mu_o\neq \mu_r$, the prototype transport direction is:
\begin{equation}
t_{o,r}
=
\frac{\mu_r-\mu_o}{\|\mu_r-\mu_o\|_2}
\end{equation}
as defined in Equation~\ref{eq:dr-transport-direction}. This direction points from the obsolete prototype toward the retained prototype.

We now examine how the internal transport candidate changes the obsolete-retained margin. Using the linear form of $M_{o,r}$:
\begin{equation}
M_{o,r}(z_r^+)
=
M_{o,r}(z)
-\alpha_{o,r}(x)w_{o,r}^\top v_{o,r}
+\tau_r(x)w_{o,r}^\top t_{o,r}
\end{equation}
Since $v_{o,r}=w_{o,r}/\|w_{o,r}\|_2$, we have:
\begin{equation}
w_{o,r}^\top v_{o,r}
=
\|w_{o,r}\|_2
\end{equation}
Hence:
\begin{equation}
M_{o,r}(z_r^+)
=
M_{o,r}(z)
-\alpha_{o,r}(x)\|w_{o,r}\|_2
+\tau_r(x)w_{o,r}^\top t_{o,r}
\end{equation}
This identity separates the two geometric components of the internal DR transport step. The stripping term $-\alpha_{o,r}(x)\|w_{o,r}\|_2$ always decreases the obsolete-retained margin. The transport term $\tau_r(x)w_{o,r}^\top t_{o,r}$ determines whether movement toward the retained prototype is compatible with the retained side of the classifier boundary.

Suppose the stripping strength is chosen so that:
\begin{equation}
\alpha_{o,r}(x)
\geq
\max\left\{
\frac{M_{o,r}(z)+m}{\|w_{o,r}\|_2},
0
\right\}
\end{equation}
for some desired margin $m\geq 0$. Equivalently, the stripping step alone ensures:
\begin{equation}
M_{o,r}(z)
-
\alpha_{o,r}(x)\|w_{o,r}\|_2
\leq
-m
\end{equation}
Now assume that prototype transport does not increase the obsolete-retained margin, meaning:
\begin{equation}
w_{o,r}^\top t_{o,r}
\leq 0
\end{equation}
Since $\tau_r(x)\geq 0$, it follows that:
\begin{equation}
\tau_r(x)w_{o,r}^\top t_{o,r}
\leq 0
\end{equation}
Substituting this into the exact margin identity gives:
\begin{equation}
M_{o,r}(z_r^+)
\leq
M_{o,r}(z)
-
\alpha_{o,r}(x)\|w_{o,r}\|_2
\leq
-m
\end{equation}
Thus, for the transported candidate $z_r^+$, the retained region $r$ still has evidence at least $m$ larger than the obsolete region $o$. In this case, prototype transport may move the representation toward retained latent structure, but it cannot undo the obsolete-margin removal achieved by the stripping step.

It remains to consider the case where the prototype direction is not margin-compatible, namely when:
\begin{equation}
w_{o,r}^\top t_{o,r} > 0
\end{equation}
Here, moving along $t_{o,r}$ can increase the obsolete-retained margin, so the transport strength must be controlled. From the exact identity:
\begin{equation}
M_{o,r}(z_r^+)
=
M_{o,r}(z)
-
\alpha_{o,r}(x)\|w_{o,r}\|_2
+
\tau_r(x)w_{o,r}^\top t_{o,r}
\end{equation}
the desired retained-side margin $M_{o,r}(z_r^+)\leq -m$ is guaranteed whenever:
\begin{equation}
M_{o,r}(z)
-
\alpha_{o,r}(x)\|w_{o,r}\|_2
+
\tau_r(x)w_{o,r}^\top t_{o,r}
\leq
-m
\end{equation}
Rearranging yields the sufficient cap:
\begin{equation}
\tau_r(x)
\leq
\frac{
\alpha_{o,r}(x)\|w_{o,r}\|_2-M_{o,r}(z)-m
}{
w_{o,r}^\top t_{o,r}
}
\end{equation}
Because the denominator is positive in the present case, this condition limits the amount of margin increase caused by prototype transport. Under this cap, any obsolete-margin increase from the prototype step is no larger than the margin surplus created by the stripping step. Therefore:
\begin{equation}
M_{o,r}(z_r^+)
\leq
-m
\end{equation}
The internal transport step therefore does not require the classifier-induced direction $v_{o,r}$ and the prototype direction $t_{o,r}$ to be collinear. The classifier-induced direction controls the obsolete-retained decision margin, while the prototype direction restores retained latent structure subject to a margin-preservation condition. This proves the claimed margin compatibility of the transport candidate used in the contrastive-gain computation.
\end{proof}

\begin{prop}
Let $|\mathcal{R}|\geq 2$,
$q^{*}_{\mathcal{R}}(x)\in\Delta_{\mathcal{R}}$, and
$\epsilon>0$, with retained scores
$e^{*}_r(x)=\log(q^{*}_r(x)+\epsilon)$. The above margin construction yields
\begin{equation}
e^{*}_o(x)
= 2\min_{r\in\mathcal{R}}e^{*}_r(x) - \log\!\left(1+|\mathcal{R}|\epsilon\right),
\qquad o\in\mathcal{O}
\end{equation}
All scores are finite, and $e^{*}_o(x)<e^{*}_r(x)$ for every $o\in\mathcal{O}$ and $r\in\mathcal{R}$. Consequently, the repaired all-region softmax assigns each obsolete region lower probability than every retained region, so an argmax over $\mathcal{O}\cup\mathcal{R}$ selects a retained region.
\end{prop}

\begin{proof}
For a fixed input $x$, let $a=\min_{r\in\mathcal{R}}e^{*}_r(x)$. The total retained weight is
\begin{equation}
\begin{aligned}
Z &= \sum_{r\in\mathcal{R}}\exp(e^{*}_r(x)) \\ 
&= \sum_{r\in\mathcal{R}}(q^{*}_r(x)+\epsilon) \\ 
&= 1+|\mathcal{R}|\epsilon
\end{aligned}
\end{equation}
The weakest retained region has normalised weight $\exp(a)/Z$. We use this region as the reference so that obsolete outputs remain below every retained alternative. To determine the separation without an additional margin hyperparameter, we choose the obsolete-to-weakest-retained weight ratio to equal $\exp(a)/Z$. This applies an additional suppression factor determined by the retained distribution. Since a score gap $m$ corresponds to a weight ratio $\exp(-m)$, this choice gives the suppression margin:
\begin{equation}
\begin{aligned}
m_{\mathrm{floor}}(x) &= -\log\!\left(\frac{\exp(a)}{Z}\right) \\ &= \log Z-a
\end{aligned}
\end{equation}
Placing each obsolete score this margin below $a$ gives:
\begin{equation}
\begin{aligned}
e^{*}_o(x) &= a-m_{\mathrm{floor}}(x) \\
&= 2a-\log Z
\end{aligned}
\end{equation}
This yields the stated expression and gives
\begin{equation}
\frac{\exp(e^{*}_o(x))}{\exp(a)} = \frac{\exp(a)}{Z} \leq \frac{1}{|\mathcal{R}|}
\end{equation}
The same ratio holds between the corresponding probabilities
after the all-region softmax. Thus, each obsolete output
receives at most $1/|\mathcal{R}|$ of the probability assigned
to the weakest retained output, with stronger suppression
when that retained output has less normalised support. Since $Z\geq|\mathcal{R}|\exp(a)$:
\begin{equation}
m_{\mathrm{floor}}(x) \geq \log|\mathcal{R}|>0
\end{equation}
Therefore, $e^{*}_o(x)<a\leq e^{*}_r(x)$ for every retained region.
The stabiliser $\epsilon>0$ ensures that all scores are
finite. Finally, softmax preserves score ordering, establishing
the claimed probability ordering and argmax guarantee.

Table~\ref{tab:ls_dr_diagnostic} shows that both methods eliminate obsolete predictions, but logit suppression leaves raw obsolete probability at $0.88\pm0.05$. Its effective confidence of $0.69\pm0.05$ therefore accompanies an unrepaired raw output with confidence $0.93\pm0.02$. DR reduces raw obsolete probability to zero at the reported precision, with matching raw and effective confidence of $0.60\pm0.03$ and entropy of $1.12\pm0.08$. This distinction matters when downstream components consume raw scores: logit suppression leaves substantial obsolete support in those scores, whereas the full DR procedure suppresses it within the repaired output itself. This shows that suppression techniques alone can make retained predictions seem more confident than they originally were, whereas DR's derivation preserves the correct confidence.
\end{proof}


\section{Case Studies}
\label{sec:appendix-case-studies}

While the experiments presented in the main body of this paper employ the datasets standard to this area of research, these do not offer insight into whether Deep Repurposing is viable in real operational contexts. The following case studies extend this, examining Deep Repurposing on datasets and scenarios drawn from real-world scenarios. The goal is to assess whether the method holds up under the constraints and ambiguities of genuine deployment conditions, where class relevance shifts organically, data is limited, and retraining from scratch is often not an option.

\subsection{Case Study 1: Mars Surface Classification}
\label{sec:appendix-case-study1}

Autonomous planetary missions provide a natural setting for Deep Repurposing, since the operational relevance of a deployed perception model may change after launch. A rover classifier may originally recognise a broad set of terrain, instrument, and scene categories, but some classes can later become unnecessary, unreliable, or harmful to retained performance. This may occur because a mission phase has ended, a hardware component has degraded, or post-deployment analysis reveals that a class is poorly specified.

Full retraining is not always desirable in this setting due to limited compute, restricted communication, and the need to preserve previously validated behaviour. A lightweight adaptation method is therefore preferable: obsolete output behaviour should be removed while maintaining, or ideally improving, recognition of the remaining mission-relevant classes.

This case study evaluates Deep Repurposing on a Mars rover image classification task based on the MSL dataset~\citep{wagstaff2018deep}. The original classifier is a ViT-B/16 model~\citep{NIPS2017_3f5ee243} trained on the full class taxonomy, which includes rover hardware, surface material, and broader scene categories. The dataset is highly imbalanced, with several visually frequent classes, such as wheel, ground, and horizon, accounting for a large proportion of the samples. This makes macro-F1 an important evaluation metric, since accuracy alone may obscure poor performance on less frequent mission-relevant classes. Example images from the dataset are shown in Figure~\ref{fig:msl_examples}. After training, we simulate mission-phase changes in which one class is no longer desirable as an explicit prediction target, and the remaining classes form the retained mission taxonomy.

\begin{figure}[htb]
    \centering
    \includegraphics[width=0.95\linewidth]{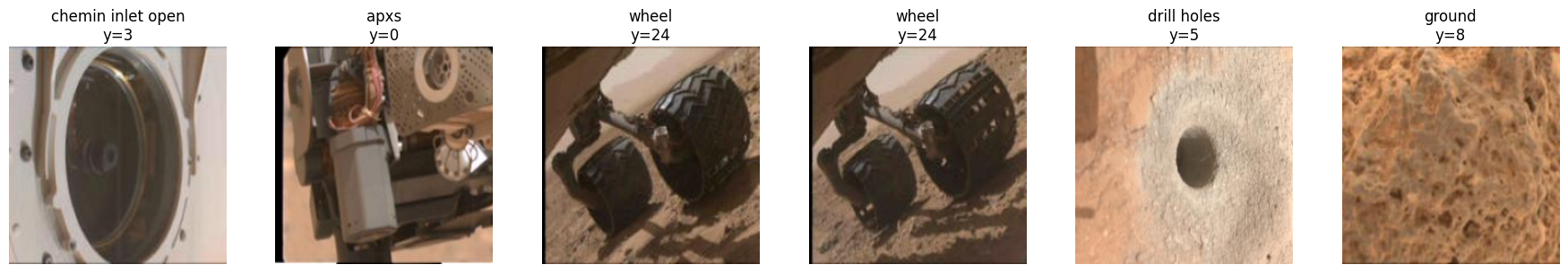}
    \caption{Example images from the MSL rover classification dataset that contains visually diverse mission-relevant categories.}
    \label{fig:msl_examples}
\end{figure}

Figure~\ref{fig:CaseStudy1_overview} summarises the case study. We consider two forms of class obsolescence. First, in the hardware-induced setting, DRT-related views are treated as obsolete, modelling a scenario in which the dust removal tool is damaged, unavailable, or no longer required after a drilling phase. Second, in the taxonomy-induced setting, the horizon class is treated as obsolete because it is weakly learned and harmful to retained macro-F1 performance.

\begin{figure*}[htb]
    \centering
    \includegraphics[width=0.95\linewidth]{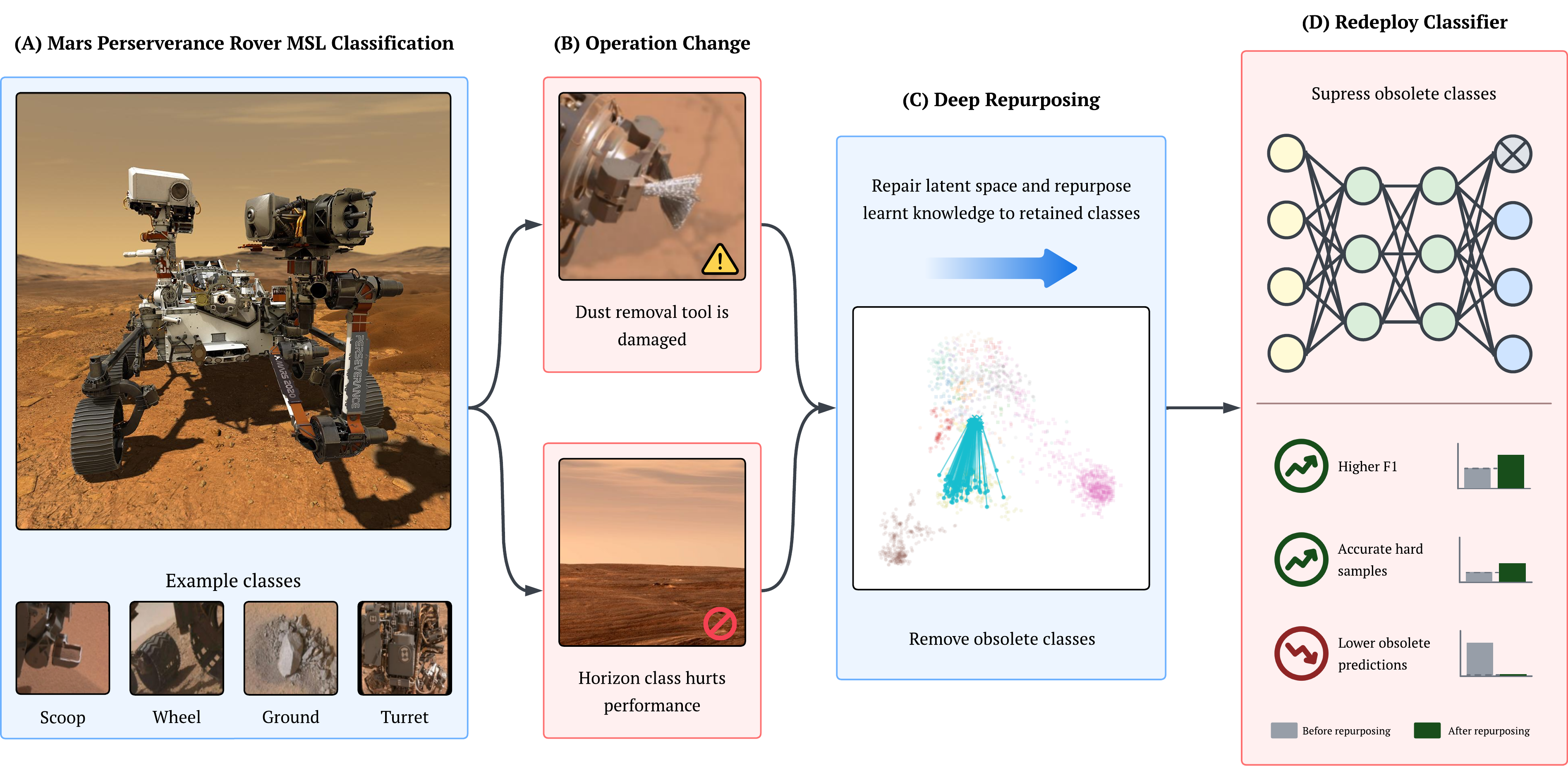}
    \caption{Overview of the Mars rover Deep Repurposing case study. A deployed MSL image classifier is adapted under two mission-phase changes: hardware-induced obsolescence, where a DRT-related class becomes irrelevant after tool damage, retirement, or mission-phase completion; and taxonomy-induced obsolescence, where the horizon class is found to harm retained macro-F1 performance. Deep Repurposing suppresses obsolete output behaviour and reallocates learned structure toward the retained mission classes.}
    \label{fig:CaseStudy1_overview}
\end{figure*}

These variants test two distinct uses of DR. The DRT setting captures mission-driven hardware or objective changes, while the horizon setting captures post-hoc taxonomy repair. Together, they evaluate whether DR can suppress obsolete output behaviour and repurpose learned structure toward the remaining mission-relevant classes.

\subsubsection{Hardware-induced obsolescence}
\label{sec:appendix-case-study1-Hardware-induced-obsolescence}
\hfill\\

In the first variant, we consider a hardware-induced form of class obsolescence. The dust removal tool (DRT) is treated as no longer operationally relevant, modelling a mission setting in which the tool has become damaged, unavailable, detached, or no longer required after a drilling campaign. We therefore retire the DRT-related classes ("DRT front" and "DRT side") from the output taxonomy and apply Deep Repurposing to suppress this obsolete behaviour while preserving the remaining mission-relevant classes.

Table~\ref{tab:case-study1-drt_results} reports the 10-run average performance before and after Deep Repurposing. DR completely removes obsolete DRT predictions on the obsolete distribution, reducing the obsolete prediction rate from $0.672$ to $0.000$. The raw probability assigned to obsolete DRT classes is similarly reduced from $0.398$ to approximately zero. This indicates that the retired DRT output behaviour is not merely weakened, but effectively removed from the classifier. Importantly, this obsolete-class suppression does not degrade retained task performance; Macro-F1 improves from $0.759$ to $0.856$, a gain of $+0.097$. This is particularly important in the Mars rover setting, where retained classes are visually diverse and severely imbalanced: the improvement in macro-F1 suggests that DR improves class-balanced recognition rather than simply preserving majority-class accuracy.

\begin{table}[htb]
\centering
\caption{Results across 10 runs (10 random seeds) for the hardware-induced obsolescence experiment, where DRT-related classes are treated as obsolete.}
\label{tab:case-study1-drt_results}
\begin{tabular}{lccc}
\toprule
Metric & Base Model & DR & $\Delta$ \\
\midrule
$D_r$ accuracy & 0.803 $\pm$ 0.001 & 0.884 $\pm$ 0.001 & +0.081 \\
$D_r$ macro-F1 & 0.759 $\pm$ 0.002 & 0.856 $\pm$ 0.002 & +0.097 \\
$D_r$ NLL & 0.887 $\pm$ 0.002 & 0.579 $\pm$ 0.002 & -0.308 \\
$D_r$ ECE & 0.157 $\pm$ 0.001 & 0.186 $\pm$ 0.001 & +0.029 \\
$D_r$ obs prediction rate & 0.001 $\pm$ 0.000 & 0.000 $\pm$ 0.000 & -0.001 \\
$D_r$ raw obs probability & 0.021 $\pm$ 0.000 & 0.000 $\pm$ 0.000 & -0.021 \\
$D_o$ prediction rate & 0.672 $\pm$ 0.001 & 0.000 $\pm$ 0.000 & -0.672 \\
$D_o$ raw obs probability & 0.398 $\pm$ 0.003 & 0.000 $\pm$ 0.000 & -0.398 \\
\midrule
DR unlearning time & \multicolumn{3}{c}{14.38 $\pm$ 0.40s} \\
\bottomrule
\end{tabular}
\end{table}

The retained predictive quality also improves modestly after repurposing. NLL decreases from $0.887$ to $0.579$, however, ECE increases from $0.157$ to $0.186$.

These results show that DR achieves the desired behaviour in the hardware-induced obsolescence setting: it removes the obsolete DRT output, preserves retained accuracy,  and improves retained macro-F1, all with an average unlearning time of only $14.38 \pm 0.40$ seconds on a T4 GPU.

Figure~\ref{fig:drt_latent_repair} visualises the effect of repurposing on DRT-related samples. Before repair, obsolete DRT samples occupy regions associated primarily with a small number of nearby retained classes, especially portion box, ground, and wheel. After repair, obsolete samples are redistributed toward retained alternatives while their direct obsolete-class probability is suppressed. This supports the intended behaviour of DR: the obsolete class is retired as an output target, but the learned structure associated with it is redirected toward the retained taxonomy rather than discarded.

\begin{figure*}[htb]
    \centering
    \includegraphics[width=\linewidth]{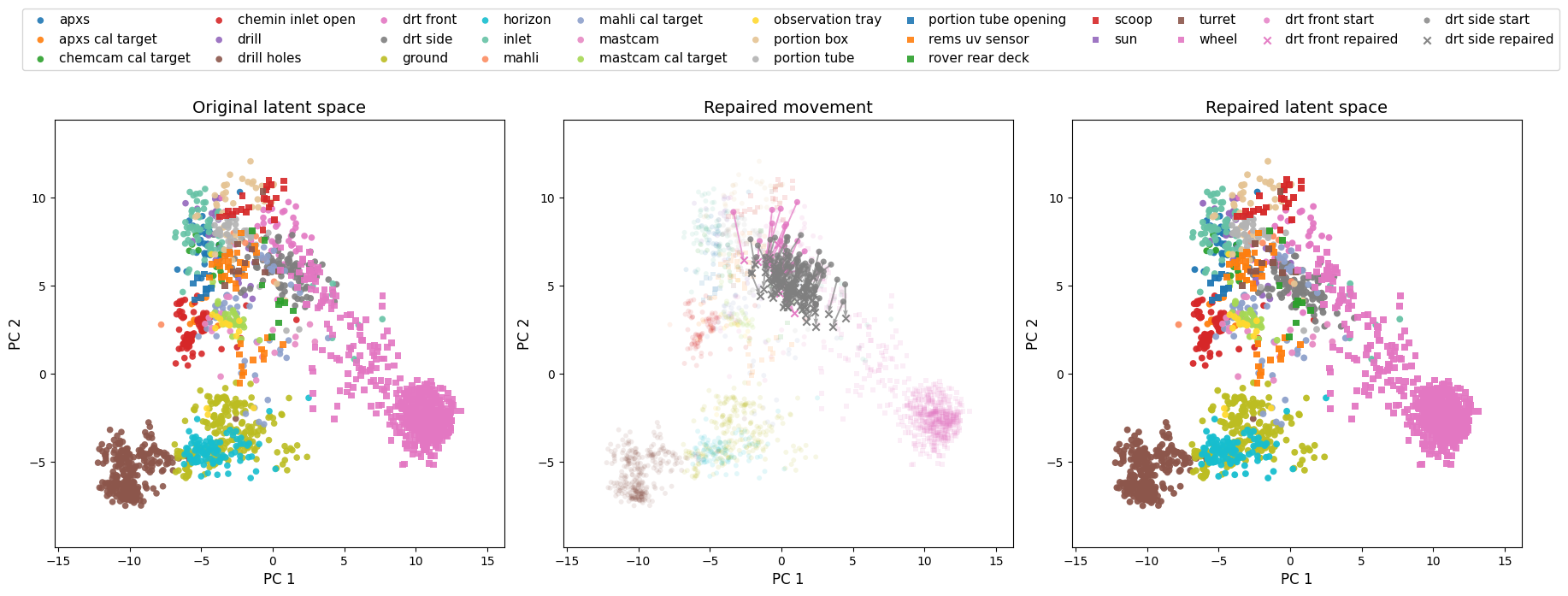}
    \caption{Qualitative visualisation of DRT repurposing from an example random seed. Left: The original representation of all MSL classes. Middle: The repurposing movement from obsolete DRT outputs toward retained alternatives (represented by arrows). Right: The new/repaired representation of all MSL classes (obsolete class still visible for qualitative purposes).}
    \label{fig:drt_latent_repair}
\end{figure*}

To better understand how obsolete DRT samples are redirected after repurposing, we additionally compute the nearest retained-class centroid for each obsolete sample before and after DR. As shown in Table~\ref{tab:drt_nearest_retained}, obsolete DRT samples are initially closest mainly to portion box, ground, and wheel centroids. After repurposing, the distribution becomes less concentrated on heavily saturated ground class (which is one of the most frequent classes) and shifts toward visually related retained hardware or tool classes, including wheel, scoop, and drill. This suggests that DR does not suppress the DRT output arbitrarily; instead, obsolete samples are reassigned toward retained regions that are already close in the model representation. Combined with the reduction in obsolete probability, this supports the interpretation of DR as a repurposing mechanism rather than simple output masking.

\begin{table}[htb]
\centering
\caption{Nearest retained-class distribution for obsolete DRT samples before and after Deep Repurposing from an example random seed.}
\label{tab:drt_nearest_retained}
\begin{tabular}{lcc}
\toprule
Nearest retained class & Before DR & After DR \\
\midrule
Portion box & 71 / 143 (49.7\%) & 54 / 143 (37.8\%) \\
Ground & 42 / 143 (29.4\%) & 20 / 143 (14.0\%) \\
Wheel & 27 / 143 (18.9\%) & 34 / 143 (23.8\%) \\
Scoop & 1 / 143 (0.7\%) & 27 / 143 (18.9\%) \\
Drill & 0 / 143 (0.0\%) & 6 / 143 (4.2\%) \\
MAHLI calibration target & 2 / 143 (1.4\%) & 2 / 143 (1.4\%) \\
\bottomrule
\end{tabular}
\end{table}

Overall, the DRT setting demonstrates the practical value of Deep Repurposing in a mission-adaptation scenario. When a hardware component becomes irrelevant or unavailable, DR can remove the corresponding output behaviour in under $15$ seconds on average while preserving retained accuracy and improving retained macro-F1.

\subsubsection{Taxonomy-induced obsolescence}
\label{sec:appendix-case-study1-Taxonomy-induced-obsolescence}
\hfill\\

In the second variant, we consider a taxonomy-induced form of class obsolescence. Here, the horizon class is treated as obsolete, not because the visual concept is physically unavailable, but because it is weakly learned and harmful to retained classification performance. This setting models a post-deployment revision of the class taxonomy, where mission scientists determine that a class is too ambiguous, over-represented, or insufficiently useful as a stand-alone prediction target. DR is therefore evaluated as a taxonomy repair mechanism: the aim is to remove the horizon output while improving recognition of the remaining mission-relevant classes.

Table~\ref{tab:case-study1-horizon_results} reports the 10-run average performance before and after DR. The base model assigns substantial probability mass to the horizon class, with an obsolete prediction rate of $0.964$ on $D_o$ and a raw obsolete probability of $0.576$. After repurposing, both are reduced to approximately zero, showing that DR effectively removes the horizon output behaviour. Unlike the DRT setting, horizon also leaks into the retained distribution: the base model predicts the obsolete horizon class on $6.3\%$ of retained samples. DR eliminates this leakage entirely, reducing the retained obsolete prediction rate from $0.063$ to $0.000$.

\begin{table}[htb]
\centering
\caption{Results across 10 runs (10 random seeds) for the taxonomy-induced obsolescence experiment, where the horizon class is obsolete.}
\label{tab:case-study1-horizon_results}
\begin{tabular}{lccc}
\toprule
Metric & Base Model & DR & $\Delta$ \\
\midrule
$D_r$ accuracy & 0.784 $\pm$ 0.001 & 0.884 $\pm$ 0.001 & +0.100 \\
$D_r$ macro-F1 & 0.739 $\pm$ 0.002 & 0.834 $\pm$ 0.002 & +0.095 \\
$D_r$ NLL & 0.933 $\pm$ 0.002 & 0.658 $\pm$ 0.002 & -0.274 \\
$D_r$ ECE & 0.151 $\pm$ 0.000 & 0.211 $\pm$ 0.001 & +0.060 \\
$D_r$ obs prediction rate & 0.063 $\pm$ 0.001 & 0.000 $\pm$ 0.000 & -0.063 \\
$D_r$ raw obs probability & 0.029 $\pm$ 0.000 & 0.000 $\pm$ 0.000 & -0.029 \\
$D_o$ obs prediction rate & 0.964 $\pm$ 0.004 & 0.000 $\pm$ 0.000 & -0.964 \\
$D_o$ raw obs probability & 0.576 $\pm$ 0.004 & 0.000 $\pm$ 0.000 & -0.576 \\
\midrule
DR unlearning time & \multicolumn{3}{c}{14.81 $\pm$ 0.37s} \\
\bottomrule
\end{tabular}
\end{table}

This removal leads to a clear improvement in retained performance. Accuracy on $D_r$ increases from $0.784$ to $0.884$, while macro-F1 improves from $0.739$ to $0.834$, corresponding to gains of $+0.100$ and $+0.095$, respectively. The larger macro-F1 gain is particularly important given the imbalance of the MSL dataset, as it indicates that removing the horizon class improves class-balanced recognition rather than only increasing performance on dominant classes. NLL also decreases from $0.933$ to $0.658$.

Figure~\ref{fig:horizon_latent_repair} provides a qualitative view of this behaviour. Horizon samples are concentrated near the ground region of the projected representation, which is consistent with the visual ambiguity between horizon and terrain/background imagery. After repurposing, the obsolete horizon output is suppressed and the affected samples are redirected toward the retained ground-centred region. This supports the interpretation of the horizon experiment as taxonomy repair: DR removes a harmful or ambiguous class and reallocates its behaviour toward a more useful retained category.

\begin{figure*}[htb]
    \centering
    \includegraphics[width=\linewidth]{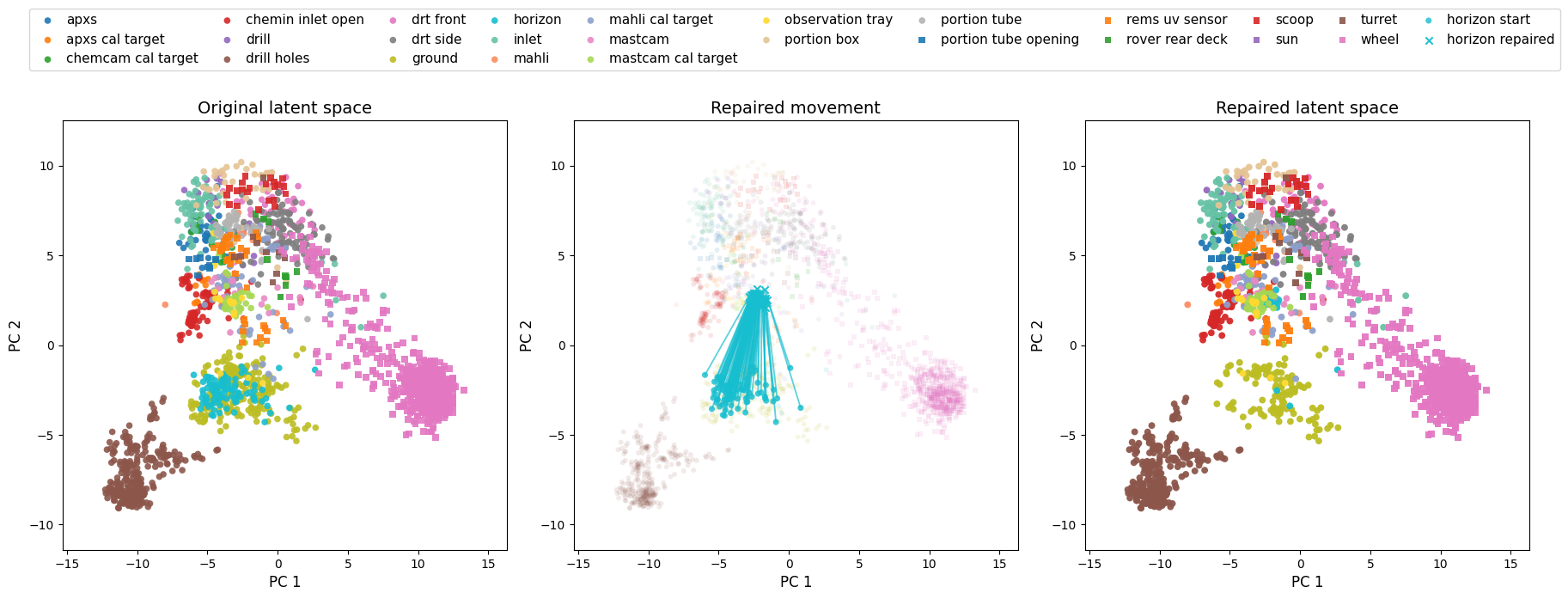}
    \caption{Qualitative visualisation of horizon repurposing from an example random seed. Left: The original representation of all MSL classes. Middle: The repurposing movement from obsolete horizon outputs toward retained alternatives (represented by arrows). Right: The new/repaired representation of all MSL classes (obsolete class still visible for qualitative purposes).}
    \label{fig:horizon_latent_repair}
\end{figure*}

Overall, the horizon setting demonstrates that DR is not limited to hardware-driven class retirement. It can also improve a deployed classifier when the original taxonomy contains a weak or harmful class. By removing the horizon output, DR improves retained accuracy, macro-F1, and NLL while eliminating obsolete predictions in approximately $14.81 \pm 0.37$ seconds.

\subsubsection{Computational efficiency}
\label{sec:appendix-case-study1-Computational-efficiency}
\hfill\\

Across both obsolescence variants, DR provides a lightweight alternative to full fine-tuning of the ViT-B/16 base model. Fine-tuning the ViT-B/16 on the MSL dataset requires substantially greater resources, consuming up to $14.5$GB VRAM on a T4 GPU and exhausting the available memory budget in our budget-hardware setting. In contrast, DR completes the repurposing step within approximately $4.1$GB VRAM, requiring only $14$-$15$ seconds on a T4 GPU across the DRT and horizon experiments. On an A100 GPU, DR requires $6.6$ seconds, compared with $310.7$ seconds for fine-tuning, corresponding to approximately a $47\times$ speedup. This efficiency is important for the intended mission-adaptation setting: DR can suppress obsolete outputs and improve retained performance without the cost of retraining or fine-tuning a large ViT model.

Taken together, the two variants show that DR can adapt a ViT-B/16 rover classifier under both hardware-induced and taxonomy-induced obsolescence. In both cases, DR eliminates obsolete predictions, improves retained macro-F1, and completes adaptation in seconds rather than minutes, supporting its use as a practical post-deployment model adaptation mechanism.

\subsection{Case Study 2: Deep-sea Fish Classification}
\label{sec:appendix-case-study2}

Autonomous underwater vehicles provide a second operational setting in which deployed perception models may need to be adapted after initial training. In this case study, a classifier may be trained to recognise a broad biological taxonomy, but the usefulness of this taxonomy can change once the system is deployed. Some fish orders may be rare, visually ambiguous, poorly represented in the training data, or consistently unreliable under deployment conditions. In such cases, retaining the original output taxonomy may be undesirable: the model may continue to allocate probability mass to a class that is not trusted by domain scientists, while also degrading performance on the remaining operationally useful classes.

\begin{figure}[htb]
    \centering
    \includegraphics[width=0.95\linewidth]{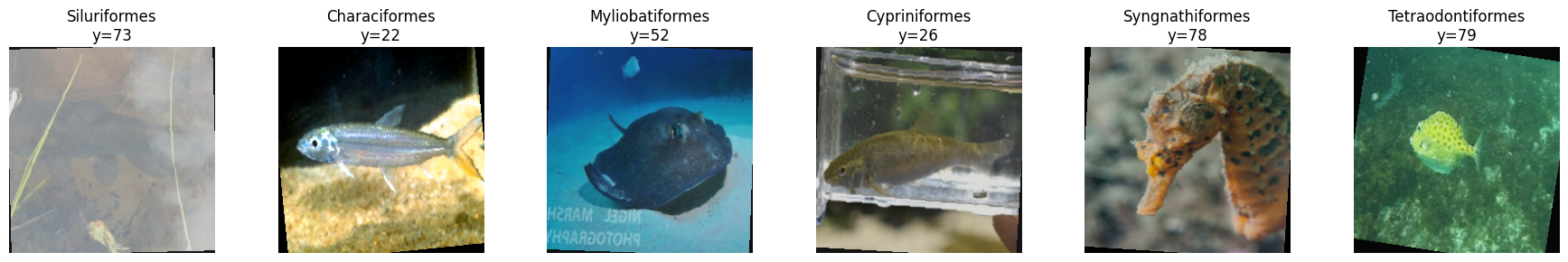}
    \caption{Example images from the FishNet classification dataset that contains visually diverse orders.}
    \label{fig:fishnet_examples}
\end{figure}

This case study evaluates Deep Repurposing on the FishNet classification dataset~\citep{khan2023fishnet}, which contains images from a visually diverse and highly imbalanced marine taxonomy. Example images are shown in Figure~\ref{fig:fishnet_examples}. The original classifier is a MobileNetV3-Large model~\citep{howard2019searching} trained on the full 83-class fish-order taxonomy. As in the Mars case study, the imbalance of the dataset makes macro-F1 an important metric, since accuracy alone may obscure poor recognition of less frequent orders.

\begin{figure*}[htb]
    \centering
    \includegraphics[width=0.95\linewidth]{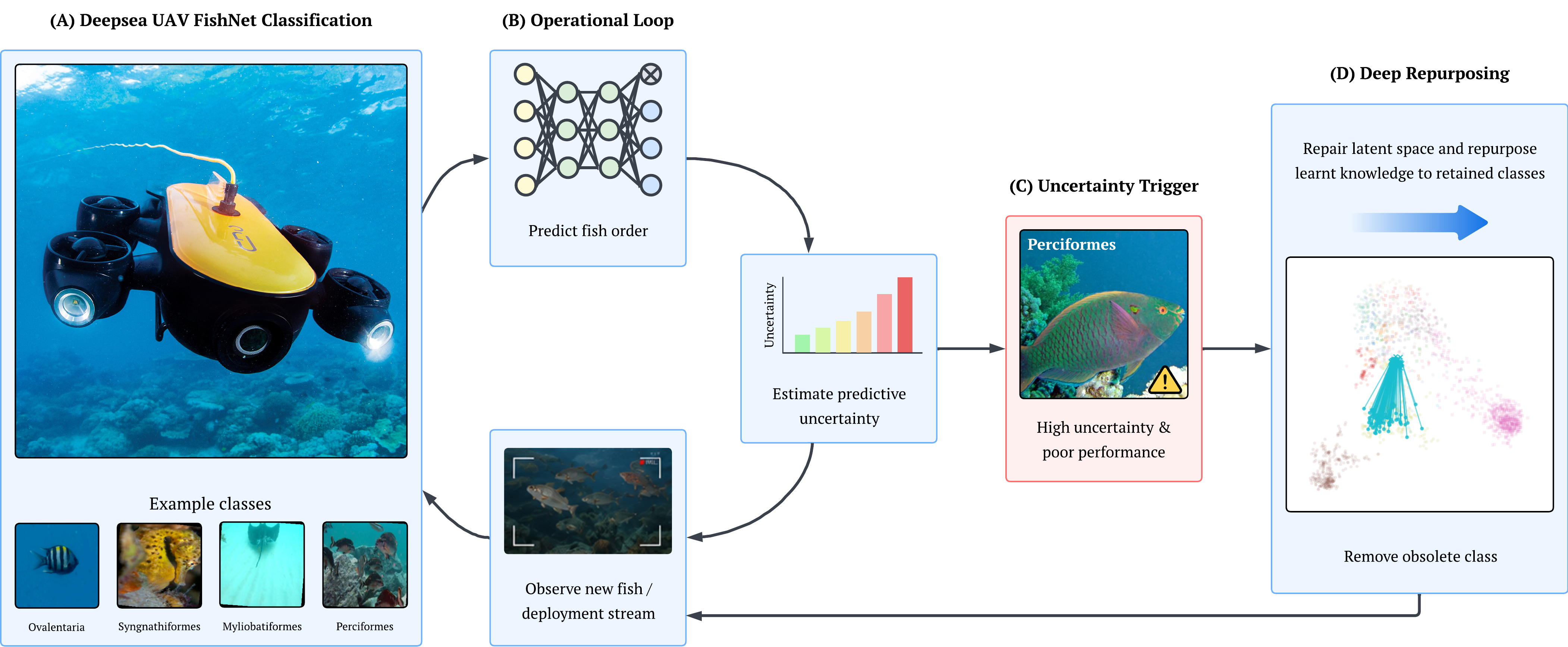}
    \caption{Overview of the uncertainty-triggered FishNet Deep Repurposing case study. During deployment, an autonomous underwater vehicle repeatedly observes new fish imagery, predicts fish orders, and estimates predictive uncertainty. Persistently high uncertainty and poor performance for the 'Perciformes' class triggers Deep Repurposing, which suppresses the obsolete class output and reallocates learned structure toward retained fish orders while preserving retained performance and reducing obsolete predictions.}
    \label{fig:CaseStudy2_overview}
\end{figure*}

\begin{figure}[htb]
    \centering
    \includegraphics[width=0.66\linewidth]{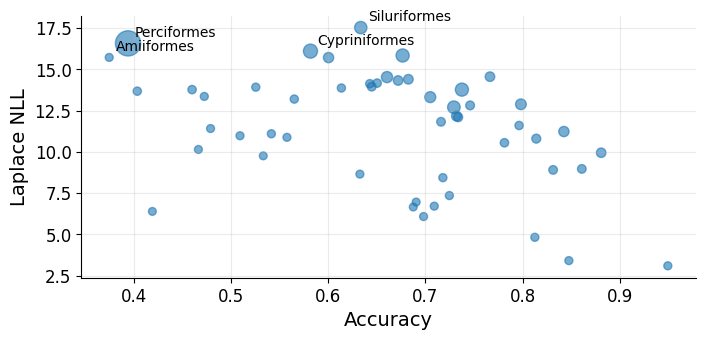}
    \caption{Uncertainty estimates and accuracy for FishNet orders using a last-layer Laplace approximation; larger bubbles indicate orders with more samples.}
    \label{fig:fishnet_uncert_acc_plot}
\end{figure}

Unlike the Mars rover setting, where obsolescence is triggered by an explicit mission-phase or hardware change, this case study considers an uncertainty-triggered adaptation loop. This can be visualised in Figure~\ref{fig:CaseStudy2_overview}. The deployed classifier is paired with a last-layer Laplace uncertainty estimator~\citep{kristiadi2020being}, which provides deployment-time uncertainty estimates for each prediction. During operation, and visualised in Figure~\ref{fig:fishnet_uncert_acc_plot}, scientists observe that the Perciformes class exhibits both consistently poor predictive performance and high uncertainty. Other classes, such as Amiiformes or Siluriformes, also perform poorly, but either exhibit higher accuracy or lower uncertainty than the Perciformes class. Rather than treating this only as a failure of the classifier, the system uses this uncertainty signal to identify Perciformes as a candidate obsolete class. Deep Repurposing is then applied to suppress the unreliable output behaviour associated with this class while preserving the remaining fish-order taxonomy.

Table~\ref{tab:case-study2-uncertainty_results} reports the average performance before and after Deep Repurposing, where Perciformes is treated as obsolete (as identified from Figure~\ref{fig:fishnet_uncert_acc_plot}). DR reduces the obsolete prediction rate on the obsolete distribution from 0.394 to 0.000, while the raw probability assigned to the obsolete class decreases from 0.334 to 0.000. This indicates that DR successfully removes the unreliable Perciformes output rather than merely reducing its confidence. The obsolete class also leaks into the retained distribution before repair: the base model predicts Perciformes on 2.6\% of retained samples, with an average raw obsolete probability of 0.039. After repurposing, both quantities are reduced to zero, showing that DR removes the unreliable class from both obsolete and retained inputs.

\begin{table}[htb]
\centering
\caption{Results across 10 runs (10 random seeds) for the uncertainty-induced obsolescence experiment, where the class with the highest uncertainty and lowest performance (Perciformes) is obsolete.}
\label{tab:case-study2-uncertainty_results}
\begin{tabular}{lccc}
\toprule
Metric & Base Model & DR & $\Delta$ \\
\midrule
$D_r$ accuracy & 0.687 $\pm$ 0.000 & 0.696 $\pm$ 0.000 & +0.009 \\
$D_r$ macro-F1 & 0.579 $\pm$ 0.000 & 0.586 $\pm$ 0.002 & +0.007 \\
$D_r$ NLL & 1.169 $\pm$ 0.0020 & 1.121 $\pm$ 0.000 & -0.048 \\
$D_r$ ECE & 0.074 $\pm$ 0.000 & 0.082 $\pm$ 0.000 & +0.008 \\
$D_r$ obs prediction rate & 0.026 $\pm$ 0.000 & 0.000 $\pm$ 0.000 & -0.026 \\
$D_r$ raw obs probability & 0.039 $\pm$ 0.000 & 0.000 $\pm$ 0.000 & -0.039 \\
$D_o$ obs prediction rate & 0.394 $\pm$ 0.000 & 0.000 $\pm$ 0.000 & -0.394 \\
$D_o$ raw obs probability & 0.334 $\pm$ 0.000 & 0.000 $\pm$ 0.000 & -0.334 \\
\midrule
DR unlearning time & \multicolumn{3}{c}{7.53 $\pm$ 0.20s} \\
\bottomrule
\end{tabular}
\end{table}

The retained task is not degraded by this suppression. Accuracy on $D_r$ increases from 0.687 to 0.696, while macro-F1 improves from 0.579 to 0.586. Although these gains are smaller than those observed in the Mars rover case study, the purpose of this experiment is different. The FishNet setting is intended to demonstrate how uncertainty estimation can be used as a practical trigger for model adaptation. In this setting, the key requirement is that DR removes an unreliable, high-uncertainty class without harming the remaining taxonomy. The reduction in NLL from 1.169 to 1.121 further suggests that repurposing improves the retained classifier. Figure~\ref{fig:fishnet_delta_acc_nll} provides a per-class view of the effect of repurposing. The highlighted obsolete class, Perciformes, is separated from the retained classes because DR intentionally removes its output behaviour. For the remaining fish orders, most classes either maintain similar accuracy or improve modestly, while many also exhibit reductions in NLL. Orders that are visually similar to the obsolete class, Gasterosteiformes and Ovalentaria, exhibit the highest NLL drop and accuracy gains. This supports the intended interpretation of the case study: DR acts as a targeted adaptation mechanism, removing a class identified through deployment-time uncertainty while leaving the broader classifier largely intact. 

\begin{figure}[htb]
    \centering
    \includegraphics[width=0.95\linewidth]{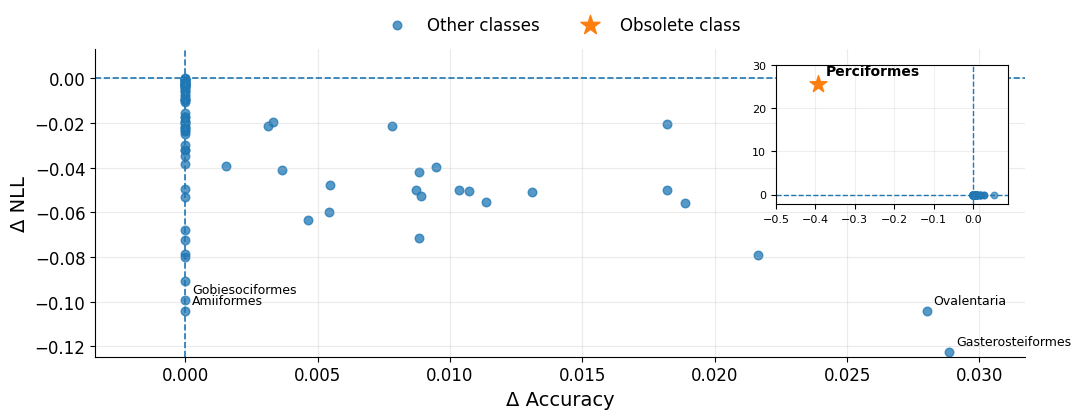}
    \caption{Per-class effect of Deep Repurposing on FishNet after removing the uncertainty-triggered obsolete class (Perciformes). Each point shows the change in accuracy and Laplace NLL for a retained fish order after repurposing, with the dashed lines indicating no change. The inset shows the full-scale view including Perciformes, whose accuracy drops sharply and whose NLL increases substantially after removal, reflecting the intended suppression of the obsolete class.}
    \label{fig:fishnet_delta_acc_nll}
\end{figure}

Overall, the FishNet experiment demonstrates a complementary use case for Deep Repurposing. The Mars rover case study shows that DR can be triggered by mission changes, hardware degradation, or post-hoc taxonomy repair. In contrast, the FishNet case study shows that uncertainty estimation can form part of a closed operational loop: a deployed model identifies a class that is persistently uncertain and poorly performing, scientists mark it as unreliable, and DR rapidly suppresses that output while preserving the remaining task. On an A40 GPU, DR completes this adaptation in $7.53 \pm 0.20$ seconds, supporting its use as a lightweight post-deployment repair mechanism for imbalanced real-world classification systems.


\section{Additional Experiments}
\label{sec:appendix-Additional Experiments}

In this section we provide additional experimental results supporting the main evaluation. We report the full results, calibration and confidence diagnostics, representation-level behaviour, retained-set performance across difficulty levels, and extended regression and segmentation results. We further compare DR against simple logit suppression and provide ablations of its main components, calibration-set size, and residual refinement.

\subsection{Full Results}
\label{sec:appendix-Full Results}

Table~\ref{tab:appendix_core_results} reports the full classification results seen in Figures~\ref{fig:classification_do_dr_acc} and~\ref{fig:classification_obs_rate_prob}.

\begin{table*}[tb]
\centering
\caption{Full classification results across all datasets reporting retained performance, obsolete-output suppression, and repair time.}
\label{tab:appendix_core_results}
\scriptsize
\setlength{\tabcolsep}{5pt}
\resizebox{\textwidth}{!}{%
\begin{tabular}{lcccc|cccc|c}
\toprule
& \multicolumn{4}{c|}{$\mathbf{D_{\mathrm{r}}}$}
& \multicolumn{4}{c|}{$\mathbf{D_{\mathrm{o}}}$}
& \textbf{Time} \\
\cmidrule(lr){2-5}\cmidrule(lr){6-9}\cmidrule(lr){10-10}
\textbf{Method}
& \textbf{Acc.} $\uparrow$
& \textbf{F1} $\uparrow$
& \textbf{Obs. Rate} $\downarrow$
& \textbf{Obs. Prob.} $\downarrow$
& \textbf{Acc.} $\downarrow$
& \textbf{F1} $\downarrow$
& \textbf{Obs. Rate} $\downarrow$
& \textbf{Obs. Prob.} $\downarrow$
& \textbf{Unlearn (s)} $\downarrow$ \\
\midrule

\multicolumn{10}{c}{\textbf{CIFAR10}} \\
\midrule
Base Model                 & 0.90 $\pm$ 0.00 & 0.82 $\pm$ 0.01 & 0.01 $\pm$ 0.00 & 0.01 $\pm$ 0.00 & 0.90 $\pm$ 0.04 & 0.12 $\pm$ 0.03 & 0.90 $\pm$ 0.04 & 0.88 $\pm$ 0.05 & -- \\
SSD                        & 0.88 $\pm$ 0.01 & 0.82 $\pm$ 0.01 & 0.00 $\pm$ 0.00 & 0.01 $\pm$ 0.00 & 0.36 $\pm$ 0.15 & 0.05 $\pm$ 0.02 & 0.36 $\pm$ 0.15 & 0.33 $\pm$ 0.13 & 2.32 $\pm$ 0.49 \\
LFSSD                      & 0.88 $\pm$ 0.01 & 0.83 $\pm$ 0.02 & 0.00 $\pm$ 0.00 & 0.01 $\pm$ 0.00 & 0.40 $\pm$ 0.20 & 0.05 $\pm$ 0.02 & 0.40 $\pm$ 0.20 & 0.36 $\pm$ 0.19 & 4.30 $\pm$ 1.34 \\
Boundary Shrink            & 0.79 $\pm$ 0.03 & 0.76 $\pm$ 0.03 & 0.00 $\pm$ 0.00 & 0.00 $\pm$ 0.00 & 0.07 $\pm$ 0.01 & 0.01 $\pm$ 0.00 & 0.07 $\pm$ 0.01 & 0.07 $\pm$ 0.01 & 13.12 $\pm$ 0.31 \\
Boundary Expand            & 0.78 $\pm$ 0.02 & 0.73 $\pm$ 0.02 & 0.00 $\pm$ 0.00 & 0.00 $\pm$ 0.00 & 0.13 $\pm$ 0.01 & 0.02 $\pm$ 0.00 & 0.13 $\pm$ 0.01 & 0.12 $\pm$ 0.01 & 19.58 $\pm$ 0.16 \\
Fictitious Class Removal   & 0.75 $\pm$ 0.05 & 0.74 $\pm$ 0.03 & 0.19 $\pm$ 0.05 & 0.19 $\pm$ 0.05 & 0.09 $\pm$ 0.03 & 0.02 $\pm$ 0.00 & 0.09 $\pm$ 0.03 & 0.08 $\pm$ 0.03 & 13.01 $\pm$ 0.20 \\
Class Vector               & 0.90 $\pm$ 0.00 & 0.81 $\pm$ 0.00 & 0.00 $\pm$ 0.00 & 0.01 $\pm$ 0.00 & 0.17 $\pm$ 0.09 & 0.03 $\pm$ 0.01 & 0.17 $\pm$ 0.09 & 0.15 $\pm$ 0.03 & 10.83 $\pm$ 0.12 \\
Fisher Forgetting          & 0.82 $\pm$ 0.03 & 0.82 $\pm$ 0.03 & 0.00 $\pm$ 0.00 & 0.00 $\pm$ 0.00 & 0.01 $\pm$ 0.01 & 0.00 $\pm$ 0.00 & 0.01 $\pm$ 0.01 & 0.01 $\pm$ 0.01 & 18.68 $\pm$ 0.32 \\
Finetune                   & 0.92 $\pm$ 0.01 & 0.90 $\pm$ 0.02 & 0.00 $\pm$ 0.00 & 0.00 $\pm$ 0.00 & 0.14 $\pm$ 0.07 & 0.02 $\pm$ 0.01 & 0.14 $\pm$ 0.07 & 0.14 $\pm$ 0.06 & 78.70 $\pm$ 0.36 \\
Retrain                    & 0.91 $\pm$ 0.01 & 0.91 $\pm$ 0.01 & 0.00 $\pm$ 0.00 & 0.00 $\pm$ 0.00 & 0.00 $\pm$ 0.00 & 0.00 $\pm$ 0.00 & 0.00 $\pm$ 0.00 & 0.00 $\pm$ 0.00 & 162.17 $\pm$ 1.12 \\
Deep Repurposing           & 0.91 $\pm$ 0.01 & 0.91 $\pm$ 0.01 & 0.00 $\pm$ 0.00 & 0.00 $\pm$ 0.00 & 0.00 $\pm$ 0.00 & 0.00 $\pm$ 0.00 & 0.00 $\pm$ 0.00 & 0.00 $\pm$ 0.00 & 1.62 $\pm$ 0.04 \\

\midrule
\multicolumn{10}{c}{\textbf{CIFAR100}} \\
\midrule
Base Model                 & 0.80 $\pm$ 0.00 & 0.79 $\pm$ 0.00 & 0.00 $\pm$ 0.00 & 0.01 $\pm$ 0.00 & 0.78 $\pm$ 0.21 & 0.13 $\pm$ 0.09 & 0.78 $\pm$ 0.21 & 0.67 $\pm$ 0.20 & -- \\
SSD                        & 0.70 $\pm$ 0.16 & 0.70 $\pm$ 0.16 & 0.00 $\pm$ 0.00 & 0.00 $\pm$ 0.00 & 0.00 $\pm$ 0.00 & 0.00 $\pm$ 0.00 & 0.00 $\pm$ 0.00 & 0.01 $\pm$ 0.00 & 3.53 $\pm$ 0.67 \\
LFSSD                      & 0.79 $\pm$ 0.01 & 0.79 $\pm$ 0.01 & 0.00 $\pm$ 0.00 & 0.00 $\pm$ 0.00 & 0.01 $\pm$ 0.03 & 0.00 $\pm$ 0.01 & 0.01 $\pm$ 0.03 & 0.02 $\pm$ 0.02 & 3.59 $\pm$ 1.34 \\
Boundary Shrink            & 0.62 $\pm$ 0.07 & 0.62 $\pm$ 0.07 & 0.00 $\pm$ 0.00 & 0.01 $\pm$ 0.00 & 0.05 $\pm$ 0.03 & 0.00 $\pm$ 0.00 & 0.05 $\pm$ 0.03 & 0.02 $\pm$ 0.00 & 2.94 $\pm$ 0.13 \\
Boundary Expand            & 0.62 $\pm$ 0.07 & 0.62 $\pm$ 0.07 & 0.00 $\pm$ 0.00 & 0.00 $\pm$ 0.00 & 0.05 $\pm$ 0.03 & 0.00 $\pm$ 0.00 & 0.05 $\pm$ 0.03 & 0.02 $\pm$ 0.01 & 4.27 $\pm$ 0.07 \\
Fictitious Class Removal   & 0.32 $\pm$ 0.12 & 0.40 $\pm$ 0.13 & 0.59 $\pm$ 0.15 & 0.57 $\pm$ 0.16 & 0.07 $\pm$ 0.08 & 0.01 $\pm$ 0.02 & 0.07 $\pm$ 0.08 & 0.07 $\pm$ 0.08 & 29.53 $\pm$ 0.15 \\
Class Vector               & 0.80 $\pm$ 0.00 & 0.79 $\pm$ 0.00 & 0.00 $\pm$ 0.00 & 0.01 $\pm$ 0.00 & 0.66 $\pm$ 0.24 & 0.09 $\pm$ 0.08 & 0.66 $\pm$ 0.24 & 0.07 $\pm$ 0.01 & 12.89 $\pm$ 0.30 \\
Fisher Forgetting          & 0.77 $\pm$ 0.01 & 0.77 $\pm$ 0.01 & 0.00 $\pm$ 0.00 & 0.00 $\pm$ 0.00 & 0.17 $\pm$ 0.23 & 0.03 $\pm$ 0.05 & 0.17 $\pm$ 0.23 & 0.11 $\pm$ 0.14 & 37.01 $\pm$ 0.41 \\
Finetune                   & 0.79 $\pm$ 0.00 & 0.79 $\pm$ 0.01 & 0.00 $\pm$ 0.00 & 0.00 $\pm$ 0.00 & 0.28 $\pm$ 0.28 & 0.03 $\pm$ 0.05 & 0.28 $\pm$ 0.28 & 0.26 $\pm$ 0.25 & 184.58 $\pm$ 2.63 \\
Retrain                    & 0.79 $\pm$ 0.00 & 0.79 $\pm$ 0.00 & 0.00 $\pm$ 0.00 & 0.00 $\pm$ 0.00 & 0.00 $\pm$ 0.00 & 0.00 $\pm$ 0.00 & 0.00 $\pm$ 0.00 & 0.00 $\pm$ 0.00 & 546.37 $\pm$ 9.79 \\
Deep Repurposing           & 0.80 $\pm$ 0.00 & 0.80 $\pm$ 0.00 & 0.00 $\pm$ 0.00 & 0.00 $\pm$ 0.00 & 0.00 $\pm$ 0.00 & 0.00 $\pm$ 0.00 & 0.00 $\pm$ 0.00 & 0.00 $\pm$ 0.00 & 8.25 $\pm$ 0.24 \\

\midrule
\multicolumn{10}{c}{\textbf{Flowers102}} \\
\midrule
Base Model                 & 0.87 $\pm$ 0.00 & 0.86 $\pm$ 0.00 & 0.00 $\pm$ 0.00 & 0.01 $\pm$ 0.00 & 0.82 $\pm$ 0.26 & 0.43 $\pm$ 0.47 & 0.82 $\pm$ 0.26 & 0.55 $\pm$ 0.30 & -- \\
SSD                        & 0.28 $\pm$ 0.43 & 0.27 $\pm$ 0.43 & 0.00 $\pm$ 0.00 & 0.01 $\pm$ 0.00 & 0.00 $\pm$ 0.00 & 0.00 $\pm$ 0.00 & 0.00 $\pm$ 0.00 & 0.01 $\pm$ 0.01 & 97.80 $\pm$ 0.96 \\
LFSSD                      & 0.83 $\pm$ 0.04 & 0.82 $\pm$ 0.04 & 0.00 $\pm$ 0.00 & 0.00 $\pm$ 0.01 & 0.51 $\pm$ 0.31 & 0.08 $\pm$ 0.07 & 0.51 $\pm$ 0.31 & 0.18 $\pm$ 0.17 & 101.10 $\pm$ 1.04 \\
Boundary Shrink            & 0.79 $\pm$ 0.02 & 0.78 $\pm$ 0.02 & 0.00 $\pm$ 0.00 & 0.00 $\pm$ 0.00 & 0.23 $\pm$ 0.23 & 0.03 $\pm$ 0.04 & 0.23 $\pm$ 0.23 & 0.06 $\pm$ 0.05 & 1.86 $\pm$ 0.19 \\
Boundary Expand            & 0.80 $\pm$ 0.02 & 0.79 $\pm$ 0.02 & 0.00 $\pm$ 0.00 & 0.00 $\pm$ 0.00 & 0.26 $\pm$ 0.24 & 0.03 $\pm$ 0.04 & 0.26 $\pm$ 0.24 & 0.07 $\pm$ 0.06 & 2.33 $\pm$ 0.28 \\
Fictitious Class Removal   & 0.62 $\pm$ 0.08 & 0.70 $\pm$ 0.07 & 0.30 $\pm$ 0.08 & 0.26 $\pm$ 0.09 & 0.11 $\pm$ 0.10 & 0.01 $\pm$ 0.01 & 0.11 $\pm$ 0.10 & 0.09 $\pm$ 0.08 & 3.35 $\pm$ 0.12 \\
Class Vector               & 0.87 $\pm$ 0.00 & 0.85 $\pm$ 0.00 & 0.00 $\pm$ 0.00 & 0.00 $\pm$ 0.00 & 0.76 $\pm$ 0.29 & 0.39 $\pm$ 0.50 & 0.76 $\pm$ 0.29 & 0.42 $\pm$ 0.30 & 7.75 $\pm$ 0.04 \\
Fisher Forgetting          & 0.84 $\pm$ 0.02 & 0.83 $\pm$ 0.02 & 0.00 $\pm$ 0.00 & 0.00 $\pm$ 0.00 & 0.33 $\pm$ 0.53 & 0.29 $\pm$ 0.47 & 0.33 $\pm$ 0.53 & 0.14 $\pm$ 0.22 & 15.38 $\pm$ 3.62 \\
Finetune                   & 0.89 $\pm$ 0.00 & 0.89 $\pm$ 0.00 & 0.00 $\pm$ 0.00 & 0.00 $\pm$ 0.00 & 0.78 $\pm$ 0.26 & 0.35 $\pm$ 0.42 & 0.78 $\pm$ 0.26 & 0.70 $\pm$ 0.29 & 56.36 $\pm$ 1.16 \\
Retrain                    & 0.91 $\pm$ 0.00 & 0.91 $\pm$ 0.00 & 0.00 $\pm$ 0.00 & 0.00 $\pm$ 0.00 & 0.00 $\pm$ 0.00 & 0.00 $\pm$ 0.00 & 0.00 $\pm$ 0.00 & 0.00 $\pm$ 0.00 & 113.74 $\pm$ 2.04 \\
Deep Repurposing           & 0.88 $\pm$ 0.00 & 0.87 $\pm$ 0.00 & 0.00 $\pm$ 0.00 & 0.00 $\pm$ 0.00 & 0.00 $\pm$ 0.00 & 0.00 $\pm$ 0.00 & 0.00 $\pm$ 0.00 & 0.00 $\pm$ 0.00 & 3.33 $\pm$ 0.28 \\

\midrule
\multicolumn{10}{c}{\textbf{ImageNet}} \\
\midrule
Base Model                 & 0.81 $\pm$ 0.00 & 0.80 $\pm$ 0.00 & 0.00 $\pm$ 0.00 & 0.00 $\pm$ 0.00 & 0.68 $\pm$ 0.29 & 0.14 $\pm$ 0.14 & 0.68 $\pm$ 0.29 & 0.33 $\pm$ 0.18 & -- \\
SSD                        & 0.02 $\pm$ 0.06 & 0.03 $\pm$ 0.06 & 0.00 $\pm$ 0.00 & 0.00 $\pm$ 0.00 & 0.00 $\pm$ 0.00 & 0.00 $\pm$ 0.00 & 0.00 $\pm$ 0.00 & 0.00 $\pm$ 0.00 & 1367.01 $\pm$ 66.71 \\
LFSSD                      & 0.80 $\pm$ 0.00 & 0.80 $\pm$ 0.00 & 0.00 $\pm$ 0.00 & 0.00 $\pm$ 0.00 & 0.00 $\pm$ 0.00 & 0.00 $\pm$ 0.00 & 0.00 $\pm$ 0.00 & 0.00 $\pm$ 0.00 & 1804.29 $\pm$ 75.08 \\
Boundary Shrink            & 0.76 $\pm$ 0.01 & 0.76 $\pm$ 0.01 & 0.00 $\pm$ 0.00 & 0.00 $\pm$ 0.00 & 0.24 $\pm$ 0.18 & 0.02 $\pm$ 0.02 & 0.24 $\pm$ 0.18 & 0.05 $\pm$ 0.02 & 0.81 $\pm$ 0.28 \\
Boundary Expand            & 0.76 $\pm$ 0.02 & 0.76 $\pm$ 0.01 & 0.00 $\pm$ 0.00 & 0.00 $\pm$ 0.00 & 0.25 $\pm$ 0.18 & 0.02 $\pm$ 0.01 & 0.25 $\pm$ 0.18 & 0.05 $\pm$ 0.02 & 1.41 $\pm$ 0.52 \\
Fictitious Class Removal   & 0.56 $\pm$ 0.22 & 0.60 $\pm$ 0.19 & 0.31 $\pm$ 0.28 & 0.30 $\pm$ 0.27 & 0.54 $\pm$ 0.48 & 0.53 $\pm$ 0.49 & 0.54 $\pm$ 0.48 & 0.53 $\pm$ 0.49 & 104.07 $\pm$ 81.55 \\
Class Vector               & 0.80 $\pm$ 0.00 & 0.79 $\pm$ 0.00 & 0.00 $\pm$ 0.00 & 0.00 $\pm$ 0.00 & 0.66 $\pm$ 0.28 & 0.17 $\pm$ 0.18 & 0.66 $\pm$ 0.28 & 0.14 $\pm$ 0.11 & 136.25 $\pm$ 64.34 \\
Fisher Forgetting          & 0.74 $\pm$ 0.01 & 0.74 $\pm$ 0.01 & 0.00 $\pm$ 0.00 & 0.00 $\pm$ 0.00 & 0.26 $\pm$ 0.32 & 0.03 $\pm$ 0.06 & 0.26 $\pm$ 0.32 & 0.06 $\pm$ 0.08 & 121.92 $\pm$ 117.96 \\
Finetune                   & 0.78 $\pm$ 0.00 & 0.78 $\pm$ 0.00 & 0.00 $\pm$ 0.00 & 0.00 $\pm$ 0.00 & 0.40 $\pm$ 0.30 & 0.05 $\pm$ 0.06 & 0.40 $\pm$ 0.30 & 0.37 $\pm$ 0.27 & 670.90 $\pm$ 404.28 \\
Retrain                    & 0.58 $\pm$ 0.00 & 0.58 $\pm$ 0.00 & 0.00 $\pm$ 0.00 & 0.00 $\pm$ 0.00 & 0.00 $\pm$ 0.00 & 0.00 $\pm$ 0.00 & 0.00 $\pm$ 0.00 & 0.00 $\pm$ 0.00 & 1167.15 $\pm$ 48.04 \\
Deep Repurposing           & 0.81 $\pm$ 0.00 & 0.80 $\pm$ 0.00 & 0.00 $\pm$ 0.00 & 0.00 $\pm$ 0.00 & 0.00 $\pm$ 0.00 & 0.00 $\pm$ 0.00 & 0.00 $\pm$ 0.00 & 0.00 $\pm$ 0.00 & 28.82 $\pm$ 44.07 \\

\bottomrule
\end{tabular}%
}
\end{table*}

Across datasets, DR achieves the most consistent repair results across all state-of-the-art comparative methods. It preserves retained utility at or above the base model level while reducing both obsolete prediction rate and obsolete probability to $0.00$ on $D_{\mathrm{o}}$. For example, on CIFAR10, DR matches retraining the full model in retained accuracy and F1, removes obsolete behaviour completely, and requires only $1.62$ seconds on average compared with $162.17$ seconds on average for retraining the model. The same pattern holds on CIFAR100, where DR preserves the base model's retained accuracy, slightly improves F1, and removes obsolete outputs in $8.25$ seconds, while retraining requires $546.37$ seconds on average. We also see that SSD and LFSSD results vary highly between datasets, showing how unstable they may be when unlearning a specific class.

Table~\ref{tab:appendix_calibration_results} reports calibration diagnostics for the classification experiments. On the retained split, DR can improve the calibration profile of the base model while removing obsolete behaviour. For CIFAR10, DR matches the base model's retained ECE at $0.03$ and improves retained NLL by $-0.03$ on average. On CIFAR100, DR again matches the base model's retained NLL and ECE, while methods such as SSD, boundary shifting, fictitious class removal, and class-vector editing introduce noticeably larger calibration errors.

\begin{table*}[htb]
\centering
\caption{Classification calibration results across all datasets reporting negative log-likelihood (NLL) and expected calibration error (ECE) on retained and obsolete test splits.}
\label{tab:appendix_calibration_results}
\scriptsize
\setlength{\tabcolsep}{3.0pt}
\resizebox{\textwidth}{!}{%
\begin{tabular}{lcccc|cccc}
\toprule
& \multicolumn{4}{c|}{\textbf{CIFAR10}}
& \multicolumn{4}{c}{\textbf{CIFAR100}} \\
\cmidrule(lr){2-5}\cmidrule(lr){6-9}
\textbf{Method}
& $\mathbf{D_r}$ NLL $\downarrow$
& $\mathbf{D_r}$ ECE $\downarrow$
& $\mathbf{D_o}$ NLL $\downarrow$
& $\mathbf{D_o}$ ECE $\downarrow$
& $\mathbf{D_r}$ NLL $\downarrow$
& $\mathbf{D_r}$ ECE $\downarrow$
& $\mathbf{D_o}$ NLL $\downarrow$
& $\mathbf{D_o}$ ECE $\downarrow$ \\
\midrule
Base Model               & 0.30 $\pm$ 0.01 & 0.03 $\pm$ 0.00 & 0.30 $\pm$ 0.13 & 0.04 $\pm$ 0.02 & 0.80 $\pm$ 0.01 & 0.03 $\pm$ 0.01 & 0.83 $\pm$ 0.64 & 0.14 $\pm$ 0.06 \\
SSD                      & 0.36 $\pm$ 0.04 & 0.02 $\pm$ 0.00 & 2.07 $\pm$ 0.62 & 0.32 $\pm$ 0.10 & 1.51 $\pm$ 0.82 & 0.21 $\pm$ 0.07 & 5.73 $\pm$ 0.62 & 0.16 $\pm$ 0.05 \\
LFSSD                    & 0.36 $\pm$ 0.04 & 0.03 $\pm$ 0.00 & 2.25 $\pm$ 0.92 & 0.32 $\pm$ 0.14 & 0.86 $\pm$ 0.04 & 0.07 $\pm$ 0.03 & 4.61 $\pm$ 0.39 & 0.35 $\pm$ 0.10 \\
Boundary Shrink          & 0.72 $\pm$ 0.13 & 0.08 $\pm$ 0.03 & 4.82 $\pm$ 0.23 & 0.60 $\pm$ 0.02 & 2.12 $\pm$ 0.25 & 0.35 $\pm$ 0.05 & 4.78 $\pm$ 0.07 & 0.10 $\pm$ 0.04 \\
Boundary Expand          & 0.67 $\pm$ 0.06 & 0.02 $\pm$ 0.00 & 3.59 $\pm$ 0.16 & 0.43 $\pm$ 0.01 & 1.81 $\pm$ 0.29 & 0.19 $\pm$ 0.03 & 4.98 $\pm$ 0.05 & 0.16 $\pm$ 0.03 \\
Fictitious Class Removal & 0.93 $\pm$ 0.27 & 0.15 $\pm$ 0.05 & 120.38 $\pm$ 16.31 & 0.62 $\pm$ 0.04 & 6.87 $\pm$ 6.68 & 0.53 $\pm$ 0.20 & 123.66 $\pm$ 117.85 & 0.25 $\pm$ 0.15 \\
Class Vector             & 0.35 $\pm$ 0.01 & 0.04 $\pm$ 0.01 & 2.09 $\pm$ 0.17 & 0.24 $\pm$ 0.04 & 1.14 $\pm$ 0.14 & 0.28 $\pm$ 0.09 & 2.87 $\pm$ 0.13 & 0.60 $\pm$ 0.23 \\
Fisher Forgetting        & 0.57 $\pm$ 0.11 & 0.07 $\pm$ 0.02 & 8.48 $\pm$ 0.76 & 0.77 $\pm$ 0.02 & 0.92 $\pm$ 0.04 & 0.05 $\pm$ 0.01 & 4.98 $\pm$ 2.14 & 0.40 $\pm$ 0.17 \\
Finetune                 & 0.26 $\pm$ 0.02 & 0.03 $\pm$ 0.00 & 3.94 $\pm$ 0.83 & 0.62 $\pm$ 0.10 & 0.89 $\pm$ 0.01 & 0.11 $\pm$ 0.00 & 4.37 $\pm$ 2.60 & 0.50 $\pm$ 0.26 \\
Retrain                  & 0.27 $\pm$ 0.02 & 0.03 $\pm$ 0.00 & $\infty$ & 0.81 $\pm$ 0.03 & 0.76 $\pm$ 0.01 & 0.48 $\pm$ 0.01 & $\infty$ & 0.70 $\pm$ 0.05 \\
Deep Repurposing         & 0.27 $\pm$ 0.02 & 0.03 $\pm$ 0.00 & 17.50 $\pm$ 2.65 & 0.60 $\pm$ 0.03 & 0.80 $\pm$ 0.01 & 0.03 $\pm$ 0.00 & 22.05 $\pm$ 6.43 & 0.29 $\pm$ 0.20 \\

\midrule
& \multicolumn{4}{c|}{\textbf{Flowers102}}
& \multicolumn{4}{c}{\textbf{ImageNet}} \\
\cmidrule(lr){2-5}\cmidrule(lr){6-9}
\textbf{Method}
& $\mathbf{D_r}$ NLL $\downarrow$
& $\mathbf{D_r}$ ECE $\downarrow$
& $\mathbf{D_o}$ NLL $\downarrow$
& $\mathbf{D_o}$ ECE $\downarrow$
& $\mathbf{D_r}$ NLL $\downarrow$
& $\mathbf{D_r}$ ECE $\downarrow$
& $\mathbf{D_o}$ NLL $\downarrow$
& $\mathbf{D_o}$ ECE $\downarrow$ \\
\midrule
Base Model               & 0.86 $\pm$ 0.00 & 0.30 $\pm$ 0.00 & 0.93 $\pm$ 0.85 & 0.25 $\pm$ 0.13 & 1.40 $\pm$ 0.00 & 0.41 $\pm$ 0.00 & 1.90 $\pm$ 1.12 & 0.33 $\pm$ 0.16 \\
SSD                      & 10.04 $\pm$ 11.36 & 0.47 $\pm$ 0.43 & 5.55 $\pm$ 1.53 & 0.36 $\pm$ 0.48 & 12.06 $\pm$ 8.44 & 0.40 $\pm$ 0.17 & 12.76 $\pm$ 7.90 & 0.42 $\pm$ 0.15 \\
LFSSD                    & 1.16 $\pm$ 0.28 & 0.37 $\pm$ 0.05 & 2.30 $\pm$ 0.90 & 0.33 $\pm$ 0.15 & 1.41 $\pm$ 0.00 & 0.41 $\pm$ 0.00 & 7.62 $\pm$ 0.39 & 0.19 $\pm$ 0.03 \\
Boundary Shrink          & 1.60 $\pm$ 0.13 & 0.47 $\pm$ 0.03 & 3.82 $\pm$ 1.34 & 0.19 $\pm$ 0.13 & 1.77 $\pm$ 0.21 & 0.44 $\pm$ 0.04 & 4.57 $\pm$ 0.84 & 0.20 $\pm$ 0.13 \\
Boundary Expand          & 1.46 $\pm$ 0.10 & 0.44 $\pm$ 0.03 & 3.74 $\pm$ 1.39 & 0.20 $\pm$ 0.13 & 1.82 $\pm$ 0.20 & 0.45 $\pm$ 0.04 & 4.56 $\pm$ 0.80 & 0.20 $\pm$ 0.14 \\
Fictitious Class Removal & 1.80 $\pm$ 0.79 & 0.28 $\pm$ 0.13 & 42.48 $\pm$ 28.54 & 0.14 $\pm$ 0.02 & 13.38 $\pm$ 9.79 & 0.56 $\pm$ 0.13 & 8.05 $\pm$ 8.56 & 0.05 $\pm$ 0.06 \\
Class Vector             & 1.03 $\pm$ 0.01 & 0.37 $\pm$ 0.01 & 1.32 $\pm$ 1.02 & 0.32 $\pm$ 0.13 & 2.62 $\pm$ 0.14 & 0.66 $\pm$ 0.02 & 3.13 $\pm$ 1.31 & 0.50 $\pm$ 0.19 \\
Fisher Forgetting        & 1.24 $\pm$ 0.12 & 0.41 $\pm$ 0.02 & 3.91 $\pm$ 2.46 & 0.27 $\pm$ 0.26 & 1.77 $\pm$ 0.01 & 0.41 $\pm$ 0.01 & 4.57 $\pm$ 1.84 & 0.24 $\pm$ 0.19 \\
Finetune                 & 0.40 $\pm$ 0.00 & 0.02 $\pm$ 0.00 & 0.76 $\pm$ 0.86 & 0.13 $\pm$ 0.11 & 0.95 $\pm$ 0.00 & 0.10 $\pm$ 0.00 & 3.56 $\pm$ 2.47 & 0.36 $\pm$ 0.20 \\
Retrain                  & 0.37 $\pm$ 0.01 & 0.05 $\pm$ 0.00 & $\infty$ & 0.49 $\pm$ 0.11 & 1.85 $\pm$ 0.00 & 0.04 $\pm$ 0.00 & $\infty$ & 0.32 $\pm$ 0.06 \\
Deep Repurposing         & 0.76 $\pm$ 0.01 & 0.26 $\pm$ 0.01 & 35.50 $\pm$ 0.02 & 0.19 $\pm$ 0.02 & 1.40 $\pm$ 0.00 & 0.41 $\pm$ 0.00 & 29.48 $\pm$ 7.41 & 0.66 $\pm$ 0.20 \\
\bottomrule
\end{tabular}%
}
\end{table*}

The comparative methods show that obsolete suppression often comes with calibration instability. SSD and LFSSD can produce large retained NLL or ECE in some settings, particularly on CIFAR100 and Flowers102, while boundary-based methods tend to increase retained uncertainty as a by-product of moving the decision boundary.

It should be noted that the obsolete split should be interpreted differently from the retained split. Since $D_{\mathrm{o}}$ is evaluated using the original labels, a method that correctly removes the obsolete class can receive a high NLL because the original target is no longer treated as valid. This explains the large $D_{\mathrm{o}}$ NLL values for DR and the infinite NLL values for retraining (since the model is retrained without the obsolete output nodes). Therefore, $D_{\mathrm{o}}$ calibration is most informative when read alongside the obsolete prediction rates from previous results.

Table~\ref{tab:appendix_conf_entr_results} reports confidence and entropy results for the classification experiments. These results complement Table~\ref{tab:appendix_core_results} by showing whether obsolete suppression is achieved with sensible predictive estimates, rather than only by changing the final class decision. On retained data, DR generally keeps the prediction distribution close to the base model. For example, on CIFAR10, DR has retained confidence/entropy of $0.94/0.18$, compared with $0.93/0.19$ for the base model, and on CIFAR100, it remains close at $0.81/0.90$ compared with $0.79/1.11$ on average. This suggests that DR does not suppress obsolete classes by globally flattening the classifier.

\begin{table*}[htb]
\centering
\caption{Confidence and predictive entropy results across all classification datasets on retained and obsolete test splits.}
\label{tab:appendix_conf_entr_results}
\scriptsize
\setlength{\tabcolsep}{3.0pt}

\begin{tabular}{lcccc|cccc}
\toprule
& \multicolumn{4}{c|}{\textbf{CIFAR10}}
& \multicolumn{4}{c}{\textbf{CIFAR100}} \\
\cmidrule(lr){2-5}\cmidrule(lr){6-9}
\textbf{Method}
& $\mathbf{D_r}$ Conf. $\downarrow$
& $\mathbf{D_r}$ Ent. $\downarrow$
& $\mathbf{D_o}$ Conf. $\downarrow$
& $\mathbf{D_o}$ Ent. $\downarrow$
& $\mathbf{D_r}$ Conf. $\downarrow$
& $\mathbf{D_r}$ Ent. $\downarrow$
& $\mathbf{D_o}$ Conf. $\downarrow$
& $\mathbf{D_o}$ Ent. $\downarrow$ \\
\midrule
Base Model               & 0.93 $\pm$ 0.00 & 0.19 $\pm$ 0.01 & 0.93 $\pm$ 0.02 & 0.19 $\pm$ 0.07 & 0.79 $\pm$ 0.01 & 1.11 $\pm$ 0.00 & 0.78 $\pm$ 0.09 & 1.09 $\pm$ 0.22 \\
SSD                      & 0.90 $\pm$ 0.02 & 0.31 $\pm$ 0.05 & 0.64 $\pm$ 0.07 & 1.05 $\pm$ 0.22 & 0.48 $\pm$ 0.21 & 2.60 $\pm$ 0.85 & 0.16 $\pm$ 0.04 & 3.92 $\pm$ 0.19 \\
LFSSD                    & 0.91 $\pm$ 0.01 & 0.26 $\pm$ 0.03 & 0.70 $\pm$ 0.08 & 0.86 $\pm$ 0.22 & 0.73 $\pm$ 0.03 & 1.42 $\pm$ 0.17 & 0.36 $\pm$ 0.10 & 2.90 $\pm$ 0.55 \\
Boundary Shrink          & 0.86 $\pm$ 0.01 & 0.40 $\pm$ 0.03 & 0.67 $\pm$ 0.01 & 0.97 $\pm$ 0.04 & 0.27 $\pm$ 0.03 & 3.65 $\pm$ 0.11 & 0.14 $\pm$ 0.02 & 4.05 $\pm$ 0.06 \\
Boundary Expand          & 0.80 $\pm$ 0.02 & 0.63 $\pm$ 0.05 & 0.55 $\pm$ 0.01 & 1.31 $\pm$ 0.02 & 0.43 $\pm$ 0.04 & 2.87 $\pm$ 0.16 & 0.20 $\pm$ 0.01 & 3.70 $\pm$ 0.06 \\
Fictitious Class Removal & 0.90 $\pm$ 0.01 & 0.26 $\pm$ 0.02 & 0.71 $\pm$ 0.05 & 0.86 $\pm$ 0.14 & 0.85 $\pm$ 0.08 & 0.62 $\pm$ 0.28 & 0.32 $\pm$ 0.23 & 3.18 $\pm$ 1.08 \\
Class Vector             & 0.88 $\pm$ 0.01 & 0.39 $\pm$ 0.02 & 0.28 $\pm$ 0.04 & 2.03 $\pm$ 0.08 & 0.52 $\pm$ 0.09 & 2.60 $\pm$ 0.50 & 0.11 $\pm$ 0.03 & 4.24 $\pm$ 0.14 \\
Fisher Forgetting        & 0.89 $\pm$ 0.01 & 0.31 $\pm$ 0.03 & 0.78 $\pm$ 0.03 & 0.61 $\pm$ 0.07 & 0.73 $\pm$ 0.02 & 1.40 $\pm$ 0.07 & 0.52 $\pm$ 0.08 & 2.22 $\pm$ 0.42 \\
Finetune                 & 0.95 $\pm$ 0.00 & 0.13 $\pm$ 0.01 & 0.76 $\pm$ 0.04 & 0.66 $\pm$ 0.10 & 0.90 $\pm$ 0.00 & 0.31 $\pm$ 0.01 & 0.77 $\pm$ 0.06 & 0.75 $\pm$ 0.21 \\
Retrain                  & 0.94 $\pm$ 0.00 & 0.18 $\pm$ 0.01 & 0.81 $\pm$ 0.03 & 0.53 $\pm$ 0.09 & 0.86 $\pm$ 0.00 & 0.48 $\pm$ 0.01 & 0.70 $\pm$ 0.05 & 0.99 $\pm$ 0.20 \\
Deep Repurposing         & 0.94 $\pm$ 0.00 & 0.18 $\pm$ 0.01 & 0.59 $\pm$ 0.05 & 1.12 $\pm$ 0.08 & 0.81 $\pm$ 0.01 & 0.90 $\pm$ 0.05 & 0.29 $\pm$ 0.20 & 3.25 $\pm$ 0.97 \\

\midrule
& \multicolumn{4}{c|}{\textbf{Flowers102}}
& \multicolumn{4}{c}{\textbf{ImageNet}} \\
\cmidrule(lr){2-5}\cmidrule(lr){6-9}
\textbf{Method}
& $\mathbf{D_r}$ Conf. $\downarrow$
& $\mathbf{D_r}$ Ent. $\downarrow$
& $\mathbf{D_o}$ Conf. $\downarrow$
& $\mathbf{D_o}$ Ent. $\downarrow$
& $\mathbf{D_r}$ Conf. $\downarrow$
& $\mathbf{D_r}$ Ent. $\downarrow$
& $\mathbf{D_o}$ Conf. $\downarrow$
& $\mathbf{D_o}$ Ent. $\downarrow$ \\
\midrule
Base Model               & 0.57 $\pm$ 0.00 & 2.13 $\pm$ 0.01 & 0.57 $\pm$ 0.28 & 2.06 $\pm$ 1.10 & 0.40 $\pm$ 0.00 & 4.39 $\pm$ 0.00 & 0.39 $\pm$ 0.13 & 4.21 $\pm$ 0.71 \\
SSD                      & 0.44 $\pm$ 0.43 & 2.65 $\pm$ 2.02 & 0.36 $\pm$ 0.48 & 2.95 $\pm$ 2.17 & 0.41 $\pm$ 0.16 & 3.35 $\pm$ 1.30 & 0.42 $\pm$ 0.15 & 3.25 $\pm$ 1.20 \\
LFSSD                    & 0.47 $\pm$ 0.09 & 2.60 $\pm$ 0.40 & 0.23 $\pm$ 0.16 & 3.47 $\pm$ 0.65 & 0.40 $\pm$ 0.00 & 4.37 $\pm$ 0.01 & 0.19 $\pm$ 0.03 & 5.17 $\pm$ 0.23 \\
Boundary Shrink          & 0.31 $\pm$ 0.04 & 3.30 $\pm$ 0.16 & 0.12 $\pm$ 0.03 & 3.95 $\pm$ 0.13 & 0.33 $\pm$ 0.05 & 4.86 $\pm$ 0.37 & 0.11 $\pm$ 0.02 & 5.98 $\pm$ 0.20 \\
Boundary Expand          & 0.36 $\pm$ 0.03 & 3.07 $\pm$ 0.13 & 0.13 $\pm$ 0.04 & 3.86 $\pm$ 0.16 & 0.31 $\pm$ 0.05 & 4.97 $\pm$ 0.34 & 0.10 $\pm$ 0.02 & 6.13 $\pm$ 0.18 \\
Fictitious Class Removal & 0.63 $\pm$ 0.05 & 1.72 $\pm$ 0.23 & 0.23 $\pm$ 0.07 & 3.39 $\pm$ 0.33 & 0.57 $\pm$ 0.16 & 3.12 $\pm$ 1.16 & 0.59 $\pm$ 0.44 & 2.79 $\pm$ 2.98 \\
Class Vector             & 0.49 $\pm$ 0.01 & 2.52 $\pm$ 0.03 & 0.44 $\pm$ 0.28 & 2.67 $\pm$ 1.10 & 0.13 $\pm$ 0.01 & 6.21 $\pm$ 0.10 & 0.15 $\pm$ 0.11 & 6.00 $\pm$ 0.58 \\
Fisher Forgetting        & 0.43 $\pm$ 0.04 & 2.80 $\pm$ 0.16 & 0.21 $\pm$ 0.16 & 3.66 $\pm$ 0.52 & 0.33 $\pm$ 0.01 & 4.73 $\pm$ 0.07 & 0.16 $\pm$ 0.04 & 5.64 $\pm$ 0.22 \\
Finetune                 & 0.88 $\pm$ 0.00 & 0.46 $\pm$ 0.00 & 0.78 $\pm$ 0.19 & 0.81 $\pm$ 0.65 & 0.87 $\pm$ 0.00 & 0.40 $\pm$ 0.00 & 0.72 $\pm$ 0.07 & 0.97 $\pm$ 0.26 \\
Retrain                  & 0.86 $\pm$ 0.00 & 0.68 $\pm$ 0.01 & 0.49 $\pm$ 0.11 & 2.06 $\pm$ 0.56 & 0.54 $\pm$ 0.00 & 2.06 $\pm$ 0.01 & 0.32 $\pm$ 0.06 & 2.97 $\pm$ 0.36 \\
Deep Repurposing         & 0.62 $\pm$ 0.01 & 1.77 $\pm$ 0.04 & 0.19 $\pm$ 0.02 & 3.28 $\pm$ 0.15 & 0.40 $\pm$ 0.00 & 4.39 $\pm$ 0.00 & 0.66 $\pm$ 0.20 & 2.18 $\pm$ 1.33 \\
\bottomrule
\end{tabular}
\end{table*}

On obsolete data, DR usually converts highly confident obsolete predictions into lower-confidence retained alternatives. On CIFAR10, the base model remains highly confident on $D_{\mathrm{o}}$ achieving $0.93$ confidence and $0.19$ entropy on average, while DR reduces confidence to $0.59$ and increases entropy to $1.12$. The same pattern is stronger on CIFAR100, where DR changes $D_{\mathrm{o}}$ confidence/entropy from $0.78/1.09$ to $0.29/3.25$. This indicates that obsolete samples are no longer treated as confident members of the retired class, but are redistributed across retained alternatives.

The comparative methods show why this diagnostic is useful as Fine-tuning often preserves sharp predictions on obsolete samples, for example, $0.76/0.66$ on CIFAR10 and $0.78/0.81$ on Flowers102, consistent with its incomplete obsolete suppression in Table~\ref{tab:appendix_core_results}. The boundary-based methods can reduce obsolete confidence, such as $0.14/4.05$ on CIFAR100 for Boundary Shrink, but this coincides with much lower retained accuracy in the full predictive results. Overall, DR provides the most balanced behaviour: retained predictions remain usable, while obsolete samples lose the high-confidence profile of the base model.

Table~\ref{tab:appendix_representation_results} reports representation-level results, testing whether obsolete behaviour remains visible in the latent space after repair. This is important as a model may have obsolete predictions suppressed, but the system in which the model is being used utilises some latent space metrics for its decision-making. We report four diagnostics. Membership inference attack (MIA) denotes the area under the receiver operating characteristic curve (AUC) of a membership-inference classifier trained to distinguish retained from obsolete samples from the repaired representation; values closer to $0.5$ indicate weaker separability. Linear probe denotes the AUC of a linear classifier trained to recover obsolete-region membership, measuring how easily obsolete information remains linearly accessible. Near-obsolete centroid rate measures the fraction of samples whose nearest region prototype is obsolete, testing whether samples remain geometrically closest to the retired region. Obsolete direction score measures the alignment of representations with the obsolete-retained direction, where larger values indicate stronger movement toward obsolete geometry.

Overall, DR substantially reduces obsolete-region geometry on $D_{\mathrm{o}}$ while maintaining retained utility. On CIFAR10, DR reduces the near-obsolete-centroid rate to $0.04$ and the obsolete-direction score to $0.30$, compared with $0.90/0.60$ for Fictitious Class Removal and $0.68/0.42$ for fine-tuning. On CIFAR100, DR similarly gives low obsolete-centroid and direction scores of $0.05/0.15$, whereas class-vector editing and fine-tuning leave much stronger obsolete structure at $0.49/0.32$ and $0.54/0.38$, respectively.

\begin{table*}[tb]
\centering
\caption{Representation-level classification results across all datasets on retained and obsolete test splits. Metrics report membership-inference attack AUROC, linear-probe AUROC, proximity to obsolete centroids, and alignment with obsolete-retained directions.}
\label{tab:appendix_representation_results}
\scriptsize
\setlength{\tabcolsep}{5pt}
\resizebox{\textwidth}{!}{%
\begin{tabular}{lcc|cc|cc}
\toprule
& \multicolumn{2}{c|}{AUC}
& \multicolumn{2}{c|}{$\mathbf{D_{\mathrm{r}}}$}
& \multicolumn{2}{c}{$\mathbf{D_{\mathrm{o}}}$} \\
\cmidrule(lr){2-3}\cmidrule(lr){4-5}\cmidrule(lr){6-7}
\textbf{Method}
& \textbf{MIA} $\downarrow$
& \textbf{Probe} $\downarrow$
& \textbf{Near Obs. Centroid Rate} $\downarrow$
& \textbf{Obs. Direction} $\downarrow$
& \textbf{Near Obs. Centroid Rate} $\downarrow$
& \textbf{Obs. Direction} $\downarrow$ \\
\midrule

\multicolumn{7}{c}{\textbf{CIFAR10}} \\
\midrule
Base Model                 & -- & -- & -- & -- & -- & -- \\
SSD                        & 0.52 $\pm$ 0.02 & 0.98 $\pm$ 0.01 & 0.00 $\pm$ 0.00 & -0.01 $\pm$ 0.01 & 0.40 $\pm$ 0.15 & 0.26 $\pm$ 0.05 \\
LFSSD                      & 0.51 $\pm$ 0.02 & 0.98 $\pm$ 0.01 & 0.00 $\pm$ 0.01 & -0.03 $\pm$ 0.02 & 0.44 $\pm$ 0.22 & 0.28 $\pm$ 0.11 \\
Boundary Shrink            & 0.52 $\pm$ 0.02 & 0.96 $\pm$ 0.01 & 0.00 $\pm$ 0.00 & -0.07 $\pm$ 0.01 & 0.10 $\pm$ 0.03 & 0.08 $\pm$ 0.02 \\
Boundary Expand            & 0.53 $\pm$ 0.02 & 0.96 $\pm$ 0.01 & 0.00 $\pm$ 0.00 & -0.05 $\pm$ 0.02 & 0.13 $\pm$ 0.04 & 0.09 $\pm$ 0.02 \\
Fictitious Class Removal   & 0.49 $\pm$ 0.01  & 0.98 $\pm$ 0.01  & 0.01 $\pm$ 0.00 & 0.00 $\pm$ 0.00 & 0.90 $\pm$ 0.04 & 0.60 $\pm$ 0.06 \\
Class Vector               & 0.54 $\pm$ 0.01 & 0.99 $\pm$ 0.01 & 0.01 $\pm$ 0.01 & -0.01 $\pm$ 0.01 & 0.13 $\pm$ 0.17 & 0.13 $\pm$ 0.03 \\
Fisher Forgetting          & 0.51 $\pm$ 0.01 & 0.97 $\pm$ 0.02 & 0.00 $\pm$ 0.00 & -0.08 $\pm$ 0.01 & 0.01 $\pm$ 0.01 & -0.00 $\pm$ 0.04 \\
Finetune                   & 0.52 $\pm$ 0.02 & 0.98 $\pm$ 0.01 & 0.00 $\pm$ 0.00 & 0.06 $\pm$ 0.01 & 0.68 $\pm$ 0.06 & 0.42 $\pm$ 0.05 \\
Retrain                    & 0.50 $\pm$ 0.00 & 0.96 $\pm$ 0.02 & 0.04 $\pm$ 0.05 & 0.07 $\pm$ 0.01 & 0.30 $\pm$ 0.16 & 0.16 $\pm$ 0.02 \\
Deep Repurposing           & 0.52 $\pm$ 0.02 & 0.97 $\pm$ 0.01 & 0.00 $\pm$ 0.00 & -0.00 $\pm$ 0.00 & 0.04 $\pm$ 0.00 & 0.30 $\pm$ 0.04 \\

\midrule
\multicolumn{7}{c}{\textbf{CIFAR100}} \\
\midrule
Base Model                 & -- & -- & -- & -- & -- & -- \\
SSD                        & 0.52 $\pm$ 0.04 & 0.98 $\pm$ 0.02 & 0.00 $\pm$ 0.00 & 0.04 $\pm$ 0.02 & 0.36 $\pm$ 0.33 & 0.25 $\pm$ 0.10 \\
LFSSD                      & 0.58 $\pm$ 0.09 & 0.96 $\pm$ 0.05 & 0.00 $\pm$ 0.00 & 0.02 $\pm$ 0.02 & 0.49 $\pm$ 0.29 & 0.38 $\pm$ 0.10 \\
Boundary Shrink            & 0.50 $\pm$ 0.05 & 0.92 $\pm$ 0.05 & 0.00 $\pm$ 0.00 & -0.06 $\pm$ 0.04 & 0.05 $\pm$ 0.03 & 0.04 $\pm$ 0.01 \\
Boundary Expand            & 0.51 $\pm$ 0.04 & 0.92 $\pm$ 0.06 & 0.00 $\pm$ 0.00 & -0.06 $\pm$ 0.01 & 0.05 $\pm$ 0.02 & 0.03 $\pm$ 0.02 \\
Fictitious Class Removal   & 0.42 $\pm$ 0.05 & 0.95 $\pm$ 0.05 & 0.00 $\pm$ 0.00 & 0.00 $\pm$ 0.00 & 0.79 $\pm$ 0.20 & 0.64 $\pm$ 0.13 \\
Class Vector               & 0.65 $\pm$ 0.07 & 0.95 $\pm$ 0.05 & 0.00 $\pm$ 0.00 & -0.05 $\pm$ 0.02 & 0.49 $\pm$ 0.34 & 0.32 $\pm$ 0.08 \\
Fisher Forgetting          & 0.56 $\pm$ 0.09 & 0.97 $\pm$ 0.02 & 0.00 $\pm$ 0.00 & -0.08 $\pm$ 0.02 & 0.19 $\pm$ 0.21 & 0.20 $\pm$ 0.19 \\
Finetune                   & 0.58 $\pm$ 0.12 & 0.97 $\pm$ 0.03 & 0.00 $\pm$ 0.00 & 0.00 $\pm$ 0.04 & 0.54 $\pm$ 0.28 & 0.38 $\pm$ 0.07 \\
Retrain                    & 0.50 $\pm$ 0.00 & 0.97 $\pm$ 0.02 & 0.00 $\pm$ 0.00 & 0.00 $\pm$ 0.04 & 0.27 $\pm$ 0.14 & 0.21 $\pm$ 0.05 \\
Deep Repurposing           & 0.56 $\pm$ 0.09 & 0.91 $\pm$ 0.08 & 0.00 $\pm$ 0.00 & 0.00 $\pm$ 0.00 & 0.05 $\pm$ 0.09 & 0.15 $\pm$ 0.15 \\

\midrule
\multicolumn{7}{c}{\textbf{Flowers102}} \\
\midrule
Base Model                 & -- & -- & -- & -- & -- & -- \\
SSD                        & 0.53 $\pm$ 0.12 & 0.81 $\pm$ 0.28 & 0.00 $\pm$ 0.00 & -0.03 $\pm$ 0.05 & 0.23 $\pm$ 0.39 & 0.08 $\pm$ 0.22 \\
LFSSD                      & 0.87 $\pm$ 0.11 & 1.00 $\pm$ 0.00 & 0.00 $\pm$ 0.00 & 0.00 $\pm$ 0.01 & 0.78 $\pm$ 0.18 & 0.36 $\pm$ 0.11 \\
Boundary Shrink            & 0.86 $\pm$ 0.12 & 0.99 $\pm$ 0.01 & 0.00 $\pm$ 0.00 & -0.12 $\pm$ 0.03 & 0.42 $\pm$ 0.20 & 0.23 $\pm$ 0.13 \\
Boundary Expand            & 0.86 $\pm$ 0.13 & 0.99 $\pm$ 0.01 & 0.00 $\pm$ 0.00 & -0.13 $\pm$ 0.01 & 0.41 $\pm$ 0.20 & 0.23 $\pm$ 0.11 \\
Fictitious Class Removal   & 0.38 $\pm$ 0.10 & 1.00 $\pm$ 0.00 & 0.00 $\pm$ 0.00 & 0.00 $\pm$ 0.00 & 0.88 $\pm$ 0.14 & 0.49 $\pm$ 0.22 \\
Class Vector               & 0.94 $\pm$ 0.07 & 1.00 $\pm$ 0.00 & 0.00 $\pm$ 0.00 & -0.02 $\pm$ 0.03 & 0.78 $\pm$ 0.21 & 0.43 $\pm$ 0.23 \\
Fisher Forgetting          & 0.85 $\pm$ 0.15 & 1.00 $\pm$ 0.00 & 0.00 $\pm$ 0.00 & -0.09 $\pm$ 0.04 & 0.46 $\pm$ 0.44 & 0.29 $\pm$ 0.26 \\
Finetune                   & 0.91 $\pm$ 0.08 & 1.00 $\pm$ 0.00 & 0.00 $\pm$ 0.00 & -0.01 $\pm$ 0.05 & 0.83 $\pm$ 0.25 & 0.47 $\pm$ 0.17 \\
Retrain                    & 0.50 $\pm$ 0.00 & 1.00 $\pm$ 0.00 & 0.00 $\pm$ 0.01 & -0.01 $\pm$ 0.04 & 0.47 $\pm$ 0.31 & 0.24 $\pm$ 0.06 \\
Deep Repurposing           & 0.53 $\pm$ 0.10 & 0.99 $\pm$ 0.02 & 0.00 $\pm$ 0.00 & 0.00 $\pm$ 0.00 & 0.63 $\pm$ 0.23 & 0.21 $\pm$ 0.04 \\

\midrule
\multicolumn{7}{c}{\textbf{ImageNet}} \\
\midrule
Base Model                 & -- & -- & -- & -- & -- & -- \\
SSD                        & 0.49 $\pm$ 0.06 & 0.62 $\pm$ 0.11 & 0.00 $\pm$ 0.00 & 0.05 $\pm$ 0.03 & 0.00 $\pm$ 0.00 & 0.06 $\pm$ 0.04 \\
LFSSD                      & 0.43 $\pm$ 0.05 & 0.86 $\pm$ 0.14 & 0.00 $\pm$ 0.00 & 0.01 $\pm$ 0.02 & 0.20 $\pm$ 0.17 & 0.33 $\pm$ 0.06 \\
Boundary Shrink            & 0.55 $\pm$ 0.06 & 0.81 $\pm$ 0.13 & 0.00 $\pm$ 0.00 & -0.04 $\pm$ 0.02 & 0.13 $\pm$ 0.12 & 0.24 $\pm$ 0.07 \\
Boundary Expand            & 0.53 $\pm$ 0.06 & 0.83 $\pm$ 0.11 & 0.00 $\pm$ 0.00 & -0.04 $\pm$ 0.02 & 0.15 $\pm$ 0.13 & 0.24 $\pm$ 0.08 \\
Fictitious Class Removal   & 0.48 $\pm$ 0.05 & 0.92 $\pm$ 0.08 & 0.00 $\pm$ 0.00 & -0.01 $\pm$ 0.01 & 0.58 $\pm$ 0.23 & 0.50 $\pm$ 0.15 \\
Class Vector               & 0.47 $\pm$ 0.03 & 0.91 $\pm$ 0.09 & 0.00 $\pm$ 0.00 & -0.03 $\pm$ 0.03 & 0.45 $\pm$ 0.31 & 0.41 $\pm$ 0.21 \\
Fisher Forgetting          & 0.38 $\pm$ 0.03 & 0.81 $\pm$ 0.19 & 0.00 $\pm$ 0.00 & -0.00 $\pm$ 0.01 & 0.32 $\pm$ 0.30 & 0.39 $\pm$ 0.15 \\
Finetune                   & 0.52 $\pm$ 0.04 & 0.90 $\pm$ 0.10 & 0.00 $\pm$ 0.00 & 0.03 $\pm$ 0.04 & 0.51 $\pm$ 0.18 & 0.46 $\pm$ 0.08 \\
Retrain                    & 0.50 $\pm$ 0.00 & 0.87 $\pm$ 0.10 & 0.00 $\pm$ 0.00 & 0.05 $\pm$ 0.06 & 0.39 $\pm$ 0.19 & 0.42 $\pm$ 0.05 \\
Deep Repurposing           & 0.53 $\pm$ 0.03 & 0.87 $\pm$ 0.12 & 0.00 $\pm$ 0.00 & -0.01 $\pm$ 0.01 & 0.11 $\pm$ 0.10 & 0.23 $\pm$ 0.08 \\

\bottomrule
\end{tabular}%
}
\end{table*}

The same pattern is visible on ImageNet, where DR obtains $0.11$ near-obsolete-centroid rate and $0.23$ obsolete-direction score. This is substantially lower than fine-tuning ($0.51/0.46$), class-vector editing ($0.45/0.41$), and retraining ($0.39/0.42$), while preserving the base retained accuracy. Boundary methods can also reduce latent obsolete structure, but this often coincides with retained-performance loss; for example, on CIFAR100, Boundary Shrink gives a similar centroid rate to DR ($0.05$), but its retained accuracy is only $0.62$ compared with $0.80$ for DR.

Table~\ref{tab:appendix_quartile_results} reports retained-set performance by difficulty quartile, testing whether a method preserves retained utility uniformly or only on easy retained samples. DR is stable across the easier quartiles and remains competitive on the hardest quartile, consistently improving accuracy and F1 on the most difficult quartile (Q4). On CIFAR10, DR obtains perfect accuracy and F1 on Q1 to Q3, and improves the hardest quartile over the base model from $0.61/0.56$ to $0.64/0.64$ in accuracy/F1. A similar pattern holds on CIFAR100, where DR preserves Q1 to Q3 at $1.00$ and slightly improves Q4 from $0.19/0.19$ to $0.20/0.20$.

\begin{table*}[tb]
\centering
\caption{Retained-set classification results by difficulty quartile across all datasets. Results report accuracy and F1 from the easiest retained examples (Q1) to the hardest retained examples (Q4).}
\label{tab:appendix_quartile_results}
\scriptsize
\setlength{\tabcolsep}{5pt}
\resizebox{\textwidth}{!}{%
\begin{tabular}{lcccc|cccc}
\toprule
& \multicolumn{4}{c|}{$\mathbf{D_{\mathrm{r}}}$ Accuracy}
& \multicolumn{4}{c}{$\mathbf{D_{\mathrm{r}}}$ F1}\\
\cmidrule(lr){2-5}\cmidrule(lr){6-9}
\textbf{Method}
& \textbf{Q1} $\uparrow$
& \textbf{Q2} $\uparrow$
& \textbf{Q3} $\uparrow$
& \textbf{Q4} $\uparrow$
& \textbf{Q1} $\uparrow$
& \textbf{Q2} $\uparrow$
& \textbf{Q3} $\uparrow$
& \textbf{Q4} $\uparrow$ \\
\midrule

\multicolumn{9}{c}{\textbf{CIFAR10}} \\
\midrule
Base Model                 & 1.00 $\pm$ 0.00 & 1.00 $\pm$ 0.00 & 1.00 $\pm$ 0.00 & 0.61 $\pm$ 0.02 & 1.00 $\pm$ 0.00 & 1.00 $\pm$ 0.00 & 1.00 $\pm$ 0.00 & 0.56 $\pm$ 0.02 \\
SSD                        & 0.99 $\pm$ 0.00 & 0.99 $\pm$ 0.01 & 0.96 $\pm$ 0.02 & 0.58 $\pm$ 0.03 & 0.99 $\pm$ 0.00 & 0.98 $\pm$ 0.01 & 0.95 $\pm$ 0.02 & 0.53 $\pm$ 0.03 \\
LFSSD                      & 0.99 $\pm$ 0.00 & 0.99 $\pm$ 0.01 & 0.96 $\pm$ 0.02 & 0.57 $\pm$ 0.03 & 0.99 $\pm$ 0.00 & 0.99 $\pm$ 0.01 & 0.95 $\pm$ 0.01 & 0.26 $\pm$ 0.22 \\
Boundary Shrink            & 0.98 $\pm$ 0.01 & 0.93 $\pm$ 0.03 & 0.81 $\pm$ 0.06 & 0.44 $\pm$ 0.05 & 0.97 $\pm$ 0.01 & 0.92 $\pm$ 0.03 & 0.80 $\pm$ 0.05 & 0.43 $\pm$ 0.04 \\
Boundary Expand            & 0.97 $\pm$ 0.01 & 0.92 $\pm$ 0.02 & 0.80 $\pm$ 0.04 & 0.44 $\pm$ 0.03 & 0.97 $\pm$ 0.01 & 0.92 $\pm$ 0.02 & 0.78 $\pm$ 0.03 & 0.40 $\pm$ 0.02 \\
Fictitious Class Removal   & 0.94 $\pm$ 0.04 & 0.86 $\pm$ 0.06 & 0.78 $\pm$ 0.06 & 0.43 $\pm$ 0.03 & 0.87 $\pm$ 0.02 & 0.83 $\pm$ 0.03 & 0.78 $\pm$ 0.04 & 0.46 $\pm$ 0.02 \\
Class Vector               & 1.00 $\pm$ 0.00 & 1.00 $\pm$ 0.00 & 1.00 $\pm$ 0.00 & 0.61 $\pm$ 0.01 & 1.00 $\pm$ 0.00 & 1.00 $\pm$ 0.00 & 0.998 $\pm$ 0.00 & 0.54 $\pm$ 0.01 \\
Fisher Forgetting          & 0.99 $\pm$ 0.01 & 0.95 $\pm$ 0.02 & 0.85 $\pm$ 0.05 & 0.49 $\pm$ 0.04 & 0.99 $\pm$ 0.01 & 0.95 $\pm$ 0.02 & 0.85 $\pm$ 0.04 & 0.47 $\pm$ 0.04 \\
Finetune                   & 1.00 $\pm$ 0.00 & 0.99 $\pm$ 0.00 & 0.99 $\pm$ 0.00 & 0.69 $\pm$ 0.02 & 1.00 $\pm$ 0.00 & 1.00 $\pm$ 0.00 & 0.99 $\pm$ 0.00 & 0.67 $\pm$ 0.03 \\
Retrain                    & 1.00 $\pm$ 0.00 & 1.00 $\pm$ 0.00 & 0.98 $\pm$ 0.00 & 0.66 $\pm$ 0.02 & 1.00 $\pm$ 0.00 & 1.00 $\pm$ 0.00 & 0.98 $\pm$ 0.00 & 0.66 $\pm$ 0.02 \\
Deep Repurposing           & 1.00 $\pm$ 0.00 & 1.00 $\pm$ 0.00 & 1.00 $\pm$ 0.00 & 0.64 $\pm$ 0.03 & 1.00 $\pm$ 0.00 & 1.00 $\pm$ 0.00 & 1.00 $\pm$ 0.00 & 0.64 $\pm$ 0.02 \\

\midrule
\multicolumn{9}{c}{\textbf{CIFAR100}} \\
\midrule
Base Model                 & 1.00 $\pm$ 0.00 & 1.00 $\pm$ 0.00 & 1.00 $\pm$ 0.00 & 0.19 $\pm$ 0.01 & 1.00 $\pm$ 0.00 & 1.00 $\pm$ 0.00 & 1.00 $\pm$ 0.00 & 0.19 $\pm$ 0.01 \\
SSD                        & 0.89 $\pm$ 0.18 & 0.89 $\pm$ 0.17 & 0.82 $\pm$ 0.27 & 0.18 $\pm$ 0.02 & 0.88 $\pm$ 0.19 & 0.88 $\pm$ 0.19 & 0.81 $\pm$ 0.28 & 0.17 $\pm$ 0.04 \\
LFSSD                      & 0.99 $\pm$ 0.00 & 0.99 $\pm$ 0.00 & 0.97 $\pm$ 0.02 & 0.21 $\pm$ 0.01 & 0.99 $\pm$ 0.00 & 0.99 $\pm$ 0.00 & 0.97 $\pm$ 0.02 & 0.20 $\pm$ 0.01 \\
Boundary Shrink            & 0.84 $\pm$ 0.10 & 0.85 $\pm$ 0.08 & 0.61 $\pm$ 0.09 & 0.19 $\pm$ 0.02 & 0.82 $\pm$ 0.10 & 0.81 $\pm$ 0.09 & 0.59 $\pm$ 0.09 & 0.16 $\pm$ 0.02 \\
Boundary Expand            & 0.84 $\pm$ 0.10 & 0.85 $\pm$ 0.08 & 0.61 $\pm$ 0.08 & 0.19 $\pm$ 0.02 & 0.82 $\pm$ 0.10 & 0.81 $\pm$ 0.09 & 0.59 $\pm$ 0.09 & 0.16 $\pm$ 0.02 \\
Fictitious Class Removal   & 0.42 $\pm$ 0.18 & 0.37 $\pm$ 0.17 & 0.40 $\pm$ 0.13 & 0.09 $\pm$ 0.02 & 0.50 $\pm$ 0.20 & 0.47 $\pm$ 0.17 & 0.51 $\pm$ 0.14 & 0.12 $\pm$ 0.02 \\
Class Vector               & 1.00 $\pm$ 0.00 & 1.00 $\pm$ 0.00 & 0.99 $\pm$ 0.00 & 0.19 $\pm$ 0.01 & 1.00 $\pm$ 0.00 & 1.00 $\pm$ 0.00 & 0.99 $\pm$ 0.00 & 0.19 $\pm$ 0.01 \\
Fisher Forgetting          & 0.99 $\pm$ 0.00 & 0.99 $\pm$ 0.01 & 0.89 $\pm$ 0.04 & 0.21 $\pm$ 0.01 & 0.99 $\pm$ 0.00 & 0.99 $\pm$ 0.01 & 0.89 $\pm$ 0.04 & 0.20 $\pm$ 0.01 \\
Finetune                   & 0.99 $\pm$ 0.00 & 0.99 $\pm$ 0.00 & 0.88 $\pm$ 0.00 & 0.30 $\pm$ 0.01 & 0.99 $\pm$ 0.00 & 0.98 $\pm$ 0.00 & 0.88 $\pm$ 0.00 & 0.29 $\pm$ 0.01 \\
Retrain                    & 0.99 $\pm$ 0.00 & 0.98 $\pm$ 0.00 & 0.87 $\pm$ 0.00 & 0.31 $\pm$ 0.01 & 0.98 $\pm$ 0.00 & 0.98 $\pm$ 0.00 & 0.87 $\pm$ 0.00 & 0.30 $\pm$ 0.01 \\
Deep Repurposing           & 1.00 $\pm$ 0.00 & 1.00 $\pm$ 0.00 & 1.00 $\pm$ 0.00 & 0.20 $\pm$ 0.01 & 1.00 $\pm$ 0.00 & 1.00 $\pm$ 0.00 & 1.00 $\pm$ 0.00 & 0.20 $\pm$ 0.01 \\

\midrule
\multicolumn{9}{c}{\textbf{Flowers102}} \\
\midrule
Base Model                 & 1.00 $\pm$ 0.00 & 1.00 $\pm$ 0.00 & 0.99 $\pm$ 0.00 & 0.49 $\pm$ 0.01 & 1.00 $\pm$ 0.00 & 1.00 $\pm$ 0.00 & 0.99 $\pm$ 0.00 & 0.46 $\pm$ 0.00 \\
SSD                        & 0.34 $\pm$ 0.53 & 0.34 $\pm$ 0.52 & 0.30 $\pm$ 0.47 & 0.13 $\pm$ 0.20 & 0.33 $\pm$ 0.53 & 0.32 $\pm$ 0.52 & 0.29 $\pm$ 0.46 & 0.11 $\pm$ 0.19 \\
LFSSD                      & 1.00 $\pm$ 0.01 & 0.98 $\pm$ 0.03 & 0.92 $\pm$ 0.09 & 0.43 $\pm$ 0.06 & 0.99 $\pm$ 0.02 & 0.98 $\pm$ 0.03 & 0.91 $\pm$ 0.10 & 0.41 $\pm$ 0.05 \\
Boundary Shrink            & 1.00 $\pm$ 0.00 & 0.97 $\pm$ 0.02 & 0.84 $\pm$ 0.04 & 0.35 $\pm$ 0.04 & 0.99 $\pm$ 0.01 & 0.96 $\pm$ 0.02 & 0.82 $\pm$ 0.04 & 0.34 $\pm$ 0.04 \\
Boundary Expand            & 1.00 $\pm$ 0.00 & 0.97 $\pm$ 0.02 & 0.85 $\pm$ 0.04 & 0.36 $\pm$ 0.04 & 0.99 $\pm$ 0.00 & 0.96 $\pm$ 0.02 & 0.83 $\pm$ 0.03 & 0.35 $\pm$ 0.04 \\
Fictitious Class Removal   & 0.78 $\pm$ 0.13 & 0.72 $\pm$ 0.10 & 0.67 $\pm$ 0.07 & 0.30 $\pm$ 0.02 & 0.85 $\pm$ 0.09 & 0.81 $\pm$ 0.08 & 0.76 $\pm$ 0.05 & 0.35 $\pm$ 0.01 \\
Class Vector               & 1.00 $\pm$ 0.00 & 1.00 $\pm$ 0.00 & 0.99 $\pm$ 0.00 & 0.48 $\pm$ 0.01 & 1.00 $\pm$ 0.00 & 1.00 $\pm$ 0.00 & 0.99 $\pm$ 0.00 & 0.43 $\pm$ 0.01 \\
Fisher Forgetting          & 1.00 $\pm$ 0.00 & 0.99 $\pm$ 0.01 & 0.93 $\pm$ 0.04 & 0.41 $\pm$ 0.04 & 1.00 $\pm$ 0.00 & 0.99 $\pm$ 0.01 & 0.92 $\pm$ 0.04 & 0.38 $\pm$ 0.03 \\
Finetune                   & 1.00 $\pm$ 0.00 & 1.00 $\pm$ 0.00 & 0.98 $\pm$ 0.00 & 0.59 $\pm$ 0.01 & 1.00 $\pm$ 0.00 & 1.00 $\pm$ 0.00 & 0.97 $\pm$ 0.00 & 0.57 $\pm$ 0.01 \\
Retrain                    & 1.00 $\pm$ 0.00 & 0.99 $\pm$ 0.00 & 0.95 $\pm$ 0.00 & 0.70 $\pm$ 0.01 & 1.00 $\pm$ 0.00 & 0.98 $\pm$ 0.00 & 0.95 $\pm$ 0.00 & 0.70 $\pm$ 0.01 \\
Deep Repurposing           & 1.00 $\pm$ 0.00 & 1.00 $\pm$ 0.00 & 1.00 $\pm$ 0.00 & 0.51 $\pm$ 0.01 & 1.00 $\pm$ 0.00 & 1.00 $\pm$ 0.00 & 1.00 $\pm$ 0.00 & 0.48 $\pm$ 0.01 \\

\midrule
\multicolumn{9}{c}{\textbf{ImageNet}} \\
\midrule
Base Model                 & 1.00 $\pm$ 0.00 & 1.00 $\pm$ 0.00 & 0.96 $\pm$ 0.00 & 0.27 $\pm$ 0.00 & 1.00 $\pm$ 0.00 & 1.00 $\pm$ 0.00 & 0.96 $\pm$ 0.00 & 0.26 $\pm$ 0.00 \\
SSD                        & 0.03 $\pm$ 0.08 & 0.03 $\pm$ 0.07 & 0.02 $\pm$ 0.06 & 0.01 $\pm$ 0.02 & 0.03 $\pm$ 0.08 & 0.03 $\pm$ 0.08 & 0.02 $\pm$ 0.06 & 0.01 $\pm$ 0.02 \\
LFSSD                      & 1.00 $\pm$ 0.00 & 1.00 $\pm$ 0.00 & 0.95 $\pm$ 0.00 & 0.27 $\pm$ 0.00 & 1.00 $\pm$ 0.00 & 1.00 $\pm$ 0.00 & 0.95 $\pm$ 0.00 & 0.26 $\pm$ 0.00 \\
Boundary Shrink            & 0.98 $\pm$ 0.01 & 0.97 $\pm$ 0.01 & 0.86 $\pm$ 0.03 & 0.25 $\pm$ 0.01 & 0.97 $\pm$ 0.01 & 0.95 $\pm$ 0.02 & 0.85 $\pm$ 0.03 & 0.24 $\pm$ 0.01 \\
Boundary Expand            & 0.98 $\pm$ 0.01 & 0.97 $\pm$ 0.01 & 0.86 $\pm$ 0.03 & 0.25 $\pm$ 0.01 & 0.97 $\pm$ 0.02 & 0.95 $\pm$ 0.02 & 0.85 $\pm$ 0.03 & 0.24 $\pm$ 0.01 \\
Fictitious Class Removal   & 0.70 $\pm$ 0.28 & 0.73 $\pm$ 0.24 & 0.65 $\pm$ 0.28 & 0.17 $\pm$ 0.10 & 0.74 $\pm$ 0.24 & 0.74 $\pm$ 0.24 & 0.70 $\pm$ 0.25 & 0.18 $\pm$ 0.08 \\
Class Vector               & 1.00 $\pm$ 0.00 & 1.00 $\pm$ 0.00 & 0.93 $\pm$ 0.00 & 0.26 $\pm$ 0.00 & 1.00 $\pm$ 0.00 & 0.99 $\pm$ 0.00 & 0.93 $\pm$ 0.00 & 0.26 $\pm$ 0.00 \\
Fisher Forgetting          & 0.95 $\pm$ 0.01 & 0.95 $\pm$ 0.01 & 0.82 $\pm$ 0.01 & 0.23 $\pm$ 0.00 & 0.95 $\pm$ 0.01 & 0.93 $\pm$ 0.01 & 0.82 $\pm$ 0.01 & 0.22 $\pm$ 0.00 \\
Finetune                   & 0.98 $\pm$ 0.00 & 0.97 $\pm$ 0.00 & 0.86 $\pm$ 0.00 & 0.30 $\pm$ 0.00 & 0.97 $\pm$ 0.00 & 0.95 $\pm$ 0.00 & 0.86 $\pm$ 0.00 & 0.28 $\pm$ 0.00 \\
Retrain                    & 0.71 $\pm$ 0.00 & 0.78 $\pm$ 0.00 & 0.62 $\pm$ 0.00 & 0.22 $\pm$ 0.00 & 0.66 $\pm$ 0.00 & 0.69 $\pm$ 0.00 & 0.60 $\pm$ 0.00 & 0.20 $\pm$ 0.00 \\
Deep Repurposing           & 1.00 $\pm$ 0.00 & 1.00 $\pm$ 0.00 & 0.96 $\pm$ 0.00 & 0.27 $\pm$ 0.00 & 1.00 $\pm$ 0.00 & 1.00 $\pm$ 0.00 & 0.96 $\pm$ 0.00 & 0.26 $\pm$ 0.00 \\

\bottomrule
\end{tabular}%
}
\end{table*}

The comparative methods show that retained-set degradation is often concentrated in harder retained examples. Boundary Shrink and Boundary Expand reduce CIFAR10 Q4 accuracy to $0.44$ on average, compared with $0.64$ for DR, and reduce CIFAR100 Q3 accuracy to $0.61$, compared with $1.00$ for DR. Fictitious Class Removal is more unstable, particularly on CIFAR100, where Q1--Q3 accuracy falls to $0.42$, $0.37$, and $0.40$. SSD also shows dataset-dependent collapse, reaching only $0.34/0.33$ accuracy/F1 on the easiest Flowers102 quartile and $0.03/0.03$ on ImageNet Q1. Fine-tuning and retraining can improve the hardest quartile in some cases, such as CIFAR10 Q4 F1 of $0.67$ and $0.66$, respectively, but this comes with the optimisation cost reported in Table~\ref{tab:appendix_core_results}.

Overall, the classification results show that DR provides a consistent repair across output, calibration, confidence, representation, and retained-difficulty diagnostics. Unlike methods that either suppress obsolete behaviour at the cost of retained utility or preserve accuracy while leaving obsolete probability and latent structure intact, DR removes obsolete predictions and probability mass across all datasets while maintaining retained performance close to the base model or improving it in some cases. Representation-level metrics show that DR also reduces obsolete geometry in the latent space, although some latent traces can remain, as expected for a repair method rather than full erasure. Finally, the quartile analysis indicates that DR does not shift the cost of repair onto difficult retained examples and can actually improve the accuracy of the most difficult examples in the test dataset.

Table~\ref{tab:appendix_utkface_full_results} provides the full UTKFace regression results corresponding to the summary reported in Table~\ref{tab:regression-main-results}. In addition to MAE, RMSE, obsolete prediction rate, and repair time, the appendix table reports within-threshold accuracy. Within 5 and Within 10 denote the fraction of predictions whose absolute error is at most 5 and 10 years, respectively; higher values therefore indicate that more predictions fall within an acceptable age-error tolerance.

\begin{table*}[htb]
\centering
\caption{UTKFace label-region regression results reporting predictive error, within-threshold accuracy, obsolete-region suppression, and repair time.}
\label{tab:appendix_utkface_full_results}
\footnotesize
\setlength{\tabcolsep}{5pt}
\resizebox{\textwidth}{!}{%
\begin{tabular}{llcccccc}
\toprule
Method & Split & MAE $\downarrow$ & RMSE $\downarrow$ & Within 5 $\uparrow$ & Within 10 $\uparrow$ & Obs. rate $\downarrow$ & Repair (s) $\downarrow$ \\
\midrule
\multirow{2}{*}{Base model}
& $D_{\mathrm r}$ & $5.69 \pm 0.06$ & $7.70 \pm 0.10$ & $0.558 \pm 0.006$ & $0.830 \pm 0.003$ & $0.014 \pm 0.002$ & \multirow{2}{*}{--} \\
& $D_{\mathrm o}$ & $10.11 \pm 0.61$ & $13.34 \pm 0.79$ & $0.339 \pm 0.020$ & $0.607 \pm 0.029$ & $0.607 \pm 0.011$ & \\
\midrule

\multirow{2}{*}{Fine-tuning}
& $D_{\mathrm r}$ & $5.28 \pm 0.10$ & $7.16 \pm 0.13$ & $0.587 \pm 0.011$ & $0.855 \pm 0.008$ & $0.003 \pm 0.001$ & \multirow{2}{*}{$282.05 \pm 140.76$} \\
& $D_{\mathrm o}$ & $14.86 \pm 0.79$ & $17.21 \pm 0.74$ & $0.111 \pm 0.024$ & $0.317 \pm 0.047$ & $0.379 \pm 0.040$ & \\
\midrule

\multirow{2}{*}{Retraining}
& $D_{\mathrm r}$ & $4.46 \pm 0.04$ & $6.26 \pm 0.04$ & $0.668 \pm 0.004$ & $0.895 \pm 0.004$ & $0.000 \pm 0.000$ & \multirow{2}{*}{$1208.13 \pm 43.73$} \\
& $D_{\mathrm o}$ & $19.68 \pm 1.07$ & $22.21 \pm 0.98$ & $0.032 \pm 0.021$ & $0.176 \pm 0.044$ & $0.000 \pm 0.000$ & \\
\midrule

\multirow{2}{*}{Deep Repurposing}
& $D_{\mathrm r}$ & $5.68 \pm 0.11$ & $7.66 \pm 0.13$ & $0.556 \pm 0.010$ & $0.830 \pm 0.005$ & $0.000 \pm 0.000$ & \multirow{2}{*}{$16.00 \pm 3.93$} \\
& $D_{\mathrm o}$ & $21.09 \pm 2.35$ & $23.73 \pm 2.24$ & $0.038 \pm 0.014$ & $0.159 \pm 0.076$ & $0.000 \pm 0.000$ & \\
\bottomrule
\end{tabular}%
}
\end{table*}

The results show the expected trade-off between retained regression error and obsolete-region suppression. Retraining gives the strongest retained predictive performance, reducing $D_{\mathrm{r}}$ MAE to $4.46$ and increasing Within 10 to $0.895$, but requires $1208.13$ seconds on average. Fine-tuning also improves retained accuracy, reducing $D_{\mathrm{r}}$ MAE to $5.28$, but only partially suppresses the obsolete region, leaving an obsolete prediction rate of $0.379$ on $D_{\mathrm{o}}$. In contrast, DR preserves the base model's retained performance almost exactly, with $D_{\mathrm{r}}$ MAE changing from $5.69$ to $5.68$ and Within 10 remaining $0.830$, while reducing the obsolete prediction rate to $0.000$ on both splits.

Table~\ref{tab:ade20k-seg-results} reports the ADE20K semantic segmentation results under taxonomy-level obsolescence; these results complement the qualitative example in Figure~\ref{fig:segmentation_example}. The obsolete set consists of fine-grained chair subtypes, namely \texttt{armchair}, \texttt{seat}, \texttt{swivel chair}, and \texttt{stool}, while the remaining ADE20K classes are retained, where the task is taxonomy-obsolescence, where sub-chair classes are all repurposed into a single \texttt{chair} class.

\begin{table*}[htb]
\centering
\caption{ADE20K semantic segmentation results. The obsolete classes are chair subtypes
$\{\texttt{armchair}, \texttt{seat}, \texttt{swivel chair}, \texttt{stool}\}$, while all other ADE20K classes are retained.}
\label{tab:ade20k-seg-results}
\resizebox{\textwidth}{!}{
\begin{tabular}{llccccccc}
\toprule
Method & Split
& Pixel Acc. $\uparrow$
& Ret. Pixel Acc. $\uparrow$
& Ret. mIoU $\uparrow$
& NLL $\downarrow$
& ECE $\downarrow$
& Obs. Pred. $\downarrow$
& Eff. Obs. Prob. $\downarrow$ \\
\midrule
Base & $D_{\mathrm{r}}$
& $0.77 \pm 0.00$
& $0.77 \pm 0.00$
& $0.36 \pm 0.00$
& $\mathbf{0.83 \pm 0.00}$
& $0.06 \pm 0.00$
& $0.00 \pm 0.00$
& $0.00 \pm 0.00$ \\

DR & $D_{\mathrm{r}}$
& $\mathbf{0.78 \pm 0.01}$
& $\mathbf{0.78 \pm 0.01}$
& $\mathbf{0.36 \pm 0.01}$
& $0.84 \pm 0.03$
& $\mathbf{0.06 \pm 0.00}$
& $\mathbf{0.00 \pm 0.00}$
& $\mathbf{0.00 \pm 0.00}$ \\

\midrule

Base & $D_{\mathrm{o}}$
& $\mathbf{0.72 \pm 0.00}$
& $0.74 \pm 0.00$
& $\mathbf{0.25 \pm 0.00}$
& $0.95 \pm 0.00$
& $\mathbf{0.05 \pm 0.00}$
& $0.07 \pm 0.00$
& $0.06 \pm 0.00$ \\

DR & $D_{\mathrm{o}}$
& $0.68 \pm 0.02$
& $\mathbf{0.75 \pm 0.01}$
& $0.25 \pm 0.00$
& -- 
& $0.06 \pm 0.00$
& $\mathbf{0.00 \pm 0.00}$
& $\mathbf{0.00 \pm 0.00}$ \\
\bottomrule
\end{tabular}
}
\end{table*}

On the retained split $D_{\mathrm{r}}$, DR preserves segmentation utility. Pixel accuracy increases slightly from $0.77$ to $0.78$, while retained mIoU remains stable at approximately $0.36$. The obsolete prediction rate and effective obsolete probability are both reduced to zero, indicating that DR removes leakage of obsolete chair-subtype predictions from retained images without degrading retained segmentation quality.

On the obsolete-containing split $D_{\mathrm{o}}$, DR achieves the intended removal of obsolete output behaviour. The obsolete prediction rate decreases from $0.07$ to $0.00$, and the effective obsolete probability mass is also reduced from $0.06$ to $0.00$. This shows that obsolete chair-subtype labels are not merely avoided at the final decision level, but are removed from the repaired probability output. Retained-pixel accuracy improves from $0.74$ to $0.75$, indicating that DR does not harm, and slightly improves the segmentation of the retained regions within images that contain obsolete objects. It does this within $21.71 \pm 0.72$ seconds, showing that DR is a fast, analytical repair tool that can reduce obsolete predictions and improve retained ones.

The lower all-pixel accuracy on $D_{\mathrm{o}}$ reflects evaluation against the original ADE20K labels, where obsolete chair subtypes remain valid targets.

\subsection{Comparison to Logit Suppression}
\label{sec:appendix-Comparison to Logit Suppression}

Logit suppression is the simplest way to remove an obsolete class at deployment: the obsolete logit is masked before prediction, preventing the class from being returned by the final decision rule. However, logit suppression only changes the post-processed prediction and therefore defaults to the next highest prediction instead of deciding the most reasonable next prediction. The original model may still remain highly confident in that decision, and preserve obsolete structure in its latent representation.

\begin{table*}[htb]
\centering
\caption{Diagnostic comparison between Logit Suppression (LS) and Deep Repurposing (DR) on CIFAR10. 
Although both methods remove obsolete predictions after post-processing, LS leaves substantial obsolete probability mass and obsolete latent structure in the underlying model. 
Lower is better for obsolete probability, centroid rate, direction score, confidence, and unlearning time.}
\label{tab:ls_dr_diagnostic}
\footnotesize
\setlength{\tabcolsep}{5pt}
\begin{tabular}{lcc|ccc}
\toprule
& \multicolumn{2}{c|}{$\mathbf{D_{\mathrm{r}}}$ utility}
& \multicolumn{3}{c}{$\mathbf{D_{\mathrm{o}}}$ obsolete-output behaviour} \\
\cmidrule(lr){2-3}
\cmidrule(lr){4-6}
\textbf{Method}
& \textbf{Acc.} $\uparrow$
& \textbf{F1} $\uparrow$
& \textbf{Obs. Rate} $\downarrow$
& \textbf{Raw Obs. Prob.} $\downarrow$
& \textbf{Unlearn (s)} $\downarrow$ \\
\midrule
Logit Suppression
& $0.91 \pm 0.01$
& $0.91 \pm 0.01$
& $0.00 \pm 0.00$
& $0.88 \pm 0.05$
& $0.00 \pm 0.00$ \\

Deep Repurposing
& $0.91 \pm 0.01$
& $0.91 \pm 0.01$
& $0.00 \pm 0.00$
& $\mathbf{0.00 \pm 0.00}$
& $1.16 \pm 0.04$ \\
\bottomrule
\end{tabular}

\vspace{0.65em}
\resizebox{\textwidth}{!}{%
\begin{tabular}{lcccc|cc}
\toprule
& \multicolumn{4}{c|}{$\mathbf{D_{\mathrm{o}}}$ confidence diagnostics}
& \multicolumn{2}{c}{$\mathbf{D_{\mathrm{o}}}$ Latent diagnostics} \\
\cmidrule(lr){2-5}
\cmidrule(lr){6-7}
\textbf{Method}
& \textbf{Eff. Conf.} $\downarrow$
& \textbf{Raw Conf.} $\downarrow$
& \textbf{Eff. Ent.} $\uparrow$
& \textbf{Raw Ent.} $\uparrow$
& \textbf{Near Obs. Cent.} $\downarrow$
& \textbf{Obs. Dir.} $\downarrow$ \\
\midrule
Logit Suppression
& $0.69 \pm 0.05$
& $0.93 \pm 0.02$
& $0.91 \pm 0.13$
& $0.19 \pm 0.07$
& $0.90 \pm 0.04$
& $0.60 \pm 0.06$ \\

Deep Repurposing
& $\mathbf{0.60 \pm 0.03}$
& $\mathbf{0.60 \pm 0.03}$
& $\mathbf{1.12 \pm 0.08}$
& $\mathbf{1.12 \pm 0.08}$
& $\mathbf{0.04 \pm 0.05}$
& $\mathbf{0.30 \pm 0.02}$ \\
\bottomrule
\end{tabular}%
}
\end{table*}

Table~\ref{tab:ls_dr_diagnostic} compares logit suppression and DR on CIFAR10. Both methods achieve the same retained accuracy and F1, and both reduce the effective obsolete prediction rate on $D_{\mathrm{o}}$ to zero. However, their behaviour differs substantially before post-processing. Under logit suppression, the raw obsolete probability remains high at $0.88 \pm 0.05$, showing that the model still strongly supports the obsolete class even though this prediction is hidden from the output. In contrast, DR reduces the raw obsolete probability to $0.00 \pm 0.00$, indicating that obsolete behaviour is removed from the repaired output itself rather than only masked externally.

The confidence and latent diagnostics show the same pattern. Logit suppression produces a mismatch between effective and raw uncertainty: after masking, confidence decreases and entropy increases, but the raw model remains highly confident with low entropy. DR avoids this mismatch, with matching effective and raw confidence and entropy because the repaired output is itself obsolete-free. Moreover, logit suppression leaves most obsolete samples nearest to the obsolete centroid ($0.90 \pm 0.04$), whereas DR reduces this rate to $0.04 \pm 0.05$. DR also lowers the obsolete-direction score from $0.60 \pm 0.06$ to $0.30 \pm 0.02$, suggesting that obsolete samples are moved away from the obsolete representation direction.

These results clarify the distinction between naive output suppression and deep repurposing. Logit suppression is essentially cost-free and is sufficient when the only requirement is to hide an obsolete label from the final prediction. DR incurs a small adaptation cost of $1.16 \pm 0.04$ seconds, but provides a stronger repair: it removes obsolete probability mass from the raw output and reduces obsolete structure in the latent space while preserving retained utility. This is important when model outputs, confidences, or representations are consumed by downstream modules, where a hidden but still internally active obsolete class may remain problematic.

\subsection{Ablation Study}
\label{sec:appendix-Ablation Study}

Next, we ablate the hyperparameters that Deep Repurposing introduces: the stripping margin $m$, the Bayes evidence weight $\beta$, and the contrastive repair weight $\lambda$. The margin $m$ controls the strength of obsolete-retained margin suppression during latent stripping. The weight $\beta$ controls how much retained prototype evidence contributes after stripping, while $\lambda$ controls the contribution of the contrastive gain used to redirect removed evidence toward retained regions.

\begin{figure*}[htb]
    \centering
    \includegraphics[width=\linewidth]{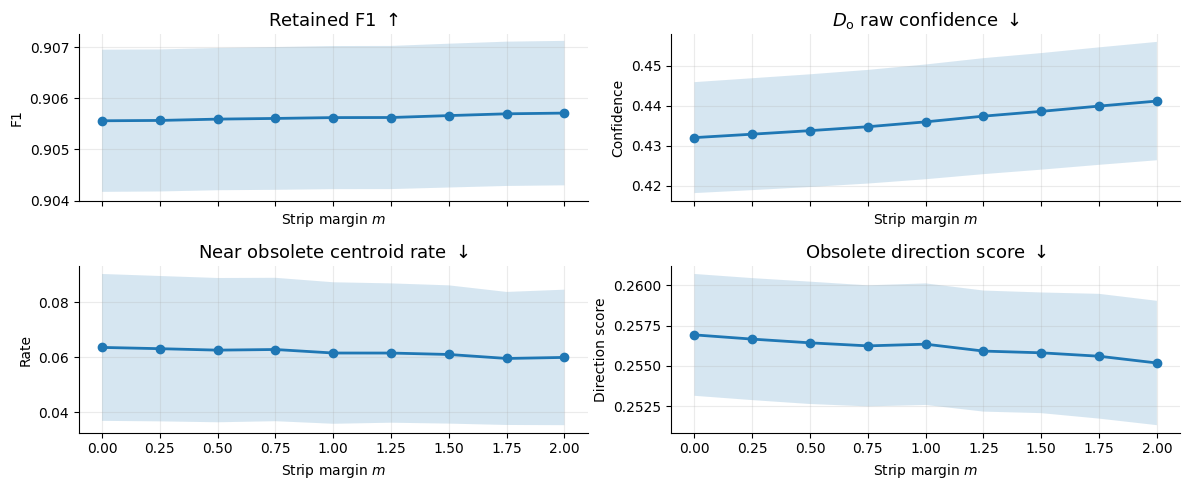}
    \caption{Ablation of the stripping margin $m$ across various values.}
    \label{fig:strip_margin_ablation}
\end{figure*}

Ablating the stripping margin $m$, seen in Figure~\ref{fig:strip_margin_ablation}, shows that DR is stable across a wide range of margin values. Retained F1 is almost flat, increasing only marginally from roughly $\approx0.9055$ to $\approx0.9057$, so stronger stripping does not degrade retained performance. In contrast, $D_o$ raw confidence rises mildly as $m$ increases, suggesting that overly strong margin stripping can increase residual confidence on obsolete samples even while improving latent separation. Overall, the effect of $m$ is small, and the overlapping uncertainty bands indicate that DR is not highly sensitive to this hyperparameter.

\begin{figure*}[htb]
    \centering
    \includegraphics[width=\linewidth]{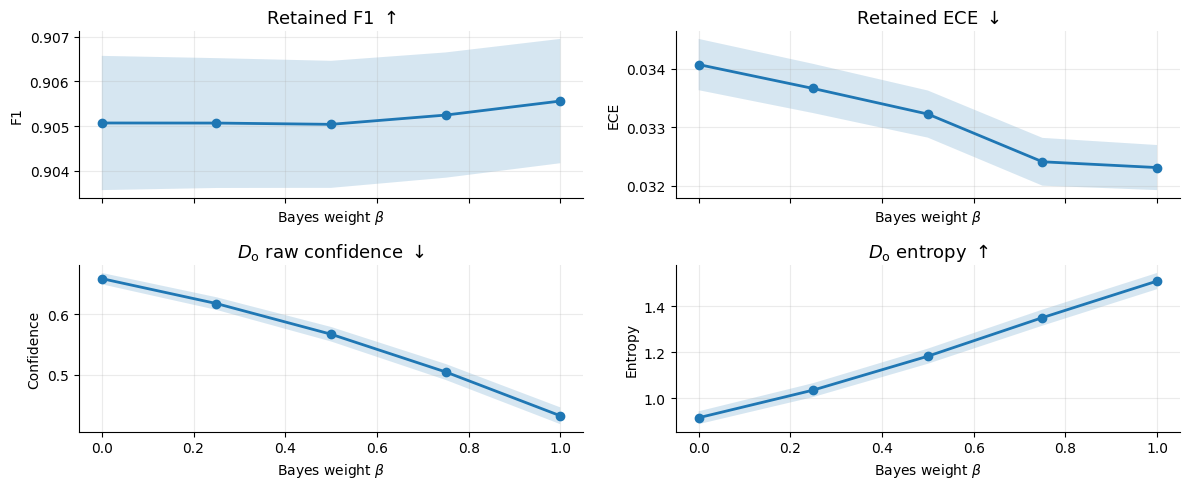}
    \caption{Ablation of the Bayes weight $\beta$ across various values.}
    \label{fig:beta_ablation}
\end{figure*}

Ablating $\beta$, seen in Figure~\ref{fig:beta_ablation}, shows that retained prototype evidence is important for stable repair. As $\beta$ increases, retained F1 remains effectively unchanged and even improves slightly at the largest value, while retained ECE decreases. This suggests that prototype evidence improves calibration slightly without sacrificing retained accuracy. $D_o$ raw confidence falls sharply as $\beta$ increases, and $D_o$ entropy rises correspondingly. Thus, larger $\beta$ makes obsolete-region inputs less overconfident after repair, while maintaining the retained task. These results indicate that the Bayes term provides useful retained-region structure that improves uncertainty and obsolete-output suppression.

\begin{figure*}[htb]
    \centering
    \includegraphics[width=\linewidth]{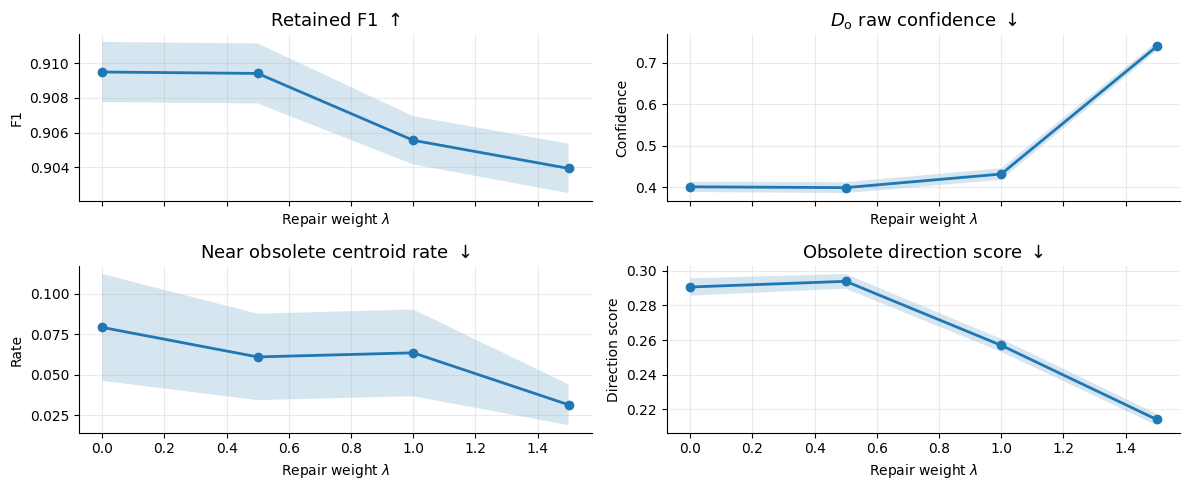}
    \caption{Ablation of the contrastive repair weight $\lambda$ across various values.}
    \label{fig:lambda_ablation}
\end{figure*}

The repair-weight ablation, seen in Figure~\ref{fig:lambda_ablation}, shows that $\lambda$ mainly controls the strength of representation-level repurposing. Increasing $\lambda$ reduces both the near-obsolete-centroid rate and the obsolete direction score, indicating stronger movement away from obsolete latent geometry as expected. This benefit comes with a trade-off: retained F1 drops slightly as $\lambda$ becomes large, and $D_o$ raw confidence rises sharply at the largest value. Thus, the contrastive repair term is beneficial, but only when applied moderately. Small-to-intermediate values preserve retained performance while improving obsolete-geometry suppression, whereas overly large values over-emphasise repair and degrade the retained-obsolete trade-off.

\begin{figure*}[htb]
    \centering
    \includegraphics[width=\linewidth]{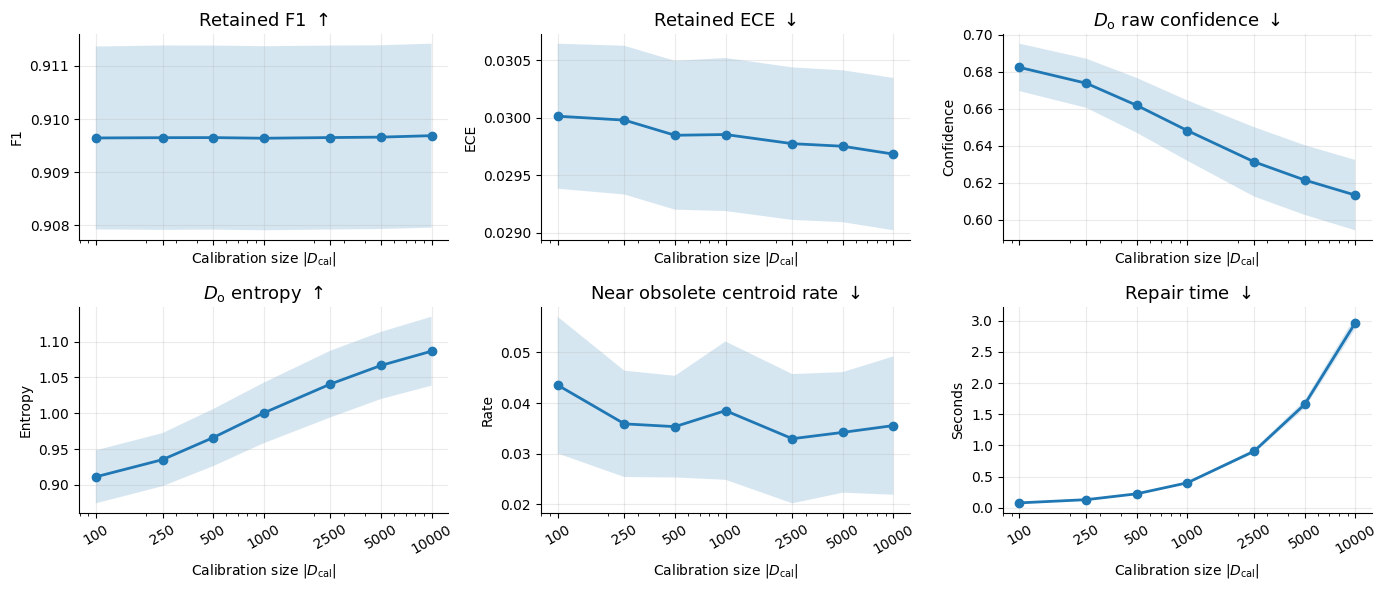}
    \caption{Ablation of the size of the calibration dataset $|\mathcal{D}_{cal}|$ across various values.}
    \label{fig:dcal_ablation}
\end{figure*}

Figure~\ref{fig:dcal_ablation} shows that DR is largely insensitive to the calibration-set size for retained performance: retained F1 is effectively flat across all values of $|\mathcal{D}_{\mathrm{cal}}|$, and retained ECE changes only marginally with overlapping uncertainty bands. Increasing the calibration size mainly affects obsolete-split behaviour, where $\mathcal{D}_o$ raw confidence decreases steadily and entropy increases, suggesting that larger calibration sets provide a better estimate of retained/obsolete geometry and lead to less over-confident redistribution of obsolete samples. The near-obsolete centroid rate remains low across the full range, with only minor non-monotonic variation, indicating that even small calibration sets are sufficient to move samples away from obsolete centroids. The main trade-off is computational: repair time increases with $|\mathcal{D}_{\mathrm{cal}}|$. At $\leq500$ samples, DR takes $\approx0.25s$ to repair on CIFAR10 but at $10,000$ samples it takes $\approx3.0s$. 

Tables~\ref{tab:msl-refinement-ablation-main} and~\ref{tab:msl-refinement-ablation-latent} examine the effect of residual refinement on the MSL case study. The unrepaired base model retains substantial obsolete behaviour, assigning obsolete samples back to the obsolete classes at an average rate of $0.672$ with obsolete probability mass of $0.398$ on average. Both DR variants remove this behaviour completely, reducing the obsolete prediction rate and obsolete probability mass to $0.000$. This confirms that the core DR repair is sufficient for obsolete-output suppression in this setting.

\begin{table*}[htb]
\centering
\caption{Residual refinement ablation on the MSL case study, including the unrepaired base model.}
\label{tab:msl-refinement-ablation-main}
\scriptsize
\setlength{\tabcolsep}{2.4pt}
\renewcommand{\arraystretch}{1.05}
\resizebox{\textwidth}{!}{%
\begin{tabular}{@{}lcccccccc@{}}
\toprule
\textbf{Method} 
& \textbf{$D_{\mathrm{r}}$ Acc. $\uparrow$}
& \textbf{$D_{\mathrm{r}}$ F1 $\uparrow$}
& \textbf{$D_{\mathrm{r}}$ NLL $\downarrow$}
& \textbf{$D_{\mathrm{r}}$ ECE $\downarrow$}
& \textbf{$D_{\mathrm{o}}$ Conf. $\downarrow$}
& \textbf{$D_{\mathrm{o}}$ Ent. $\uparrow$}
& \textbf{Obs. Rate $\downarrow$}
& \textbf{Obs. Prob. $\downarrow$} \\
\midrule
Base
& $0.803 \pm 0.001$
& $0.759 \pm 0.002$
& $0.887 \pm 0.002$
& $0.157 \pm 0.001$
& $0.433 \pm 0.003$
& $2.092 \pm 0.008$
& $0.672 \pm 0.007$
& $0.398 \pm 0.003$ \\

DR (w/o ref)
& $0.803 \pm 0.001$
& $0.792 \pm 0.002$
& $0.866 \pm 0.002$
& $\mathbf{0.149 \pm 0.001}$
& $\mathbf{0.253 \pm 0.001}$
& $\mathbf{2.556 \pm 0.004}$
& $\mathbf{0.000 \pm 0.000}$
& $\mathbf{0.000 \pm 0.000}$ \\

DR
& $\mathbf{0.884 \pm 0.001}$
& $\mathbf{0.856 \pm 0.002}$
& $\mathbf{0.579 \pm 0.002}$
& $0.186 \pm 0.001$
& $0.280 \pm 0.007$
& $2.347 \pm 0.028$
& $\mathbf{0.000 \pm 0.000}$
& $\mathbf{0.000 \pm 0.000}$ \\
\bottomrule
\end{tabular}%
}
\end{table*}

\begin{table}[htb]
\centering
\caption{Representation and hard-example metrics for the MSL DR refinement ablation.}
\label{tab:msl-refinement-ablation-latent}
\scriptsize
\setlength{\tabcolsep}{4pt}
\renewcommand{\arraystretch}{1.05}
\begin{tabular}{@{}lcccc@{}}
\toprule
\textbf{Method}
& \textbf{$D_{\mathrm{o}}$ Centroid $\downarrow$}
& \textbf{$D_{\mathrm{o}}$ Direction $\downarrow$}
& \textbf{Q4 Acc. $\uparrow$}
& \textbf{Q4 F1 $\uparrow$} \\
\midrule
Base
& -- 
& -- 
& $0.222 \pm 0.002$
& $0.177 \pm 0.010$ \\

DR (w/o ref.)
& $\mathbf{0.452 \pm 0.080}$
& $0.430 \pm 0.013$
& $0.222 \pm 0.002$
& $0.186 \pm 0.011$ \\

DR
& $0.816 \pm 0.062$
& $\mathbf{0.278 \pm 0.008}$
& $\mathbf{0.476 \pm 0.046}$
& $\mathbf{0.238 \pm 0.013}$ \\
\bottomrule
\end{tabular}
\end{table}

Residual refinement mainly affects retained utility. Without refinement, DR improves retained macro-F1 over the base model from $0.759$ to $0.792$, but retained accuracy remains unchanged at $0.803$. With refinement, retained accuracy increases to $0.884$ and macro-F1 to $0.856$, while retained NLL decreases substantially from $0.887$ for the base model to $0.579$ on average.

The hard-example metrics in Table~\ref{tab:msl-refinement-ablation-latent} support this interpretation. On the hardest retained quartile, DR without refinement gives similar accuracy to the base model, whereas refinement increases Q4 accuracy from $0.222$ to $0.476$ and Q4 macro-F1 from $0.186$ to $0.238$. The representation-level effects are mixed: refinement lowers the obsolete-direction score from $0.430$ to $0.278$, indicating weaker alignment with the obsolete-retained direction, but increases the nearest-obsolete-centroid rate from $0.452$ to $0.816$. Thus, refinement improves retained predictive repair and reduces one form of obsolete alignment, while not uniformly improving all latent-space metrics.


\section{Additional Discussion}
\label{sec:appendix-Additional Discussion}

The experiments above demonstrate the practical motivation for Deep Repurposing: deployed models may require targeted behavioural repair when classes become obsolete, unreliable, or no longer operationally useful. This section discusses the broader implications of this setting, including how DR relates to alternative update mechanisms/deep learning subfields, what assumptions it makes, and where it should be viewed as complementary rather than competitive with existing model maintenance pipelines.

\subsection{Federated Learning}
\label{sec:appendix-discussion-Federated Learning}

Federated learning (FL) is a natural alternative to consider for deployed systems, since it enables multiple clients to collaboratively train a shared model while keeping raw data decentralised~\citep{mcmahan2017communication, kairouz2021advances}. However, FL and Deep Repurposing address different failure modes. FL primarily concerns where training data reside and how model updates are aggregated across clients; DR instead concerns what should be done when an already deployed model contains an output behaviour that has become obsolete, unreliable, or operationally undesirable.

Thus, FL could in principle be used to retrain a new model, particularly as part of a longer-term fleet-level maintenance process. However, this requires communication rounds, client coordination, aggregation, validation, and redeployment. Even when raw data are not transmitted, communicating model updates or distributing a newly trained model can remain expensive in bandwidth-, energy-, or latency-constrained environments~\citep{mcmahan2017communication, bonawitz2019towards}. These assumptions are especially restrictive for remote or autonomous platforms such as rovers, UAVs, and underwater vehicles, where connectivity may be intermittent, and the need to suppress an unreliable output may arise before the next scheduled model update.

DR could additionally be seen as complementary to FL rather than a replacement for it. FL may provide a mechanism for eventual global retraining when sufficient communication, data, and coordination are available. DR instead provides a local post-hoc repair mechanism: once a class is identified as obsolete, the deployed classifier can be modified without waiting for a federated training cycle or transmitting a replacement model through the full deployment pipeline.

\subsection{Deep Unlearning and Privacy}
\label{sec:appendix-discussion-Deep Unlearning and Privacy}

Machine unlearning is a useful comparison, especially when the obsolete region corresponds to data, users, or groups whose influence should be removed from the model. In that setting, the aim is not just to alter the final prediction, but to make the model behave as though the forget set had not contributed to training~\citep{guo2020certified}. This is why unlearning is closely tied to privacy, membership-inference risk, and right-to-erasure requirements. Recent methods, therefore, update or dampen model parameters to reduce target information while trying to preserve retained utility~\citep{foster2024fast}.

Deep Repurposing has a different purpose. DR does not certify that an obsolete class, region, or group has been removed from the model's training influence. The pretrained parameters remain fixed, and the obsolete region may still have shaped the learned representation. DR should therefore not be viewed as privacy-preserving unlearning, or as a substitute for certified deletion when the obsolete set involves personal data, consent withdrawal, or regulatory removal. In those cases, machine unlearning, retraining, or another deletion mechanism remains the appropriate tool.

Instead, DR treats obsolescence as an operational repair problem. A region may become invalid, unreliable, or no longer useful even when there is no requirement to erase its historical influence. Examples include a retired rover hardware class (Section~\ref{sec:appendix-case-study1}), an unreliable biological taxonomy label (Section~\ref{sec:appendix-case-study2}), or a weather regime that is no longer valid for the revised task. In such cases, removing all learned structure associated with the obsolete region may be unnecessary, and may discard information that still helps retained predictions. DR instead suppresses the obsolete region as a valid output while reallocating retained-compatible evidence through the model's latent geometry.

\subsection{Continual Learning}
\label{sec:appendix-discussion-Continual Learning}

Continual learning addresses model adaptation after deployment. Its usual objective is to update a model as new tasks, classes, or distributions arrive, while limiting catastrophic forgetting of previous knowledge~\citep{parisi2019continual}. This is different from the setting considered here. DR does not aim to learn a new task sequence or acquire new classes; instead, it repairs a fixed model when part of the existing output space has become obsolete.

This distinction also affects the computational setting. Continual learning methods typically require some form of optimisation, replay, regularisation, adapter training, or parameter update. In this sense, they are closer to fine-tuning than to DR in our experiments: they can provide a principled long-term update mechanism, but they still require training data, validation, and deployment of updated parameters. DR avoids this cost by keeping the pretrained model fixed and using a calibration-derived analytic repair map at inference time.



\section{Experimental Setup}
\label{sec:appendix-Experimental Setup}

This section outlines our experimental setup in detail to ensure reproducibility and to make the evaluation protocol transparent. We first present which and how we selected obsolete classes in our evaluation, after which we present the datasets employed for evaluation. Next, we summarise the models/architectures used throughout this paper and then detail the comparative methods considered in our study. Lastly, we report training details and runtime complexity.

\subsection{Selection of Obsolete Classes/Regions}
\label{sec:appendix-Selection of Obsolete Classes}

For CIFAR-10, the obsolete-class evaluation is performed exhaustively: each of the ten classes is treated as obsolete once, giving ten obsolete-class runs. This is feasible because the label space is small and allows us to avoid any sampling decision when defining the obsolete class.

For datasets with substantially larger label spaces, however, an exhaustive obsolete-class sweep becomes impractical. CIFAR-100, Flowers-102, and ImageNet contain 100, 102, and 1000 classes, respectively. Evaluating every possible obsolete class would require 100, 102, and 1000 obsolete-class configurations per method and seed. Since each configuration must be evaluated across the full retained set, obsolete set, and all baselines, this would substantially increase the experimental cost while also producing many redundant class-removal cases. We therefore use a stratified selection protocol for large-label datasets.

The goal of this protocol is not to choose classes that favour any particular method, but to ensure that the selected obsolete classes cover qualitatively different levels of base-model difficulty. In practical deployment, an obsolete class may correspond to a class that the model recognises very reliably, a class of average difficulty, or a class that is already uncertain or error-prone. These cases place different demands on a repurposing method: removing an easy class tests whether the method can suppress a highly confident obsolete decision channel; removing a hard class tests whether the method remains stable when the obsolete class is poorly separated or inconsistently represented; and medium classes provide a representative middle case.

To construct the obsolete-class subset, we first evaluate the trained base model independently on each class in a separate validation set from the train/cal/test sets. For every class, we record class-level accuracy, negative log-likelihood (NLL), expected calibration error (ECE), mean predictive confidence, and predictive entropy. We then define a reliability-risk score:

\begin{equation}
R_c = z(\mathrm{NLL}_c) + z(\mathrm{ECE}_c) + z(H_c) - z(\mathrm{Acc}_c)
\label{eq:reliability-risk-score}
\end{equation}

where $\mathrm{NLL}_c$ is the mean negative log-likelihood for class $c \in \{0, 1, \ldots C\}$, $\mathrm{ECE}_c$ is the class-level calibration error, $H_c$ is the mean predictive entropy, and $\mathrm{Acc}_c$ is the class-level accuracy. The operator $z(\cdot)$ denotes standardisation over classes within the same dataset:

\begin{equation}
z(m_c)=\frac{m_c - \mu_m}{\sigma_m + \epsilon}
\label{eq:standardisation-over-classes}
\end{equation}

where $m$ denotes one of the class-level metrics, $\mu_m$ and $\sigma_m$ are the mean and standard deviation of metric $m$ computed across all classes in the dataset, and $\epsilon=10^{-12}$ for numerical stability. The signs in $R_c$ follow the direction of risk: larger NLL, ECE, and entropy indicate a less reliable class, whereas larger accuracy indicates a more reliable class and is therefore subtracted. Standardising each term prevents any single metric from dominating purely because of its numerical scale~\citep{joint2008handbook}. The score is used only as a selection heuristic for ranking classes from easy to hard; it is not used as an evaluation metric for any repurposing method. Classes with lower reliability risk correspond to easier, more reliable base-model classes, while classes with higher reliability risk correspond to harder, less reliable classes. We then rank all classes within each dataset by this score and select two classes from the easy end of the ranking, two classes from the middle of the ranking, and two classes from the hard end. For a dataset with $C$ classes, this corresponds to selecting classes at ranks $\{1,2\}$, the two central ranks, and ranks $\{C-1,C\}$. Table~\ref{tab:selected_obsolete_classes} reports the selected classes and their base-model statistics.

\begin{table*}[htb]
\centering
\caption{Selected obsolete classes for the large-label datasets. Classes are ranked by the base model's reliability-risk score $R_c$, where lower values indicate easier/more reliable classes and higher values indicate harder/less reliable classes.}
\label{tab:selected_obsolete_classes}
\scriptsize
\resizebox{0.85\textwidth}{!}{%
\begin{tabular}{lllrlrrrr}
\toprule
\textbf{Dataset} & \textbf{Group} & \textbf{Class} & \textbf{Rank} & \textbf{ID} & \textbf{Acc.} & \textbf{NLL} & \textbf{ECE} & $\mathbf{R_c}$ \\
\midrule
CIFAR-100 & Easy   & tractor            & 1    & 89  & 0.970 & 0.269 & 0.107 & -4.959 \\
CIFAR-100 & Easy   & orange             & 2    & 53  & 0.960 & 0.262 & 0.095 & -4.922 \\
CIFAR-100 & Medium & table              & 50   & 84  & 0.860 & 0.637 & 0.111 & -0.880 \\
CIFAR-100 & Medium & worm               & 51   & 99  & 0.820 & 0.611 & 0.104 & -0.861 \\
CIFAR-100 & Hard   & boy                & 99   & 11  & 0.530 & 1.519 & 0.200 & 8.952 \\
CIFAR-100 & Hard   & girl               & 100  & 35  & 0.530 & 1.651 & 0.229 & 9.981 \\
\midrule
Flowers-102 & Easy   & oxeye daisy        & 1    & 48  & 1.000 & 0.150 & 0.138 & -6.428 \\
Flowers-102 & Easy   & black-eyed susan   & 2    & 62  & 1.000 & 0.160 & 0.139 & -6.427 \\
Flowers-102 & Medium & desert-rose        & 51   & 84  & 0.953 & 0.735 & 0.367 & 0.023 \\
Flowers-102 & Medium & wallflower         & 52   & 45  & 0.932 & 0.717 & 0.348 & 0.028 \\
Flowers-102 & Hard   & carnation          & 101  & 30  & 0.656 & 1.734 & 0.361 & 6.295 \\
Flowers-102 & Hard   & canterbury bells   & 102  & 2   & 0.400 & 2.104 & 0.149 & 6.706 \\
\midrule
ImageNet & Easy   & can opener          & 1    & 473 & 0.822 & 1.172 & 0.302 & -6.655 \\
ImageNet & Easy   & seat belt           & 2    & 785 & 0.911 & 0.936 & 0.375 & -6.253 \\
ImageNet & Medium & whiptail            & 500  & 41  & 0.956 & 1.033 & 0.541 & -0.072 \\
ImageNet & Medium & window shade        & 501  & 905 & 0.733 & 1.636 & 0.370 & -0.067 \\
ImageNet & Hard   & velvet              & 999  & 885 & 0.400 & 3.083 & 0.226 & 6.206 \\
ImageNet & Hard   & letter opener       & 1000 & 623 & 0.244 & 3.437 & 0.122 & 7.781 \\
\bottomrule
\end{tabular}%
}
\end{table*}

This produces six obsolete classes per large-label dataset. For example, CIFAR-100 moves from highly reliable classes such as \emph{tractor} and \emph{orange} to harder classes such as \emph{boy} and \emph{girl}. Similarly, Flowers-102 ranges from perfectly classified easy classes to substantially harder flower categories, while ImageNet ranges from relatively reliable classes such as \emph{can opener} and \emph{seat belt} to difficult classes such as \emph{velvet} and \emph{letter opener}. 

Importantly, this selection is performed once using only the base model, before applying any unlearning, suppression, fine-tuning, retraining, or Deep Repurposing method. The same selected obsolete classes are then used for every method. This prevents method-specific cherry-picking while keeping the evaluation tractable and ensuring that the obsolete-class benchmark includes easy, medium, and hard removal scenarios.

For regression, we use the same fixed-selection principle, but the meaning of an obsolete region changes with the task and classification-specific metrics are swapped with analogous regression variants. In classification, each obsolete region corresponds directly to a discrete class label. In regression, the obsolete region may instead be either a target-value interval or an application-defined semantic regime.

For UTKFace, obsolescence is defined over the target space. We partition the continuous age target into seven age regions: 0–12, 13–19, 20–29, 30–39, 40–49, 50–59, and 60–120. The base age regressor is trained on the full original task, with all age regions included. We then evaluate the trained base model separately on each age region using the calibration split. The selected obsolete region is the oldest age interval, 60–120. This region has the largest obsolete-candidate score on the calibration split. It also has the highest calibration MAE, the highest RMSE, the lowest Within-5 accuracy, and the highest mean predictive uncertainty, making it the least reliable target region under the base model. The same region is also ranked highest on the held-out test split. We therefore fix the obsolete set as $\mathcal{O} = \{60 - 120\}$, with all remaining age intervals forming the retained set $\mathcal{R}$. This defines a target-associated regression obsolescence setting, where a previously valid output interval is no longer treated as part of the repaired task.

For Bike Sharing, obsolescence is defined over a semantic region rather than over the numeric regression target. The model predicts the number of registered bike rentals, while the task regions are induced by the \texttt{weathersit} attribute: \texttt{clear}, \texttt{mist}, \texttt{light\_precip}, and \texttt{heavy\_precip}. This evaluates a different deployment scenario from UTKFace: the obsolete region is not a range of rental counts, but a weather regime whose semantic role is no longer valid under the revised task.

We select \texttt{light\_precip} as the obsolete semantic region because it is the most difficult non-degenerate weather regime under the base model. On the calibration split, \texttt{light\_precip} has the largest MAE and largest 90th-percentile absolute error among the populated weather regions. The same pattern holds on the held-out val split, where \texttt{light\_precip} again has the largest MAE, RMSE, and tail error. We therefore fix $\mathcal{O} = \{light\_precip\}$, with \texttt{clear}, \texttt{mist}, and \texttt{heavy\_precip} retained. This defines a semantic-region regression obsolescence setting, where DR is evaluated on whether it can suppress reliance on an obsolete operating regime while preserving the scalar prediction task.

In both regression settings, the obsolete region is selected before applying Deep Repurposing or any comparative method, using only the base model and the fixed calibration protocol. The same retained and obsolete test splits are then used for all methods. This prevents method-specific selection while allowing the regression experiments to test two distinct forms of task obsolescence: target-associated obsolescence in UTKFace and semantic-region obsolescence in Bike Sharing.

\subsection{Datasets}
\label{sec:appendix-datasets}

We conduct evaluations using a broad suite of well-established computer vision and tabular benchmarks that cover different domains and levels of task complexity. This design allows for a thorough assessment of repurposing across diverse conditions. Detailed information on each dataset, including sample counts, input resolution, class composition, and data splits, is provided below. Representative examples are displayed in Figure~\ref{fig:appendix-dataset-vis}.

\begin{figure*}[htb]
    \centering
    \begin{subfigure}{0.475\linewidth}
        \centering
        \includegraphics[width=\linewidth]{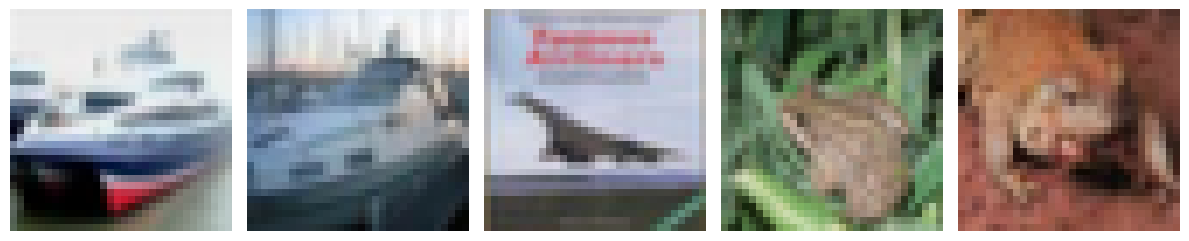}
        \caption{CIFAR10}
    \end{subfigure}
    \hfill
    \begin{subfigure}{0.475\linewidth}
        \centering
        \includegraphics[width=\linewidth]{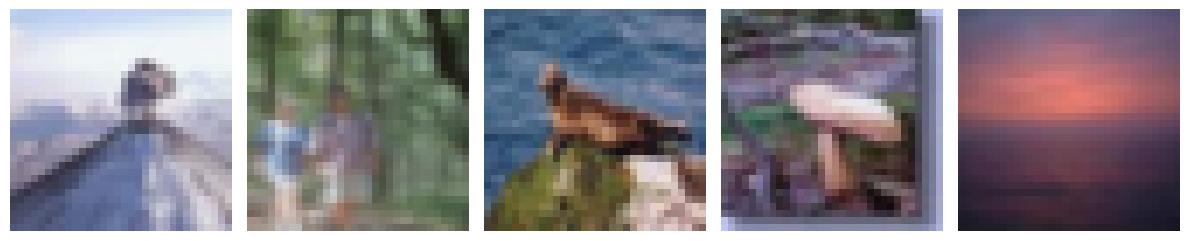}
        \caption{CIFAR100}
    \end{subfigure}
    \begin{subfigure}{0.475\linewidth}
        \centering
        \includegraphics[width=\linewidth]{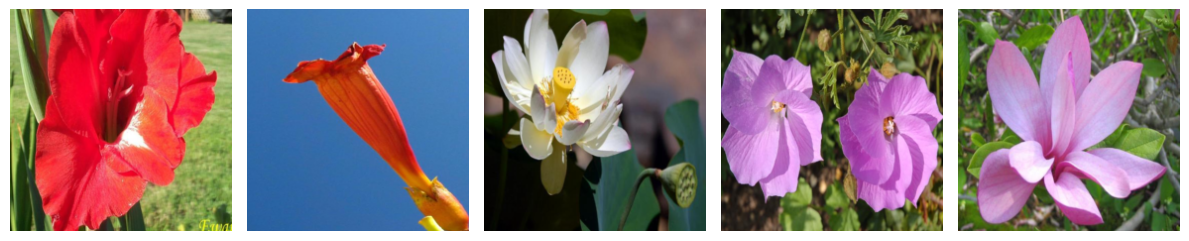}
        \caption{Oxford Flowers 102}
    \end{subfigure}
    \hfill
    \begin{subfigure}{0.475\linewidth}
        \centering
        \includegraphics[width=\linewidth]{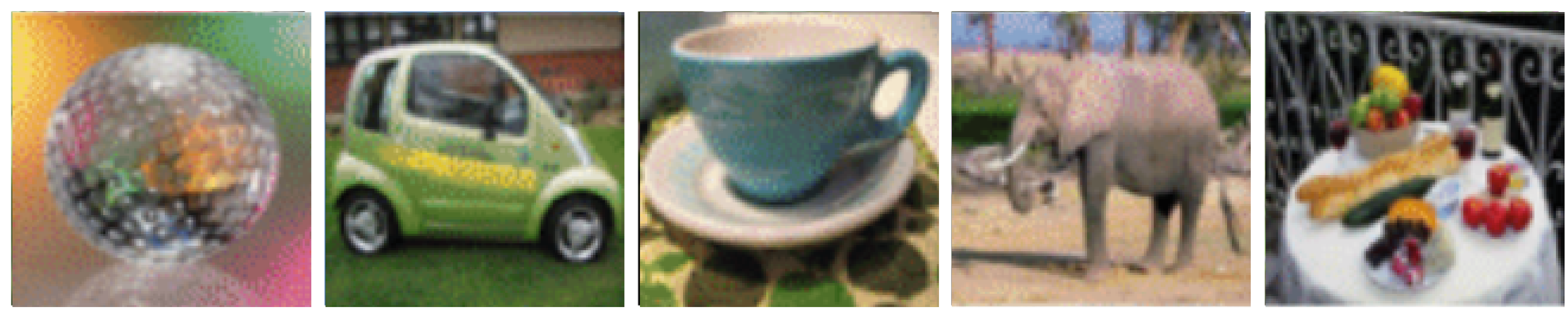}
        \caption{ImageNet}
    \end{subfigure}
    \begin{subfigure}{0.475\linewidth}
        \centering
        \includegraphics[width=\linewidth]{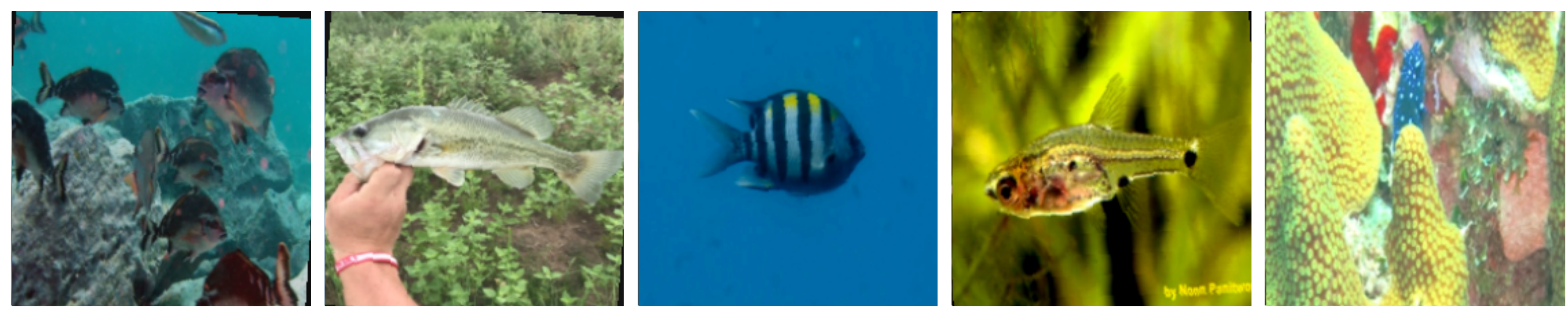}
        \caption{FishNet}
    \end{subfigure}
    \hfill
    \begin{subfigure}{0.475\linewidth}
        \centering
        \includegraphics[width=\linewidth]{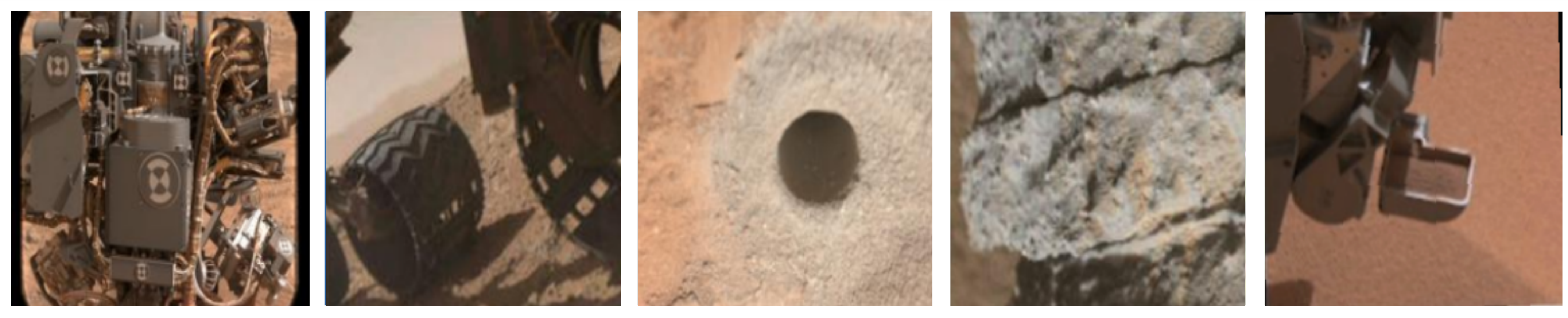}
        \caption{MSL}
    \end{subfigure}
    \begin{subfigure}{0.475\linewidth}
        \centering
        \includegraphics[width=\linewidth]{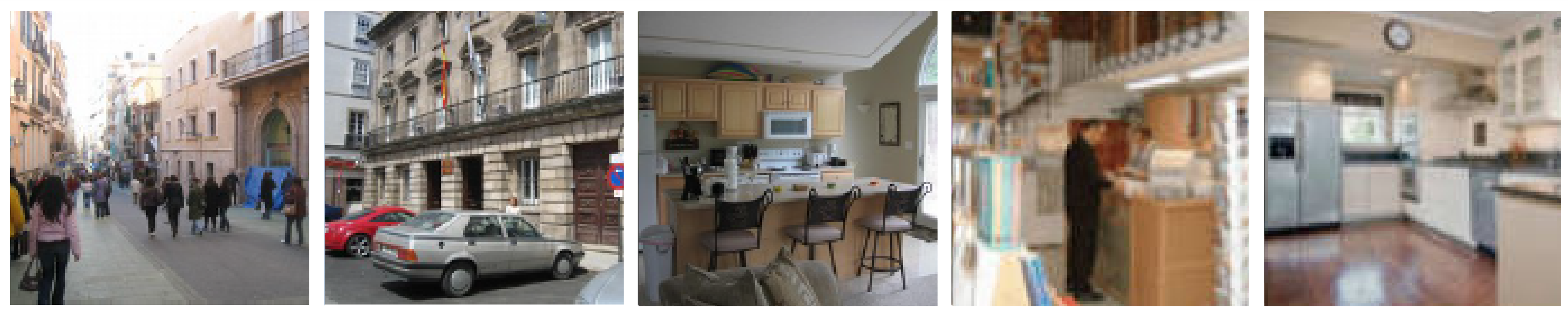}
        \caption{ADE20k}
    \end{subfigure}
    \hfill
    \begin{subfigure}{0.475\linewidth}
        \centering
        \includegraphics[width=\linewidth]{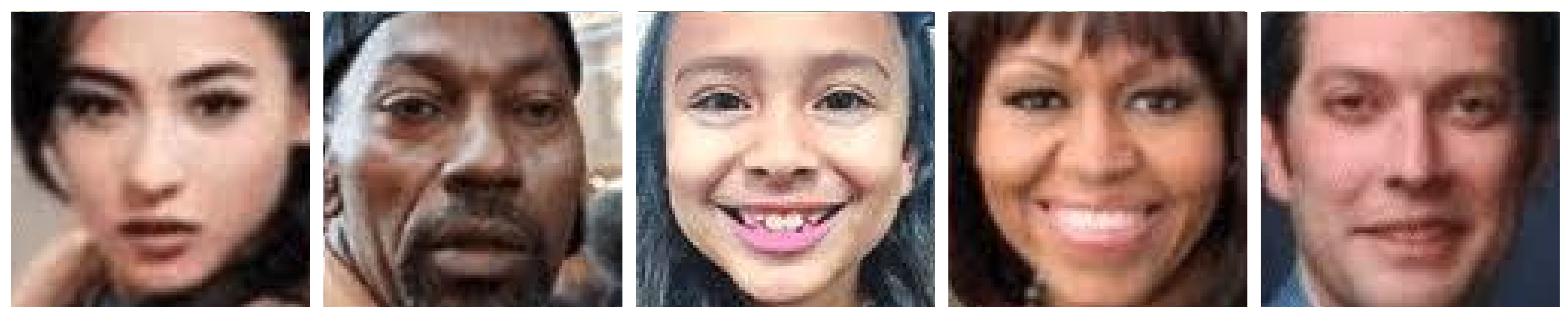}
        \caption{UTKFace}
    \end{subfigure}
\caption{Example images from the datasets used in our experimental evaluation.}
\label{fig:appendix-dataset-vis}
\end{figure*}

\begin{itemize}[noitemsep, nolistsep]
    \item \textbf{CIFAR10~\citep{krizhevsky2009learning}:} consists of 32x32 RGB images of real-world objects (e.g, birds, trucks, airplanes, frogs, etc) spanning 10 classes. We utilise the widely used standard split of 50,000 training samples, 8000 test samples, and 2000 validation samples. Pixel normalisation was applied, bounded $[0,1]$ across all three RGB channels. 
    \item \textbf{CIFAR100~\citep{krizhevsky2009learning}:} consists of 32x32 RGB images of real-world objects spanning of 100 classes that are fine-grained counterparts of the classes from CIFAR10. For example, CIFAR10 has the class \textit{truck}, whilst CIFAR100 has the classes \textit{pickup truck} and \textit{train}. We utilise the widely used standard split of 50,000 training samples, 8000 test samples, and 2000 validation samples. Pixel normalisation was applied, bounded $[0,1]$ across all three RGB channels.
    \item \textbf{Oxford Flowers~\citep{Nilsback08}:} consists of 64x64 RGB images of flowers commonly found in the United Kingdom, spanning 102 (e.g., Water Lily, Wild Pansy, etc.) classes. We utilise a split of 9826 training samples, 1638 test samples, and 1638 validation samples. This is double the size of the original dataset, as each image has been duplicated with a random augmentation to aid generalisation within training. Pixel normalisation was applied, bounded $[0,1]$ across all three RGB channels.
    \item \textbf{ImageNet~\citep{deng2009imagenet}:} consists of high-resolution RGB images of real-world objects spanning 1000 classes, including animals, vehicles, household objects, and tools. We use a split containing 1,281,167 training images, 10,000 validation images, and 40,000 test samples. Images were resized to $256\times256$ pixels, centre-cropped to $224\times224$, and normalised using the standard ImageNet channel means and standard deviations.
    \item \textbf{FishNet~\citep{khan2023fishnet}:} consists of RGB images of aquatic species spanning 83 taxonomic orders (e.g., Perciformes, Siluriformes, and Cypriniformes). We utilise a split of 65,626 training samples, 10,000 validation samples, and 18,906 test samples. Images were resized to $224\times224$ pixels and normalised using the standard ImageNet channel means and standard deviations.
    \item \textbf{MSL~\citep{wagstaff2018deep}:} consists of $256\times256$ RGB images collected from the Mars surface by the Mars Science Laboratory Curiosity rover, spanning 24 classes, including rover instruments, hardware components, calibration targets, ground, and horizon. We utilise the standard split of 3746 training samples, 1640 validation samples, and 1305 test samples. Images were resized to $224\times224$ pixels and normalised using the standard ImageNet channel means and standard deviations.
    \item \textbf{ADE20K~\citep{zhou2017scene}:} consists of RGB images of diverse indoor and outdoor scenes with pixel-level semantic annotations spanning 150 classes, including objects, surfaces, and scene regions. We utilise a split of 10,210 training samples, 10,000 validation samples, and 2000 test samples. Images and segmentation masks were resized to $512\times512$ pixels, with RGB channels normalised using the standard ImageNet mean and standard deviation.
    \item \textbf{UTKFace~\citep{zhifei2017cvpr}:} consists of RGB facial images labelled with ages ranging from 0 to 116 years, alongside gender and ethnicity attributes. We use age as the continuous regression target and utilise a stratified split of 17,069 training samples, 1897 validation samples, and 4742 test samples. Images were resized to $256\times256$ pixels, centre-cropped to $224\times224$, and normalised using the standard ImageNet channel means and standard deviations.
    \item \textbf{Bike Sharing~\citep{fanaee2014event}:} consists of hourly bike-rental records collected over two years, with tabular features describing temporal, seasonal, weather, and environmental conditions. We predict the number of registered bike rentals and utilise a split of 12,513 training samples, 1390 validation samples, and 3476 test samples. Categorical features were one-hot encoded, while continuous features were standardised using their feature-wise means and standard deviations. Example rows can be visualised in Table~\ref{tab:bikesharing-example-rows}.
\end{itemize}

\begin{table}[htb]
\centering
\caption{Example rows from the Bike Sharing hourly dataset.}
\label{tab:bikesharing-example-rows}
\begin{tabular}{@{}ccccccccc@{}}
\toprule
\textbf{Date} & \textbf{Season} & \textbf{Hour} & \textbf{Working day} & \textbf{Weather} & \textbf{Temp.} & \textbf{Humidity} & $\ldots$ & \textbf{cnt} \\
\midrule
2011-01-01 & 1 & 0 & 0 & 1 & 0.24 & 0.81 & $\ldots$ & 16 \\
2011-01-01 & 1 & 1 & 0 & 1 & 0.22 & 0.80 & $\ldots$ & 40 \\
2011-01-01 & 1 & 2 & 0 & 1 & 0.22 & 0.80 & $\ldots$ & 32 \\
\bottomrule
\end{tabular}
\end{table}

\subsection{Models}
\label{sec:appendix-models}

To assess the performance and generalisability of Deep Repurposing, we assess its performance against multiple models and architectures across our experimental evaluation. Below, we provide a concise description of each model together with the implementation details and hyperparameter settings adopted in our experiments:

\begin{itemize}[noitemsep, nolistsep]
    \item \textbf{EfficientNet-B0~\citep{tan2019efficientnet}:} is used as the base classifier for CIFAR10. We adapt the initial convolutional stem for the $32\times32$ CIFAR10 input resolution and replace the original ImageNet classification layer with a 10-class linear output head.
    \item \textbf{EfficientNetV2-S~\citep{tan2021efficientnetv2}:} is used as the base classifier for CIFAR100 and Oxford Flowers. For CIFAR100, we adapt the input stem for the $32\times32$ image resolution and replace the original ImageNet classification layer with a 100-class linear output head. For Oxford Flowers, the standard input stem is retained, and the classification layer is replaced with a 102-class linear output head.
    \item \textbf{ResNet-50~\citep{he2016deep}:} is used as the base classifier for ImageNet and as the base age-regression model for UTKFace. For ImageNet, we retain the standard architecture and its 1000-class classification head, using weights pretrained on the complete ImageNet-1K training set. For UTKFace, the model is initialised from the same ImageNet-pretrained weights, with the original classification head replaced by a single-output linear regression layer for age prediction and dropout applied to the pooled feature representation.
    \item \textbf{ViT-B/16~\citep{dosovitskiy2020image}:} is used as the base classifier for the MSL dataset. The original ImageNet classification head is replaced with an output layer matching the MSL class taxonomy.
    \item \textbf{MobileNetV3-Large~\citep{howard2019searching}:} is used as the base classifier for the FishNet dataset. The original ImageNet classification head is replaced with an output layer matching the 83 fish-order classes used in our experiments.
    \item \textbf{SegFormer-B0~\citep{xie2021segformer}:} is used as the base semantic segmentation model for ADE20K. We use the \texttt{nvidia/segformer-b0-finetuned-ade-512-512} checkpoint, which is pretrained and fine-tuned on the complete ADE20K semantic segmentation task with 150 output classes. Images are processed at $512\times512$ resolution, and the model produces dense pixel-level class predictions.
    \item \textbf{Multilayer Perceptron:} is used as the base regression model for the Bike Sharing dataset. The network contains three fully connected hidden layers of width 128, followed by a 64-dimensional latent layer and a single-output linear regression head. Each hidden stage uses layer normalisation, GELU activation, and dropout with probability $0.10$. The model is trained on the complete bike-demand prediction task before any semantic region is designated as obsolete, and the resulting checkpoint remains fixed across all comparative repair methods.
\end{itemize}

For all models, the model is initialised from ImageNet-pretrained weights for classification, ADE20K pretrained weights for segmentation and subsequently fine-tuned on the complete original task before any classes are designated as obsolete. The resulting dataset-specific checkpoints are fixed and shared across all comparative repair methods.

\subsection{Comparative Approaches}
\label{sec:appendix-Comparative Methods}

To assess the performance of DR, we benchmark it against a set of recent model unlearning, editing, and retraining methods that capture the state of the art in this area. Below, we provide a concise description of each method together with the implementation details and hyperparameter settings adopted in our experiments:

\begin{itemize}[noitemsep, nolistsep]
    \item \textbf{Base Model:} refers to the frozen base neural network on which all unlearning, editing, etc methods, including DR, are applied. It serves as the baseline.
    \item \textbf{Selective Synaptic Dampening (SSD)~\citep{foster2024fast}:} is a post-hoc, retraining-free unlearning method that estimates parameter importance using the diagonal Fisher information on the full and obsolete datasets. The selection weighting is set to $\alpha=10$, the dampening constant to $\lambda=1$, the dampening exponent to $1$, and the maximum dampening multiplier to $1$. Fisher importance is estimated from the calibration set using batches of 128.
    \item \textbf{Loss-Free Selective Synaptic Dampening (LFSSD)~\citep{foster2024loss}:} extends SSD by replacing its label-dependent Fisher importance with the gradient magnitude of the squared $\ell_2$ norm of the model output, allowing parameter importance to be estimated without labels or a task loss. In our implementation, the selection weighting is set to $\alpha=10$, the dampening constant to $\lambda=1$, the dampening exponent to $1$, and the maximum dampening multiplier to $1$. Importance is estimated from the calibration set using batches of 128, with the squared $\ell_2$ norm of the logits as the output objective.
    \item \textbf{Boundary Shrink~\citep{chen2023boundary}:} is a decision-boundary-based unlearning method that generates a perturbed cross-sample for each obsolete example and uses the original model to identify its nearest incorrect class. The obsolete samples are then relabelled with these neighbouring retained classes and used to fine-tune the model, shrinking the obsolete-class decision region. In our implementation, we use a perturbation magnitude of $\epsilon=0.25$ with one search step, followed by 10 fine-tuning epochs using SGD with a learning rate of $10^{-5}$, momentum of $0.9$, no weight decay, and a batch size of 128. Where necessary, the relabelled target is restricted to the highest-scoring retained class.
    \item \textbf{Boundary Expand~\citep{chen2023boundary}:} is a decision-boundary-based unlearning method that introduces an additional shadow class and relabels all obsolete samples to this auxiliary output. The expanded model is then fine-tuned so that obsolete information is redirected towards the shadow class, after which the added output neuron is removed to discard the associated activation. In our implementation, the model is fine-tuned for 10 epochs using SGD with a learning rate of $10^{-5}$, momentum of $0.9$, no weight decay, and a batch size of 128. For multiple obsolete classes, all obsolete samples are mapped to a single shared shadow class.
    \item \textbf{Fictitious Class Removal (FCR)~\citep{hakemi2025post}:} is a post-hoc editing method that treats the obsolete concept as a fictitious class and identifies the final-layer connection most strongly associated with it using class-specific activations, gradients, and entropy. The selected connection is then analytically modified to neutralise the corresponding decision hyperplane while leaving the remaining parameters unchanged. In our implementation, a probe copy is fine-tuned for 5 epochs using a learning rate of $10^{-3}$ and no weight decay to make the fictitious class identifiable. We apply a neutralisation rate of $1.0$, modify at most one weight per obsolete class, and estimate the association scores using batches of 128.
    \item \textbf{Fisher Forgetting~\citep{golatkar2020eternal}:} is a post-hoc weight-scrubbing method that approximates the retained-data loss curvature using the diagonal Fisher information matrix and modifies parameters according to their estimated sensitivity. We apply a deterministic scrub with scale $10^{-4}$, damping $10^{-2}$, and a maximum absolute update of $10^{-3}$.
    \item \textbf{Class Vector Unlearning~\citep{kim2026exploring}:} is a classifier-editing method that represents the adaptation of each class as the difference between its mean pretrained and fine-tuned latent representations. To unlearn an obsolete class, this class vector is reversed and mapped into the model weights so that obsolete samples are shifted away from their learned class-specific representation while limiting interference with retained classes. We evaluate the weight-space variant, using a scaling coefficient of $\lambda=-1.5$. Class centroids are estimated from the calibration set, and the final feature block is optimised for 20 epochs using AdamW with a learning rate of $10^{-4}$, no weight decay, a batch size of 32, one reference sample per obsolete class, and a gradient clipping norm of $1.0$.
    \item \textbf{Fine-tuning:} updates the pretrained model using only retained training samples, with all obsolete-class samples removed from the optimisation dataset. In our implementation, all model parameters are fine-tuned for 20 epochs using AdamW with a learning rate of $10^{-4}$, no weight decay, and a batch size of 128. This baseline represents conventional continual adaptation of the existing model without explicitly modifying or removing its obsolete output nodes.
    \item \textbf{Retraining:} trains a newly initialised model from scratch using only retained training samples, providing the conventional reference for a model that has not been exposed to the obsolete classes during training. The output layer is reduced to the retained classes during optimisation and subsequently mapped back to the original output space for evaluation, with obsolete logits fixed to $-10^{9}$. In our implementation, the model is trained for 50 epochs using AdamW with a learning rate of $10^{-4}$, weight decay of $10^{-4}$, and a batch size of 128.
    \item \textbf{Deep Repurposing (DR):} is our post-hoc, parameter-free repair method that operates in the latent space of a fixed pretrained model. During offline calibration, DR estimates diagonal Gaussian region prototypes, obsolete-to-retained removal and repair directions, and retained residual-refinement vectors from the calibration set. At inference, obsolete evidence is analytically removed and redistributed towards compatible retained regions, after which the output and latent representation are reconstructed using retained-region structure. We use a stripping margin of $m=1$, Bayes weight $\beta=1$, repair weight $\lambda=1$.
\end{itemize}

Unless otherwise stated, the hyperparameter settings for all comparative methods follow the recommendations of their respective papers and official implementations. Where a dataset- or architecture-specific adjustment was required, a small amount of manual tuning was performed, with the same settings then fixed across all corresponding experiments.

\subsection{Training Details}
\label{sec:appendix-Training Details}

To ensure consistency and reproducibility across all experiments, we adopt a fixed training configuration and architectural setup for all models, including our proposed Deep Repurposing and baseline comparators, unless explicitly stated otherwise.

All experiments were implemented in PyTorch and executed on a large performance GPU cluster using a maximum of three Nvidia A40 GPUs, 32 CPU cores, 167GB of memory, unless stated otherwise. All runs used 10 random seeds per obsolete class evaluated. So, for example, on the base CIFAR10 experiments in Table~\ref{tab:appendix_core_results}, we evaluate where each class is obsolete. So each method, therefore, has 100 runs as there are 10 classes in CIFAR10.

\subsection{Runtime and Complexity}
\label{sec:Runtime and Complexity}

Let $N=|\mathcal{D}_{\mathrm{cal}}|$ denote the number of calibration representations, $d$ the latent dimensionality, and $|\mathcal{C}|$, $|\mathcal{O}|$, and $|\mathcal{R}|$ the numbers of total, obsolete, and retained regions, respectively, where $\mathcal{R}=\mathcal{C}\setminus\mathcal{O}$. DR performs a single frozen forward pass over $\mathcal{D}_{\mathrm{cal}}$, after which the region prototypes and pairwise obsolete-retained directions are estimated without gradient optimisation. In the current implementation, fitting the region statistics requires $O(N|\mathcal{C}|+Nd)$ operations, while constructing the pairwise directions requires $O(|\mathcal{O}||\mathcal{R}|d)$. The main calibration cost arises from applying the core repair map to the calibration representations: in the worst case, contrastive transport evaluates each retained destination against all calibrated region prototypes, giving an overall one-off calibration complexity of:
\begin{equation}
O\!\left(
F_\theta(\mathcal{D}_{\mathrm{cal}})
+N|\mathcal{R}||\mathcal{C}|d
+|\mathcal{O}||\mathcal{R}|d
\right)
\end{equation}
where $F_\theta(\mathcal{D}_{\mathrm{cal}})$ denotes the cost of the frozen model forward pass. This calibration is performed only once before deployment. 

At inference, let $P$ denote the number of prediction locations and $M\leq P$ the number for which the repair mask $M(x)$ is active. Region-likelihood evaluation requires $O(P|\mathcal{C}|d)$ operations, while contrastive transport contributes $O(M|\mathcal{R}||\mathcal{C}|d)$. Residual refinement requires $O(P|\mathcal{R}|^2)$ and simplex projection requires $O(P|\mathcal{R}|\log|\mathcal{R}|)$. Thus, excluding score-consistent latent reconstruction, the worst-case DR overhead is:
\begin{equation}
O\!\left(
P|\mathcal{R}||\mathcal{C}|d
+P|\mathcal{R}|^2
\right)
\end{equation}
when all prediction locations activate repair. For score-based classification and segmentation, the current implementation additionally reconstructs the repaired latent using an SVD-based Moore-Penrose pseudoinverse of the retained head, contributing:
\begin{equation}
O\!\left(
|\mathcal{R}|d\min(|\mathcal{R}|,d)
+d^2|\mathcal{R}|
+Pd^2
+P|\mathcal{R}|d
\right)
\end{equation}
The persistent DR calibration state requires
$O(|\mathcal{C}|d+|\mathcal{O}||\mathcal{R}|d+|\mathcal{R}|^2)$
additional storage.

Despite this worst-case scaling, the region sets used in practice are finite and the operations are predominantly batched tensor computations, with contrastive transport evaluated only for representations that activate the repair mask. Consequently, DR remains lightweight in our experiments, requiring only $\approx$1.62s on CIFAR10 and remaining seconds-scale across substantially larger models and task spaces.


\end{document}